\documentclass[twoside,11pt,preprint]{article}

\usepackage{jmlr2e}

\usepackage{lastpage}

\jmlrheading{25}{2024}{1-\pageref{LastPage}}{6/22; Revised
1/24}{2/24}{22-0667}{George H.~Chen}
\ShortHeadings{Survival Kernets}{Chen}

\firstpageno{1}

\usepackage{amsmath}
\usepackage{amssymb}
\usepackage{booktabs}
\usepackage{enumitem}
\usepackage{dsfont}
\newcommand{\ind}{\mathds{1}}
\def\smallunderbrace#1{\mathop{\vtop{\m@th\ialign{##\crcr
   $\hfil\displaystyle{#1}\hfil$\crcr
   \noalign{\kern3\p@\nointerlineskip}%
   \tiny\upbracefill\crcr\noalign{\kern3\p@}}}}\limits}
\makeatother
\numberwithin{equation}{section}
\newcommand{\eventInd}{{D}}
\newcommand{\pre}{\circ}
\usepackage{threeparttable}
\usepackage{multirow}
\usepackage{rotating}
\usepackage{float}
\counterwithin{table}{section}
\usepackage{caption}
\usepackage{subcaption}
\newtheorem{claim}[theorem]{Claim}
\usepackage{pifont}
\usepackage[dvipsnames]{xcolor}
\newcommand{\cmark}{\textcolor{Green}{\ding{51}}}
\newcommand{\xmark}{\textcolor{BrickRed}{\ding{55}}}
\newcommand{\bftab}{\fontseries{b}\selectfont}
\usepackage{adjustbox}

\begin{document}

\title{Survival Kernets: Scalable and Interpretable \\ Deep Kernel Survival Analysis with an Accuracy Guarantee}

\author{\name George H.~Chen \email georgechen@cmu.edu \\
       \addr Heinz College of Information Systems and Public Policy\\
       Carnegie Mellon University\\
       Pittsburgh, PA 15213, USA}

\editor{Ji Zhu}

\maketitle

\begin{abstract}
Kernel survival analysis models estimate individual survival distributions with the help of a kernel function, which measures the similarity between any two data points. Such a kernel function can be learned using deep kernel survival models. In this paper, we present a new deep kernel survival model called a \emph{survival kernet}, which scales to large datasets in a manner that is amenable to model interpretation and also theoretical analysis. Specifically, the training data are partitioned into clusters based on a recently developed training set compression scheme for classification and regression called \emph{kernel netting} that we extend to the survival analysis setting. At test time, each data point is represented as a weighted combination of these clusters, and each such cluster can be visualized. For a special case of survival kernets, we establish a finite-sample error bound on predicted survival distributions that is, up to a log factor, optimal. Whereas scalability at test time is achieved using the aforementioned kernel netting compression strategy, scalability during training is achieved by a warm-start procedure based on tree ensembles such as \textsc{xgboost} and a heuristic approach to accelerating neural architecture search. On four standard survival analysis datasets of varying sizes (up to roughly 3 million data points), we show that survival kernets are highly competitive compared to various baselines tested in terms of time-dependent concordance index. Our code is available at: \\ \url{https://github.com/georgehc/survival-kernets}
\end{abstract}


\section{Introduction}
\label{sec:intro}

Survival analysis is about modeling the amount of time that will elapse before a critical event of interest happens. Examples of such critical events include death, hospital readmission, disease relapse, device failure, or a convicted criminal reoffending. A key technical challenge in survival analysis is that typically when collecting training data, we cannot wait until the critical event happens for every training data point, such as waiting for everyone in a clinical study to be deceased. Thus, in our training data, we observe the time duration we aim to predict for only some but not all of the data points. This is in contrast to standard classification and regression problem settings in which the target label to be predicted is observed across all training data. Importantly, the data points that do not encounter the critical event could still provide valuable information and should not simply be discarded from the data analysis. For instance, they may have specific characteristics that significantly delay when the critical event will happen.

In the last few decades, survival analysis datasets have dramatically grown in size, from hundreds of data points (e.g., the German Breast Cancer Study Group dataset \citep{schumacher1994randomized}) to now millions (e.g., customer churn data from the music streaming service \textsc{kkbox}\footnote{\url{https://www.kaggle.com/c/kkbox-churn-prediction-challenge}}). We anticipate that in the years to come, survival analysis datasets larger than ones typically encountered today will become the norm. Meanwhile, we remark that many survival analysis problems are in high-stakes application domains such as healthcare. For such applications, it can be helpful for the survival analysis models used to be interpretable.
With these large-scale high-stakes applications in mind, in this paper, we propose a new deep survival analysis model called a \emph{survival kernet} that has all of the following properties:
\begin{itemize}[itemsep=0.2em,parsep=0em,partopsep=0em]
\item flexible (model is nonlinear and nonparametric, and achieves concordance indices nearly as high as or higher than the best baselines tested in our experiments)
\item scalable (uses a compression technique to construct a test-time predictor that can handle large datasets such as the \textsc{kkbox} dataset; also uses a scalable neural net warm-start initialization strategy)
\item interpretable (any data point is represented by a weighted combination of a few clusters, each of which can be visualized---this is somewhat like how topic modeling for text describes each text document in terms of a few topics, each of which can be visualized)
\item for a special case of the model, comes with a theoretical guarantee on prediction accuracy (a finite-sample error bound for predicted survival distributions)
\end{itemize}
In contrast, existing deep survival analysis models that have been developed typically are not easily interpretable (e.g., \citealt{ranganath2016deep,fotso2018deep,chapfuwa2018adversarial,giunchiglia2018rnn,katzman2018deepsurv,lee2018deephit,kvamme2019time,engelhard2020neural,zhong2021deep,danks2022derivative,tang2022soden}). Instead, to get a deep survival analysis model to be interpretable, \citet{zhong2022deep}, for instance, model a subset of features by a linear model (that is straightforward to interpret) and then model the rest of the features by an arbitrarily complex neural net that they do not aim to interpret. While this sort of interpretability could be valuable, it does not resolve the difficulty of interpreting the part of the model that actually uses the neural net. Meanwhile, \citet{chapfuwa2020survival}, \citet{nagpal2021deep}, and \citet{manduchi2022deep} have developed deep survival models that are interpretable in the sense of representing data points in terms of clusters, but none of these models have accuracy guarantees. Somewhat like using clustering, neural survival topic models by \citet{li2020neural} represents data in terms of ``topics'', where each topic corresponds to specific raw features being more probable and also to either higher or lower survival times; these models again lack accuracy guarantees. We are aware of only two existing deep survival analysis models that have accuracy guarantees \citep{zhong2021deep,zhong2022deep}, both of which have only been tested on small datasets (the largest real dataset these authors consider has about two thousand data points) and the theoretical analyses for these models are limited to a special family of neural nets (involving $\ell_0$ and sup-norm constraints) that are not used in practice. We summarize how our proposed survival kernet model compares to some representative existing models in Table~\ref{tab:property-comparison}; note that this table is partially based on a similar comparison table by \citet{tang2022soden}.

\begin{table}[!t]
\centering
\caption{Comparison of the proposed survival kernet model with some representative existing survival models.}
\label{tab:property-comparison}
\vspace{-.5em}
\setlength{\tabcolsep}{5pt}
\adjustbox{max width=\textwidth}{%
\begin{tabular}{cccccc}
\toprule
\multirow{2}{*}{Model} & \multirow{2}{*}{Nonlinear} & No proportional & Designed to be & Accuracy & Comment on \\
 &  & hazards assumption & interpretable$^*$ & guarantee$^\dagger$ & computational scalability$^\ddagger$ \\ \midrule
\addlinespace[0.5em]
Cox \citep{cox1972regression} & \xmark & \xmark & \cmark & \cmark & SGD \\
\addlinespace[0.25em]
XGBoost \citep{chen2016xgboost} & \cmark & \xmark~$\!^\S$ & \xmark & \xmark & various optimizations \\
\addlinespace[0.25em]
DeepSurv \citep{katzman2018deepsurv} & \cmark & \xmark & \xmark & \xmark & SGD \\
\addlinespace[0.25em]
DeepHit \citep{lee2018deephit} & \cmark & \cmark & \xmark & \xmark & SGD \\
\addlinespace[0.25em]
Nnet-survival \citep{gensheimer2019scalable} & \cmark & \cmark & \xmark & \xmark & SGD \\
\addlinespace[0.25em]
Cox-time \citep{kvamme2019time} & \cmark & \cmark & \xmark & \xmark & SGD \\
\addlinespace[0.25em]
Deep Extended Hazard \citep{zhong2021deep} & \cmark & \cmark~$\!^\P$ & \xmark & \cmark & SGD \\
\addlinespace[0.25em]
SODEN \citep{tang2022soden} & \cmark & \cmark & \xmark & \xmark & SGD \\
\addlinespace[0.25em]
Neural survival topic models \citep{li2020neural} & \cmark & \cmark & \cmark & \xmark & SGD \\
\addlinespace[0.5em]
\multirow{2}{*}{Deep Cox mixtures \citep{nagpal2021deep}} & \multirow{2}{*}{\cmark} & \multirow{2}{*}{\cmark} & \multirow{2}{*}{\cmark} & \multirow{2}{*}{\xmark} & uses spline fitting procedure with \\
& & & & & computation linear in training set size$^\|$ \\
\addlinespace[0.5em]
\multirow{2}{*}{Deep kernel survival analysis \citep{chen2020deep}} & \multirow{2}{*}{\cmark} & \multirow{2}{*}{\cmark} & \multirow{2}{*}{\cmark} & \multirow{2}{*}{\xmark} & SGD but prediction for single point uses \\
& & & & & computation linear in training set size \\
\addlinespace[0.5em]
\multirow{2}{*}{Survival kernets (proposed)} & \multirow{2}{*}{\cmark} & \multirow{2}{*}{\cmark} & \multirow{2}{*}{\cmark} & \multirow{2}{*}{\cmark} & SGD + tunable test-time compression \\
& & & & & + scalable warm-start \\
\bottomrule
\end{tabular}} \\
\smallskip
{\scriptsize
$^*$ For models designed to be interpretable, the notion of interpretability varies (the Cox model is linear, neural survival topic models decompose into a topic model and a survival model that each can be interpreted, Deep Cox mixtures and survival kernets represents data using clusters, and deep kernel survival analysis can explicitly indicate which training points contribute to each prediction and what their relative weights are) \\[3pt]
$^\dagger$ Accuracy guarantees for the deep extended hazard model and survival kernets are for special cases of these models \\[3pt]
$^\ddagger$ For some models that can be trained using (minibatch) SGD, doing so is not actually theoretically justified but is straightforward to implement in practice \\[3pt]
$^\S$ Specific to using Cox regression as the learning objective, which we do in our experiments \\[3pt]
$^\P$ The deep extended hazard model is a hybrid between a proportional hazards model and an accelerated failure time model so that it partially uses a proportional hazards assumption (but from training, the model could automatically decide to rely on the accelerated failure time component and not really use the proportional hazards component) \\[3pt]
$^\|$ Deep Cox mixtures uses an Expectation-Maximum (EM) training procedure that partly uses minibatches but every few minibatches looks at the entire training set in computing a survival curve (as a spline) per cluster \\[-1em]}
\end{table}

Our work builds on an existing model called deep kernel survival analysis \citep{chen2020deep}. Unlike non-kernel-based deep survival models that have been developed, deep kernel survival analysis and a Bayesian variant by \citet{wu2021uncertainty} learn a kernel function that measures how similar any two data points are. To predict the survival distribution of a specific test point, these kernel-based methods use information from training points most similar to the test point according to the learned kernel function. The learned kernel function can help us probe a resulting survival model. For instance, these kernel functions can provide forecast evidence in terms of which training points contribute to a test point's prediction. As pointed out by \citet{chen2020deep}, the kernel functions can also be used to construct statistically valid prediction intervals that are \emph{relative} (e.g., the hospital length of stay of patients \emph{similar to Alice} are within the interval [0.5, 2.5] days with probability at least 80\%). Separate from these practical advantages, a kernel-based approach is also amenable to theoretical analysis, where we import proof techniques by \citet{chen2019nearest} to analyze survival kernets.

A key obstacle to using kernel survival models in practice is that at test time, these models in principle depend on knowing the similarity between a test point and every training point. The computation involved in a naive implementation becomes impractical when the training dataset size grows large. In the Bayesian setting, this scalability problem has been addressed but without guarantees in the survival analysis setting and also thus far without experiments on datasets with censoring \citep{wu2021uncertainty}.

Our main technical contribution in this paper is to show how to scale deep kernel survival analysis at test time to large datasets in a manner that not only yields an interpretable model but also achieves a finite-sample accuracy guarantee on predicted survival distributions.
The resulting model is what we call a \emph{survival kernet}.
We achieve this test-time scalability by extending an existing data compression scheme (\emph{kernel netting} by \citet{kpotufe2017time} developed for classification and regression) to the survival analysis setting (``survival kernet'' combines ``survival analysis'' and ``kernel netting'' into a short phrase). Kernel netting could be viewed as partitioning the training data into clusters and then representing any data point as a weighted combination of a few clusters. We show how to visualize such clusters using a strategy similar to that of \citet{chapfuwa2020survival}. We remark that Kpotufe and Verma did not consider interpretability or how to visualize clusters in their original kernel netting paper.

We next turn to training scalability. Here, minibatch gradient descent readily enables fitting a deep kernel survival model to large datasets for a specific neural net architecture. The challenge is that there could be many neural net architectures to try, making the overall training procedure computationally expensive. Our second contribution is in proposing a warm-start approach to deep kernel survival model training that drastically reduces the amount of computation needed when sweeping over neural net architecture choices. Our proposed warm-start strategy first learns a kernel function using a scalable tree ensemble such as \textsc{xgboost} \citep{chen2016xgboost}. We then fit a neural net (trying different neural net architectures) to this learned tree ensemble kernel function before fine tuning using survival kernet training (at which point the neural net architecture is fixed). We call our warm-start approach \textsc{tuna} (\emph{Tree ensemble Under a Neural Approximation}). We remark that \citet{chen2020deep} had come up with an earlier tree ensemble initialization strategy but it does not scale to large datasets.

We demonstrate survival kernets with and without \textsc{tuna} on four standard survival analysis datasets, three of which are healthcare-related on predicting time until death (with number of data points ranging from thousands to tens of thousands), and the fourth is the \textsc{kkbox} dataset (millions of data points). Survival kernets with \textsc{tuna} achieve accuracy scores nearly as high or higher than the best performing baselines that we tested. Meanwhile, using \textsc{tuna} to accelerate training consistently results in higher accuracy models than not using \textsc{tuna} and reduces overall training times by 17\%--85\% in our experiments (with time savings that are more dramatic on larger datasets). We show how to interpret survival kernet models trained on all four datasets we consider.

\section{Background\label{sec:deep-kernel-prediction}}

We first review the standard right-censored survival analysis setup (model and prediction task) and deep kernel survival analysis \citep{chen2020deep}. For the latter, we do not address the Bayesian formulation (e.g., \citealt{wu2021uncertainty}) as the machinery involved is a bit different and our theoretical analysis later is frequentist. For ease of exposition, we phrase terminology in terms of predicting time until death although in general, the critical event of interest need not be death.

\paragraph{Model}
Let $(X_{1},Y_{1},D_{1}),\dots,(X_{n},Y_{n},D_{n})$ denote the training data, where the $i$-th training point has feature vector $X_{i}\in\mathcal{X}$, observed nonnegative time duration $Y_{i}\ge0$, and event indicator $D_i\in\{0,1\}$; if $D_i=1$, then $Y_i$ is a time until death whereas if $D_i=0$, then $Y_i$ is a time until censoring (i.e., the $i$-th point's true time until death is at least $Y_{i}$). Each point $(X_i,Y_i,D_i)$ is assumed to be generated i.i.d.~as follows:
\begin{enumerate}[itemsep=0.2em,parsep=0em,partopsep=0em]
\item Sample feature vector $X_{i}\sim\mathbb{P}_{\mathsf{X}}$.
\item Sample nonnegative survival time $T_{i}\sim\mathbb{P}_{\mathsf{T}\mid\mathsf{X}=X_{i}}$.
\item Sample nonnegative censoring time $C_{i}\sim\mathbb{P}_{\mathsf{C}\mid\mathsf{X}=X_{i}}$.
\item If $T_{i}\le C_{i}$ (death happens before censoring): set $Y_{i}=T_{i}$ and $\eventInd_{i}=1$.

Otherwise (death happens after censoring): set $Y_{i}=C_{i}$ and $\eventInd_{i}=0$.
\end{enumerate}
Distributions $\mathbb{P}_{\mathsf{X}}$, $\mathbb{P}_{\mathsf{T}\mid\mathsf{X}}$, and $\mathbb{P}_{\mathsf{C}\mid\mathsf{X}}$ are unknown to the learning method. We assume that the distribution $\mathbb{P}_{\mathsf{T}|\mathsf{X}=x}$ is for a continuous random variable with CDF $F_{\mathsf{T}|\mathsf{X}}(t|x)$ and PDF $f_{\mathsf{T}|\mathsf{X}}(t|x)=\frac{\partial}{\partial t}F_{\mathsf{T}|\mathsf{X}}(t|x)$.

\sloppy
\paragraph{Prediction task} A standard prediction task is to estimate, for a test feature vector~$x$, the conditional survival function
\[
S(t|x):=\mathbb{P}(\text{time until death}>t~\!|~\!\text{feature vector}=x)=1-F_{\mathsf{\mathsf{T}|\mathsf{X}}}(t|x),
\]
which is defined for all $t\ge0$. A closely related problem is to estimate the so-called hazard function
\begin{equation}
h(t|x)
:=
-\frac{\partial}{\partial t} \log S(t|x)
= -\frac{\frac{\partial}{\partial t}S(t|x)}{S(t|x)}
=-\frac{\frac{\partial}{\partial t}[1-F_{\mathsf{T}|\mathsf{X}}(t|x)]}{S(t|x)}
=\frac{f_{\mathsf{\mathsf{T}|\mathsf{X}}}(t|x)}{S(t|x)},
\label{eq:hazard}
\end{equation}
which is (from the final expression) the instantaneous rate of death at time $t$ given survival at least through time $t$ for feature vector $x$. By how the hazard function is defined, $S(t|x)=\exp(-\int_{0}^{t}h(s|x)ds)$, so estimating $h(\cdot|x)$ yields an estimate of $S(\cdot|x)$.

\paragraph{Kernel estimators}
Let $\mathbb{K}:\mathcal{X}\times\mathcal{X}\rightarrow[0,\infty)$ denote a kernel function that measures how similar any two feature vectors are (a higher value means more similar). We explain how this function can be learned shortly in a neural net framework. For now, consider it to be pre-specified. Then we can estimate the hazard function as follows.

\emph{Hazard function estimator.}
Let $t_{1}<t_{2}<\cdots<t_{m}$ denote the unique times of death in the training data, and define $t_0:=0$. Then a kernel predictor for a discretized version of the hazard function is, for time indices $\ell=1,2,\dots,m$ and feature vector $x\in\mathcal{X}$,
\begin{equation}
\widehat{h}(\ell|x):=\frac{\sum_{j=1}^{n}\mathbb{K}(x,X_{j})\eventInd_{j}\ind\{Y_{j}=t_{\ell}\}}{\sum_{j=1}^{n}\mathbb{K}(x,X_{j})\ind\{Y_{j}>t_{\ell-1}\}},
\label{eq:hazard-predictor}
\end{equation}
where $\ind\{\cdot\}$ is the indicator function that is~1 when its argument is true and 0 otherwise. Note that $\widehat{h}(\ell|x)$ estimates the probability of dying at time~$t_{\ell}$ conditioned on surviving beyond time~$t_{\ell-1}$ for feature vector~$x$. To see this, consider when $\mathbb{K}(x, X_j)=1$ for all~$j$, in which case the numerator counts the total number of deaths at time $t_\ell$, and the denominator counts the total number of data points that survived beyond time $t_{\ell-1}$; adding kernel weights that are not necessarily always~1 weights the contribution of the $j$-th training point depending on how similar test feature vector $x$ is to $X_j$. As a corner case, in evaluating equation~\eqref{eq:hazard-predictor}, if the numerator and denominator are both 0, we use the convention that $0/0=0$.

\emph{Survival function estimator.}
The hazard estimate $\widehat{h}(\cdot|x)$ can be used to estimate the conditional survival function $S(\cdot|x)$ by taking products of empirical probabilities of surviving from time 0 to $t_{1}$, $t_{1}$ to $t_{2}$, and so forth up to time $t$:
\begin{equation}
\widehat{S}(t|x):=\prod_{\ell=1}^{m}\big(1-\widehat{h}(\ell|x)\big)^{\ind\{t_{\ell}\le t\}}.
\label{eq:kernel-survival-estimator}
\end{equation}
This kernel predictor is called the \emph{conditional Kaplan-Meier estimator} \citep{beran1981nonparametric} and has known finite-sample error bounds \citep{chen2019nearest}. For equation~\eqref{eq:kernel-survival-estimator}, the special case where $\mathbb{K}(x,x')=1$ for all $x,x'\in\mathcal{X}$ yields the classical Kaplan-Meier estimator \citep{kaplan1958nonparametric} that does not depend on feature vectors and is a population-level survival curve estimate:
\begin{equation}
\widehat{S}^{\text{KM}}(t)
:=
\prod_{\ell=1}^{m}
\Big(
1-\frac{\sum_{j=1}^{n}\eventInd_{j}\ind\{Y_{j}=t_{\ell}\}}{\sum_{j=1}^{n}\ind\{Y_{j}>t_{\ell-1}\}}
\Big)^{\ind\{t_{\ell}\le t\}}.
\label{eq:kaplan-meier-estimator}
\end{equation}

\paragraph{Deep kernel survival analysis}
To automatically learn the kernel function $\mathbb{K}$,
deep kernel survival analysis \citep{chen2020deep} parameterizes $\mathbb{K}$ in terms of a base neural net $\phi:\mathcal{X}\rightarrow\widetilde{\mathcal{X}}$ that maps a raw feature vector $x$ to an embedding vector $\widetilde{x}=\phi(x)\in\widetilde{\mathcal{X}}$; throughout this paper we denote embedding vectors with tildes. We always assume that the embedding space $\widetilde{\mathcal{X}}$ is a subset of~$\mathbb{R}^{d}$ and that kernel function $\mathbb{K}$ is of the form
\begin{equation}
\mathbb{K}(x,x';\phi)=K(\rho(x,x';\phi)),
\;\,
\rho(x,x';\phi)=\|\phi(x)-\phi(x')\|_2,
\label{eq:kernel-basic-form}
\end{equation}
where $K:[0,\infty)\rightarrow[0,\infty)$ is a nonincreasing function, and $\|\cdot\|_2$ denotes Euclidean distance.
For example, a common choice is the Gaussian kernel $K(u)=\exp(-\frac{u^2}{2\sigma^2})$ where $\sigma^2$ is the user-specified variance hyperparameter. Learning $\mathbb{K}$ or distance function~$\rho$ amount to learning the base neural net~$\phi$. The architecture of~$\phi$ is left for the user to specify, where standard strategies can be used. For example, when working with images, we can choose~$\phi$ to be a convolutional neural net and when working with time series, $\phi$ can be a recurrent neural net.

To prevent overfitting during training, we replace the hazard estimate $\widehat{h}(\ell|x)$ in equation~\eqref{eq:hazard-predictor} with a ``leave-one-out'' training version, and in terms of notation, we now also emphasize the dependence on the base neural net $\phi$:
\begin{equation}
\widehat{h}_{\text{train}}(\ell|i;\phi):=\frac{\sum_{j=1\text{ s.t.~}j\ne i}^{n}\mathbb{K}(X_{i},X_{j};\phi)\eventInd_{j}\ind\{Y_{j}=t_{\ell}\}}{\sum_{j=1\text{ s.t.~}j\ne i}^{n}\mathbb{K}(X_{i},X_{j};\phi)\ind\{Y_{j}>t_{\ell-1}\}}.
\label{eq:hazard-predictor-loo}
\end{equation}
In particular, the prediction for the $i$-th training point does not depend on the $i$-th training point's observed outcomes $Y_i$ and $D_i$. Similarly, we can define a leave-one-out version $\widehat{S}^{\text{train}}(t|x;\phi)$ of the survival function estimate $\widehat{S}(t|x)$ given in equation~\eqref{eq:kernel-survival-estimator}.

As for what loss function to minimize to learn the base neural net $\phi$ (and thus the kernel function $\mathbb{K}$ for use with the hazard or survival function estimators), one possibility is the negative log likelihood loss by \citet{brown1975use}:
\[
 L_{\text{NLL}}(\phi) \\
 :=\frac{1}{n}\sum_{i=1}^{n}\bigg(L_\text{BCE}(i;\phi)
 +\underbrace{\sum_{\substack{\ell=1\\\text{ s.t.~}t_{\ell}<Y_{i}}}^{m}\log\frac{1}{1-\widehat{h}_{\text{train}}(\ell|i;\phi)}}_{i\text{-th individual survives at times before }Y_{i}}\bigg),
\]
where $L_{\text{BCE}}(i;\phi)$ is the binary cross-entropy loss
\[
L_{\text{BCE}}(i;\phi):=
\eventInd_{i}\log\frac{1}{\widehat{h}_{\text{train}}(\kappa(Y_{i})|i;\phi)}
+ (1-\eventInd_{i})\log\frac{1}{1-\widehat{h}_{\text{train}}(\kappa(Y_{i})|i;\phi)},
\]
and $\kappa(Y_{i})$ denotes the sorted time index (from $1,2,\dots,m$) that time $Y_{i}$ corresponds to. Minimizing $L_{\text{NLL}}$ via minibatch gradient descent yields the deep kernel survival analysis approach by \citet{chen2020deep}.

Chen also discussed how to incorporate the \textsc{deephit} ranking loss term \citep{lee2018deephit}, which could be written
\[
L_{\text{rank}}(\phi)
:=
\frac{1}{n^2}\sum_{\substack{i=1\\\text{s.t.~}D_i=1}}^n~\sum_{\substack{j\ne i\\\text{s.t.~}Y_j>Y_i}}\exp\Big(\frac{\widehat{S}_{\text{train}}(Y_i|X_i;\phi)-\widehat{S}_{\text{train}}(Y_i|X_j;\phi)}{\sigma_{\text{rank}}}\Big),
\]
where $\sigma_{\text{rank}}>0$ is a hyperparameter; for simplicity, we use a normalization factor of $1/n^2$, which is the same normalization factor used in the \textsc{deephit} implementation that is part of the now standard \textsc{pycox} software package \citep{kvamme2019time}. Our experiments later use the following overall deep kernel survival analysis (DKSA) loss term that trades off between Brown's loss $L_{\text{NLL}}$ and the ranking loss $L_{\text{rank}}$:
\begin{equation}
L_{\text{DKSA}}(\phi)
:= \eta L_{\text{NLL}}(\phi)
+ (1-\eta) L_{\text{rank}}(\phi), \label{eq:full-dksa-loss}
\end{equation}
where $\eta\in[0,1]$ is another hyperparameter.

\emph{Implementation remarks.}
As stated, the hazard estimator \eqref{eq:hazard-predictor} as well as the conditional Kaplan-Meier estimator \eqref{eq:kernel-survival-estimator} are defined in terms of the unique observed times of death. In practice, \citet{chen2020deep} observed that discretizing time into time steps of equal size can sometimes yield more accurate survival predictions. This discretization acts as a form of regularization since we are essentially smoothing the resulting estimated hazard and conditional survival functions. Furthermore, note that the conditional Kaplan-Meier estimator is defined to be piecewise constant. However, in practice, interpolation improves test-time prediction accuracy. Chen's implementation of deep kernel survival analysis uses the constant density interpolation strategy by \citet{kvamme2019continuous}.

Separately, during training, using an infinite-support kernel function such as a Gaussian kernel works well in practice with neural net frameworks to prevent vanishing gradients. If we use a kernel function with finite support (such as the box, triangle, or \mbox{Epanechnikov} kernels), and we initialize~$\phi$ such that for each training point, no other training point is found to be similar enough as to have a nonzero kernel weight, then we would struggle to learn an improved embedding representation. As a concrete example, suppose that we use the box kernel $K(u)=\ind\{u\le1\}$, and the base neural net is simply $\phi(x)=w x$ for some scalar weight $w\in\mathbb{R}$ (i.e., $w$ is the only neural net parameter in this case). Then if during neural net training, $w$ becomes too large in absolute value (namely, $|w|>\frac{1}{\min_{i\ne j}\|X_i - X_j\|_2}$), then for all $i\ne j$,
\begin{align*}
\mathbb{K}(X_i,X_j;\phi)
&= K(\rho(X_i,X_j;\phi)) \\
&= \ind\{\|\phi(X_i)-\phi(X_j)\|_2\le1\} \\
&= \ind\{\|w X_i - w X_j\|_2\le1\} \\
&= \ind\{|w| \|X_i - X_j\|_2\le1\} \\
&= 0.
\end{align*}
When this happens, the numerator and denominator of the leave-one-out training hazard estimator \eqref{eq:hazard-predictor-loo} are both~0, so by the earlier stated convention that $0/0=0$, the overall predicted hazard is~0 for all time. This will mean that the loss does not depend on the neural net parameter~$w$ at all, so the gradient of the loss with respect to $w$ is~0. By simply using an infinite-support kernel function during training, we avoid having to deal with these zero kernel weight issues.

\section{Scalable and Interpretable Test-Time Prediction with an Accuracy Guarantee
\label{sec:scaling}}

To make a prediction for test feature vector $x$, we would in principle have to compute the similarity between~$x$ and every training feature vector, which is computationally expensive for large training datasets. To address this problem, we apply \emph{kernel netting} \citep{kpotufe2017time} to deep kernel survival analysis, obtaining a model we call a \emph{survival kernet}. Kernel netting constructs a compressed version of the training data for use at test time using the standard notion of $\varepsilon$-nets (e.g., see the textbook by \citet{vershynin2018high}). As a technical remark, kernel netting was originally developed for classification and regression; extending the proof ideas to survival analysis requires carefully combining proof ideas by \citet{kpotufe2017time} and \citet{chen2019nearest}. Separately, Kpotufe and Verma did not consider interpretability in their original kernel netting paper.

\paragraph{Sample splitting}
For our theoretical guarantee on test time prediction error later, we assume that the base neural net $\phi$ has already been trained on ``pre-training'' data $(X_{1}^{{\pre}},Y_{1}^{{\pre}},D_1^{{\pre}})$, $\dots$, $(X_{n_{{\pre}}}^{{\pre}},Y_{n_{{\pre}}}^{{\pre}},D_{n_{{\pre}}}^{{\pre}})$ that are independent of training data $(X_{1},Y_{1},D_1),\dots,(X_{n},Y_{n},D_n)$. As shorthand notation, we use $X_{1:n_{{\pre}}}^{{\pre}}:=(X_1^{{\pre}},\dots,X_{n_{{\pre}}}^{{\pre}})\in\mathcal{X}^{n_{{\pre}}}$, and similarly define $Y_{1:n_{{\pre}}}^{{\pre}}$, $D_{1:n_{{\pre}}}^{{\pre}}$, $X_{1:n}$, $Y_{1:n}$, and $D_{1:n}$. The pre-training data $(X_{1:n_{{\pre}}}^{{\pre}},Y_{1:n_{{\pre}}}^{{\pre}},D_{1:n_{{\pre}}}^{{\pre}})$ need not be sampled in the same manner as the ``proper'' training data $(X_{1:n},Y_{1:n},D_{1:n})$. In practice, one could for example take a complete training dataset and randomly split it into two portions, the first portion to treat as the pre-training data, and the second portion to treat as the proper training data. After training $\phi$ on pre-training data, we refer to the learned neural net as $\widehat{\phi}$. Our theory treats how $\widehat{\phi}$ is learned as a black box, but requires that the output space of $\widehat{\phi}$ (the embedding space $\widetilde{\mathcal{X}}$) satisfy some regularity conditions. Later in our experiments, we also intentionally try setting the pre-training and training datasets to be the same although our theory does not cover this scenario.

\paragraph{Survival kernets}
We now state how to train and make predictions with a survival \mbox{kernet}. At test time, for a test feature vector $x$, we only consider using training data within a threshold distance~$\tau$ from~$x$, where distances are computed via the learned distance $\rho(x,x';\widehat{\phi})=\|\widehat{\phi}(x)-\widehat{\phi}(x')\|_{2}$ with pre-trained neural net~$\widehat{\phi}$. In particular, at test time, we replace the function $\mathbb{K}$ from equation~\eqref{eq:kernel-basic-form} with the ``truncated'' version
\begin{equation}
\widehat{\mathbb{K}}(x,x';\widehat{\phi})=K(\rho(x,x';\widehat{\phi}~\!))\ind\{\rho(x,x';\widehat{\phi}~\!)\le \tau\}.\label{eq:truncated-kernel}
\end{equation}
From a computational viewpoint, we can take advantage of recent advances in (approximate) nearest neighbor search data structures for Euclidean distance to find neighbors of $x$ that are within distance~$\tau$ (e.g., \citealt{andoni2015practical,andoni2015optimal,malkov2020efficient,prokhorenkova2020graph}).

\emph{Training.}
The training procedure for a survival kernet works as follows, for a user-specified (approximate) Euclidean distance nearest neighbor data structure:
\begin{enumerate}[itemsep=0.2em,parsep=0em,partopsep=0em]
\item Learn base neural net $\phi$ with pre-training data $(X_{1:n_{{\pre}}}^{{\pre}},Y_{1:n_{{\pre}}}^{{\pre}},D_{1:n_{{\pre}}}^{{\pre}})$ by minimizing the deep kernel survival analysis loss $L_{\text{DKSA}}$ given in equation~\eqref{eq:full-dksa-loss} with minibatch gradient descent, \emph{without truncating the kernel function}. Denote the learned neural net as $\widehat{\phi}$.
\item For the training (and not pre-training) feature vectors $X_{1:n}$, compute the embedding vectors $\widetilde{X}_{1}=\widehat{\phi}(X_{1}),\widetilde{X}_{2}=\widehat{\phi}(X_{2}),\dots,\widetilde{X}_{n}=\widehat{\phi}(X_{n})$, and construct a Euclidean-distance-based nearest neighbor data structure using training embedding vectors $\widetilde{X}_{1:n}:=(\widetilde{X}_{1},\widetilde{X}_{2},\dots,\widetilde{X}_{n})$.
\item With the help of the nearest neighbor data structure, compute a subsample of $\widetilde{X}_{1:n}$ that we denote as $\widetilde{\mathcal{Q}}_{\varepsilon}\subseteq\widetilde{X}_{1:n}$, where $\varepsilon>0$ is an approximation parameter (as $\varepsilon\rightarrow0$, no subsampling is done, i.e., $\widetilde{\mathcal{Q}}_{\varepsilon}$ becomes~$\widetilde{X}_{1:n}$). Specifically, $\widetilde{\mathcal{Q}}_{\varepsilon}$ is an $\varepsilon$-net; $\widetilde{\mathcal{Q}}_{\varepsilon}$ can be computed efficiently with the help of a nearest neighbor data structure as follows:
\begin{itemize}
\item[(a)] Initialize $\widetilde{\mathcal{Q}}_{\varepsilon}$ to be the empty set.
\item[(b)] For $i\in\{1,2,\dots,n\}$: if $\widetilde{X}_i$'s nearest neighbor in $\widetilde{\mathcal{Q}}_{\varepsilon}$ is not within Euclidean distance $\varepsilon$, add $\widetilde{X}_i$ to $\widetilde{\mathcal{Q}}_{\varepsilon}$.
\end{itemize}
\item For each training embedding vector $\widetilde{X}_i$, assign it to a single closest exemplar point in $\widetilde{\mathcal{Q}}_{\varepsilon}$ (break ties arbitrarily). After this assignment, each exemplar point $\widetilde{q}\in\widetilde{\mathcal{Q}}_{\varepsilon}$ is assigned to a subset $\mathcal{I}_{\widetilde{q}}\subseteq\{1,2,\dots,n\}$ of training points. This could be viewed as a clustering assignment, where the training points of each cluster is represented by an exemplar.
\item For each exemplar point $\widetilde{q}\in\widetilde{\mathcal{Q}}_{\varepsilon}$, recalling that $t_1<\cdots<t_m$ are the unique times of death in the training data, compute the following summary functions for $\ell=1,\dots,m$:
\begin{equation}
\mathbf{D}_{\widetilde{q}}(\ell)
:=
\sum_{j\in\mathcal{I}_{\widetilde{q}}} \eventInd_j \ind\{Y_j=t_\ell\},
\qquad \mathbf{R}_{\widetilde{q}}^+(\ell)
:=
\sum_{j\in\mathcal{I}_{\widetilde{q}}} \ind\{Y_j\ge t_\ell\}.
\label{eq:unweighted-summaries}
\end{equation}
Note that $\mathbf{D}_{\widetilde{q}}(\ell)$ is the number of deaths at time $t_{\ell}$ across training points assigned to exemplar $\widetilde{q}$, and $\mathbf{R}_{\widetilde{q}}^+(\ell)$ is the number of these training points that are ``at risk'' (could possibly die) at time $t_{\ell}$.
\end{enumerate}

\emph{Prediction.} After training a survival kernet, prediction works as follows:
\begin{itemize}
\item[~] For test feature vector $x$, first compute the embedding vector $\widetilde{x}=\widehat{\phi}(x)$. Then form the hazard estimate
\begin{equation}
\widetilde{h}_{\widetilde{\mathcal{Q}}_{\varepsilon}}(\ell|\widetilde{x})
:=
{\displaystyle \frac{\sum_{\widetilde{q}\in \widetilde{\mathcal{Q}}_{\varepsilon}}\widetilde{\mathbb{K}}(\widetilde{x},\widetilde{q})\mathbf{D}_{\widetilde{q}}(\ell)}{\sum_{\widetilde{q}\in \widetilde{\mathcal{Q}}_{\varepsilon}}\widetilde{\mathbb{K}}(\widetilde{x},\widetilde{q})\mathbf{R}_{\widetilde{q}}^+(\ell)}}
\qquad\text{for }\ell=1,2,\dots,m,
\label{eq:kernel-netting-for-survival-analysis-hazard}
\end{equation}
where
$
\widetilde{\mathbb{K}}(\widetilde{x},\widetilde{q})
:= {K(\|\widetilde{x}-\widetilde{q}\|_{2})\ind\{\|\widetilde{x}-\widetilde{q}\|_{2}\le \tau\}}
$; as a reminder, $K:[0,\infty)\rightarrow[0,\infty)$ is a nonincreasing function (e.g., $K(u)=\exp(-u^2)$). Note that the nearest neighbor data structure constructed in step~2 of the training procedure can be used to find all exemplars in $\widetilde{\mathcal{Q}}_{\varepsilon}$ within distance~$\tau$ of $\widetilde{x}$ (this is needed to compute~$\widetilde{\mathbb{K}}$). The conditional survival function can be estimated by computing
\begin{equation}
\widetilde{S}_{\widetilde{\mathcal{Q}}_{\varepsilon}}(t|\widetilde{x}):={\prod_{\ell=1}^{m}(1-\widetilde{h}_{\widetilde{\mathcal{Q}}_{\varepsilon}}(\ell|\widetilde{x}))^{\ind\{t_{\ell}\le t\}}}
\qquad\text{for }t\ge 0.
\label{eq:kernel-netting-for-survival-analysis-survival}
\end{equation}
As a corner case, if all the kernel weights are zero for test feature vector~$x$, then we output the training set Kaplan-Meier survival function estimate~\eqref{eq:kaplan-meier-estimator} as the prediction.
\end{itemize}
\paragraph{The form of the predicted survival function and how it relates to the Kaplan-Meier estimator}
The predicted conditional survival function $\widetilde{S}_{\widetilde{\mathcal{Q}}_{\varepsilon}}$ in equation~\eqref{eq:kernel-netting-for-survival-analysis-survival} has the following interpretation. Consider an embedding vector $\widetilde{x} = \widehat{\phi}(x)$. Suppose that for some exemplar $\widetilde{q}\hspace{1pt}'\in\widetilde{\mathcal{Q}}_{\varepsilon}$, the similarity score between $\widetilde{x}$ and $\widetilde{q}\hspace{1pt}'$ is equal to $\widetilde{\mathbb{K}}(\widetilde{x},\widetilde{q}\hspace{1pt}') = 1$, whereas the similarity score between $\widetilde{x}$ and all other exemplars is~0. Put another way, the embedding vector $\widetilde{x}$ is ``purely explained'' by exemplar $\widetilde{q}\hspace{1pt}'$ and none of the other exemplars. Then in this case, equation~\eqref{eq:kernel-netting-for-survival-analysis-hazard} would become
\[
\widetilde{h}_{\widetilde{\mathcal{Q}}_{\varepsilon}}(\ell|\widetilde{x})
=
\frac{\mathbf{D}_{\widetilde{q}\hspace{1pt}'}(\ell)}{\mathbf{R}_{\widetilde{q}\hspace{1pt}'}^+(\ell)}
\qquad\text{for }\ell=1,2,\dots,m,
\]
and equation~\eqref{eq:kernel-netting-for-survival-analysis-survival} would become
\begin{equation}
\widetilde{S}_{\widetilde{\mathcal{Q}}_{\varepsilon}}(t|\widetilde{x})
={\prod_{\ell=1}^{m}(1-\widetilde{h}_{\widetilde{\mathcal{Q}}_{\varepsilon}}(\ell|\widetilde{x}))^{\ind\{t_{\ell}\le t\}}}
={\prod_{\ell=1}^{m}\Big(1-\frac{\mathbf{D}_{\widetilde{q}\hspace{1pt}'}(\ell)}{\mathbf{R}_{\widetilde{q}\hspace{1pt}'}^+(\ell)}\Big)^{\ind\{t_{\ell}\le t\}}}
\qquad\text{for }t\ge 0,
\label{eq:cluster-KM}
\end{equation}
which is precisely the Kaplan-Meier estimator (equation~\eqref{eq:kaplan-meier-estimator}) restricted to training points that have been assigned to the cluster of exemplar~$\widetilde{q}\hspace{1pt}'$ (in Step~4 of the survival kernet training procedure).

In general, a data point $x$ with embedding vector $\widetilde{x}$ can have a similarity score that is nonzero for multiple exemplars. This is akin to how in topic modeling, a data point could have nonzero weights for multiple topics. In topic modeling, a key visualization strategy is to focus on a single topic at a time and look at the distribution of features (i.e., words) that show up for that topic. This visualization strategy could be thought of as reasoning about what a data point would look like if it were to be purely explained by a single topic. In a similar vein, for our survival analysis setting, our visualization strategy later is based on reasoning about what a data point would look like if it were to be purely explained by a single cluster or exemplar (again, in our setting, a cluster directly corresponds to an exemplar). Equation~\eqref{eq:cluster-KM} reveals that a data point that is purely explained by exemplar $\widetilde{q}\hspace{1pt}'$ would have a predicted conditional survival function that is just the Kaplan-Meier estimator restricted to data points in the cluster of $\widetilde{q}\hspace{1pt}'$. Note that plotting such an estimated survival function is standard in survival analysis and the resulting plots are called \emph{Kaplan-Meier curves}. Our visualization strategy later plots the Kaplan-Meier curve specific to each cluster/exemplar.

\paragraph{Connections to deep kernel survival analysis, to the original kernel netting, and to the original Kaplan-Meier estimator}
For a survival kernet model, note that if $\varepsilon=0$ (i.e., the $\varepsilon$-net consists of all training embedding vectors), and the pre-training and training sets are set to be identical, then the model becomes an approximate version of the original deep kernel survival analysis model by \citet{chen2020deep}. In particular, the approximation comes from the kernel function being truncated (set to be 0 beyond the critical threshold distance~$\tau$) at test time.

Separately, consider if the base neural net is the identity function $\phi(x)=x$ (which would require no pre-training data to learn, so we skip steps~1 and~2 of the survival kernet training procedure). Then in this case, we obtain a non-neural-net extension of the original kernel netting procedure by \citet{kpotufe2017time} to the survival analysis setting.

Lastly, if $\varepsilon\rightarrow\infty$, then the $\varepsilon$-net $\widetilde{\mathcal{Q}}_{\varepsilon}$ would consist of a single exemplar that summarizes the whole training (and not pre-training) data. Put another way, there would only be a single cluster that consists of all the training data. If $\tau\rightarrow\infty$, then the prediction for any test point would simply be the original Kaplan-Meier survival function \citep{kaplan1958nonparametric} fitted to the training data. The prediction would not actually depend on the base neural net $\phi$ in this case.

\paragraph{Outline for the remainder of this section}
In the remainder of this section, we discuss three main topics. First, we discuss model interpretability in Section~\ref{sec:model-interpretability}. The key idea here has to do with the interpretation we described above of how if a data point were to be purely explained by a single exemplar $\widetilde{q}\hspace{1pt}'$, then its predicted conditional survival function is just the Kaplan-Meier estimator restricted to training points in the cluster of $\widetilde{q}\hspace{1pt}'$.

Next, we provide an overview of the theory we have developed for survival kernets in Section~\ref{sec:theory}, starting from assumptions and the generalization error used to stating the main finite-sample guarantee and some interpretations and implications. In a nutshell, our theory says that if the embedding vectors are in some sense ``nice'' (satisfying standard theoretical assumptions made in nonparametric estimation and survival analysis), then a survival kernet estimates the conditional survival function arbitrarily accurately (up to some pre-specified time horizon) with high probability as the amount of training data grows large. The error bound we obtain has an order of growth that, ignoring a log factor, is optimal.

Lastly, we present a variant of our training procedure in Section~\ref{sec:summary-fine-tuning} that turns out to yield significant accuracy improvements in practice although it is not covered by our theoretical analysis. This variant adds a step at the end of survival kernet training that fine-tunes the summary functions from training step~5. Specifically, note that the summary functions $\mathbf{D}_{\widetilde{q}}$ and $\mathbf{R}_{\widetilde{q}}^+$ in equation~\eqref{eq:unweighted-summaries} are constructed from training data, and it could be that these are noisy or inaccurate. After training the base neural net $\phi$ and treating it as fixed, we could then set up a fine-tuning step in which we learn $\mathbf{D}_{\widetilde{q}}$ and $\mathbf{R}_{\widetilde{q}}^+$ in a neural net framework.

\subsection{Model Interpretability}
\label{sec:model-interpretability}

\paragraph{Visualizing clusters}
Recall that each exemplar $\widetilde{q}\hspace{1pt}'\in\widetilde{\mathcal{Q}}_{\varepsilon}$ corresponds to a cluster of training points, i.e., there is a one-to-one correspondence between exemplars and clusters. For tabular data, we can create a heatmap visualization to help us quickly identify how clusters differ in terms of whether specific feature values are more prominent for specific clusters. In particular, we set the rows of the heatmap to correspond to different features, the columns to correspond to different clusters, and the heatmap intensity values to correspond to the fraction of points in a cluster with a specific feature value. As a concrete example, we show this heatmap visualization in Figure~\ref{fig:support}\subref{subfig:support-heatmap} for
a dataset on predicting time until death of hospitalized patients from the Study to Understand Prognoses, Preferences, Outcomes, and Risks of Treatment (\textsc{support}) \citep{knaus1995support}. For instance, we see that the leftmost cluster (column) in the heatmap corresponds to patients who often have metastatic cancer, who are mostly between 49.28 and 76.02 years of age, and who have at least one comorbidity. In contrast, the rightmost cluster corresponds to young patients without cancer. Our proposed heatmap visualization is a more visual way of conveying information like that of \citet{chapfuwa2020survival} in their Table~2.

\begin{figure}[!t]
\centering
\begin{subfigure}[b]{0.48\linewidth}
\centering
\includegraphics[scale=.42]{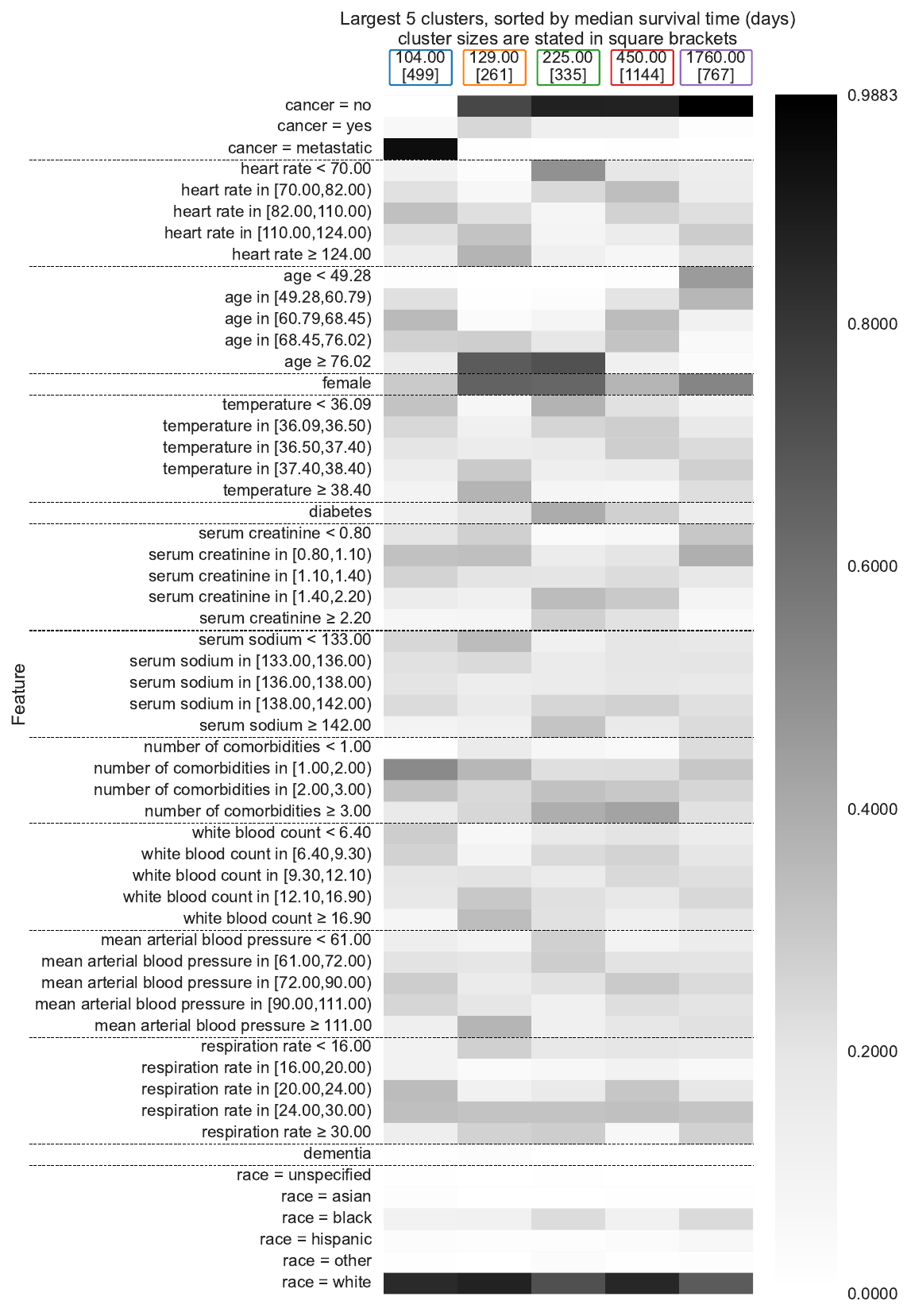}
\caption{}
\label{subfig:support-heatmap}
\end{subfigure}
~~~
\begin{subfigure}[b]{0.48\linewidth}
\centering
\hspace{1em}
\includegraphics[scale=.42]{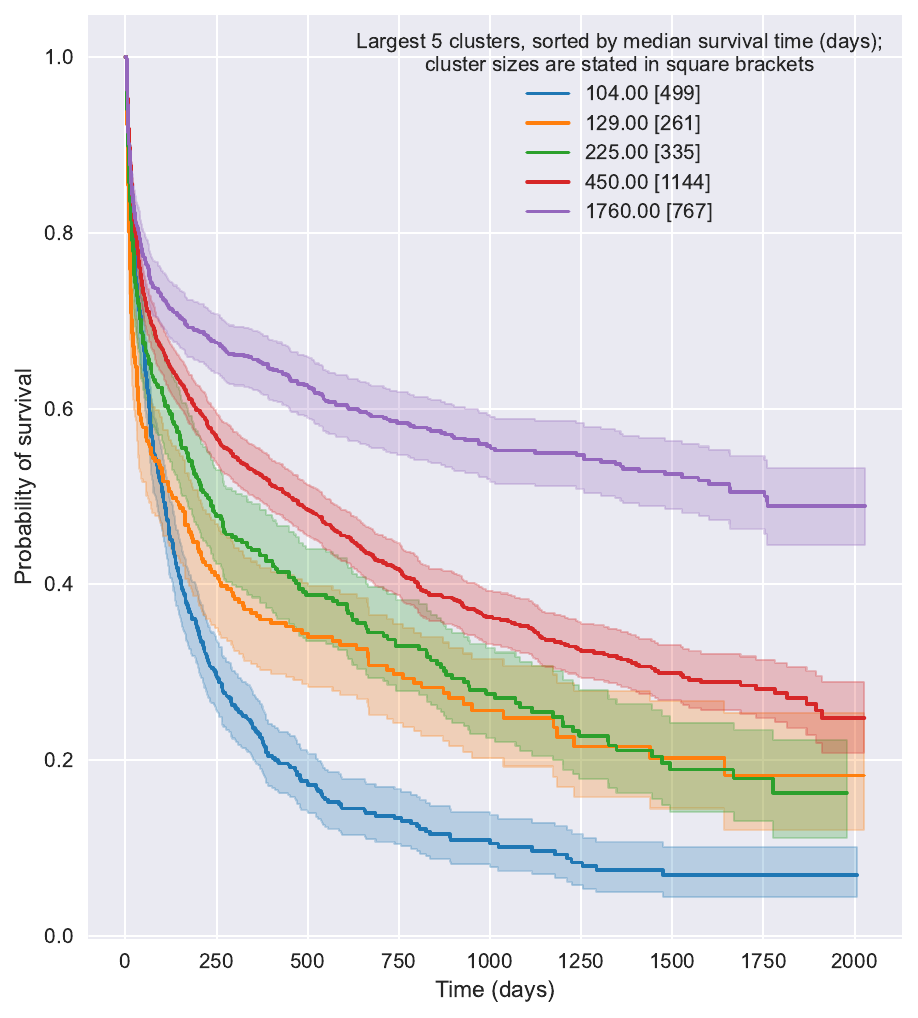}
\caption{}
\label{subfig:support-km}
\end{subfigure}
\caption{Visualization of the largest 5 clusters found by a survival kernet model trained on the \textsc{support} dataset (we limit the number of clusters shown for ease of exposition and to prevent the plots from being too cluttered); more information on how the model is trained is in Section~\ref{sec:experiments}. Panel (a) shows a heatmap visualization that readily provides information on how the clusters are different, highlighting feature values that are prominent for specific clusters; the dotted horizontal lines separate features that correspond to the same underlying variable. Panel (b) shows Kaplan-Meier survival curves with 95\% confidence intervals for the same clusters as in panel (a); the x-axis measures the number of days since a patient entered the study.}
\label{fig:support}
\end{figure}

Moreover, per exemplar $\widetilde{q}\hspace{1pt}'\in\widetilde{\mathcal{Q}}_{\varepsilon}$, we can compute its corresponding cluster's Kaplan-Meier survival curve using equation~\eqref{eq:cluster-KM}. As we pointed out previously, this predicted survival curve is precisely what the survival kernet model would predict for a data point that is purely explained by $\widetilde{q}\hspace{1pt}'$ and none of the other exemplars. We can plot the survival curves of different exemplars as shown in Figure~\ref{fig:support}\subref{subfig:support-km}, where we have also included 95\% confidence intervals using the standard exponential Greenwood formula \citep{kalbfleisch1980statistical}; the x-axis is the time since a patient entered the study.

Note that the time (x-axis value) at which a Kaplan-Meier survival curve crosses probability 1/2 (y-axis value) corresponds to a median survival time estimate (since half the patients survived up to this time and the rest survive beyond this time). Thus, per cluster, we can obtain a median survival time estimate (if the cluster's Kaplan-Meier curve never crosses probability 1/2, then the median survival time is greater than the largest observed time for the cluster). In fact, for the heatmap in Figure~\ref{fig:support}\subref{subfig:support-heatmap}, we have sorted the clusters (i.e., the columns) by median survival time from smallest to largest.

Comparing Kaplan-Meier curves across different groups of individuals is standard practice in survival analysis and gives a quick way to see which group is better or worse off over time. For instance, in this case, the purple cluster (with median survival time $1760$ days) generally has higher survival probability across time compared to the four other clusters that have been plotted. Meanwhile, initially the orange cluster (median survival time 129 days) appears worse off than the green cluster (median survival time 225 days) but then these two clusters' Kaplan-Meier curves start heavily overlapping after around 1200 days since patients entered the study.

We remark that the SUPPORT dataset has patients within nine disease groups, of which three directly are related to cancer (colon cancer, lung cancer, multiple organ system failure with cancer). For these three disease groups, age and comorbidites are known to be predictors of survival (e.g., \citealt{van2015impact,asmis2008age,frey2007co}). In this sense, the heatmap appears to surface some trends that agree with existing clinical literature. There are of course other variables present too. Overall, our proposed heatmap visualization and accompanying survival curve plot is meant as a debugging tool, to help a modeler see how the clusters relate to raw features and also to survival time distributions, revealing possible associations that may warrant additional investigation and that could be discussed with domain experts.

We provide a few technical details regarding how we generated Figure~\ref{fig:support}\subref{subfig:support-heatmap}. Only for visualization purposes, continuous features have been discretized into 5 equal-sized bins (fewer bins are used if there are not enough data points in a cluster, or if the 20/40/60/80 percentile threshold values for a continuous feature are not all unique). The survival kernet models themselves do \emph{not} require continuous features to be discretized first. Moreover, we have sorted the features (i.e., rows of the heatmap) in the following manner: per row, we compute the intensity range (i.e., maximum minus minimum intensity values), and then we sort the rows from largest to smallest intensity range, with the constraint that rows corresponding to the same underlying variable (the same continuous feature that has been discretized, or the same categorical variable) are still grouped together. For example, the reason why the variable ``cancer'' shows up first in Figure~\ref{fig:support}\subref{subfig:support-heatmap} is that among the three different cancer rows, one of them has the highest intensity range across all heatmap rows.

\paragraph{Data-point-specific information}
Next, we observe that for any test feature vector $x$ with embedding vector $\widetilde{x}=\widehat{\phi}(x)$, its hazard estimate is given by equation~\eqref{eq:kernel-netting-for-survival-analysis-hazard}, reproduced below for ease of exposition:
\[
\widetilde{h}_{\widetilde{\mathcal{Q}}_{\varepsilon}}(\ell|\widetilde{x})
=
{\displaystyle \frac{\sum_{\widetilde{q}\in \widetilde{\mathcal{Q}}_{\varepsilon}}\widetilde{\mathbb{K}}(\widetilde{x},\widetilde{q})\mathbf{D}_{\widetilde{q}}(\ell)}{\sum_{\widetilde{q}\in \widetilde{\mathcal{Q}}_{\varepsilon}}\widetilde{\mathbb{K}}(\widetilde{x},\widetilde{q})\mathbf{R}_{\widetilde{q}}^+(\ell)}}.
\]
In the numerator and denominator summations, the only exemplars $\widetilde{q}$ that contribute to the calculation are ones for which $\widetilde{\mathbb{K}}(\widetilde{x},\widetilde{q})$ is positive. Again, since there is a one-to-one correspondence between exemplars and clusters, this means that each test feature vector's hazard is modeled by only a subset of the clusters.

For any test feature vector~$x$, we can readily determine which exemplars/clusters could possibly contribute to the prediction for $x$ by figuring out which $\widetilde{q}\in\widetilde{\mathcal{Q}}_{\varepsilon}$ satisfy $\widetilde{\mathbb{K}}(\widetilde{x},\widetilde{q})>0$. Of course, how large these weights are can also give a sense of the relative importance of the clusters for $x$. We could, for instance, make the same plots as in Figure~\ref{fig:support} but only show the clusters that have nonzero weight for a specific test feature vector $x$.

\subsection{Theory of Survival Kernets}
\label{sec:theory}

We now provide an overview of our theoretical analysis of survival kernets.
We begin with assumptions on the embedding space that are standard in nonparametric estimation theory (Section~\ref{sec:assumptions-emb-space}), although these assumptions are typically imposed on the raw feature space rather than the embedding space. We then state our assumption on how the embedding space relates to survival and censoring times (Section~\ref{sec:assumptions-emb-to-label}).
All proofs are in Appendix~\ref{sec:proofs}.

\subsubsection{The Embedding Space}
\label{sec:assumptions-emb-space}

We assume that the raw feature vectors are sampled i.i.d.~from a marginal distribution $\mathbb{P}_{\mathsf{X}}$. As we treat the pre-trained neural net $\widehat{\phi}$ as fixed, then the random embedding vector $\widetilde{X}=\widehat{\phi}(X)$ is sampled from some distribution $\mathbb{P}_{\widetilde{\mathsf{X}}}$ instead. We require $\mathbb{P}_{\widetilde{\mathsf{X}}}$ to satisfy some mild regularity conditions.

As an example, we remark that an embedding space that is uniform over a unit hypersphere (i.e., embedding vectors are Euclidean vectors with norm~1) satisfies all the assumptions of this subsection. For ways to encourage the learned embedding space to be uniform over the unit hypersphere, see the papers by \citet{wang2020understanding} and \citet{liu2021learning}.

Our theory requires a technical assumption that ensures that the event that $\widetilde{X}$ lands within distance $\tau$ of an embedding vector $\widetilde{x}\in\widetilde{\mathcal{X}}$ has a well-defined probability. This event could be phrased as~$\widetilde{X}$ landing in the closed ball of radius~$\tau$ centered at~$\widetilde{x}$, denoted as $B(\widetilde{x},\tau):={\{\widetilde{x}'\in\widetilde{\mathcal{X}}:\|\widetilde{x}-\widetilde{x}'\|_{2}\le \tau\}}$.

\medskip
\noindent
\textbf{Assumption $\mathbf{A}^{\normalfont \textit{technical}}$.}
\emph{Distribution $\mathbb{P}_{\widetilde{\mathsf{X}}}$ is a Borel probability measure with compact support ${\widetilde{\mathcal{X}}\subseteq\mathbb{R}^d}$.}

\medskip
Compactness of $\widetilde{\mathcal{X}}$ eases the exposition. Our results trivially extend to the case where embedding vectors sampled from $\mathbb{P}_{\widetilde{\mathsf{X}}}$ land in a compact region (where our theory applies) with probability at least $1-\delta$ for some $\delta\ge0$, and otherwise when the embedding vectors land outside the compact region, we tolerate a worst-case prediction error with probability~$\delta$.

We use the standard notion of covers, packings, and nets (all specialized to Euclidean distance).
\begin{definition}
For any radius $\varepsilon>0$:
\begin{itemize}[itemsep=0.2em,parsep=0em,partopsep=0em]
\item A subset $\widetilde{\mathcal{Q}}$ of $\widetilde{\mathcal{X}}$ is called an $\varepsilon$\emph{-cover} of $\widetilde{\mathcal{X}}$ if for any $\widetilde{x}\in\widetilde{\mathcal{X}}$, there exists $\widetilde{q}\in\widetilde{\mathcal{Q}}$ such that $\|\widetilde{x}-\widetilde{q}\|_{2}\le\varepsilon$.
\item A subset $\widetilde{\mathcal{Q}}$ of $\widetilde{\mathcal{X}}$ is called an $\varepsilon$\emph{-packing} of $\widetilde{\mathcal{X}}$ if for any two distinct $\widetilde{q}$ and $\widetilde{q}\hspace{1pt}'$ in $\widetilde{\mathcal{Q}}$, we have ${\|\widetilde{q}-\widetilde{q}\hspace{1pt}'\|_{2}>\varepsilon}$.
\item A subset $\widetilde{\mathcal{Q}}$ of $\widetilde{\mathcal{X}}$ is called an $\varepsilon$\emph{-net} of $\widetilde{\mathcal{X}}$ if $\widetilde{\mathcal{Q}}$ is both an $\varepsilon$-cover and an $\varepsilon$-packing of~$\widetilde{\mathcal{X}}$.
\end{itemize}
\end{definition}
Next, our theory makes use of the standard complexity notions of covering numbers and intrinsic dimension to describe the embedding space, where lower complexity will correspond to tighter error bounds. As an illustrative example, we explain how these behave when the embedding space is the unit hypersphere $\mathbb{S}^{d-1}=\{\widetilde{x}\in\mathbb{R}^{d}:\|\widetilde{x}\|_{2}=1\}$ (with $d\ge2$).
\begin{definition}
The \emph{$\varepsilon$-covering number} of $\widetilde{\mathcal{X}}$ (for Euclidean distance) is the smallest size possible for an $\varepsilon$-cover for $\widetilde{\mathcal{X}}$. We denote this number by $N(\varepsilon;\widetilde{\mathcal{X}})$.
\end{definition}
The embedding space $\widetilde{\mathcal{X}}$ being compact implies that for every $\varepsilon>0$, we have $N(\varepsilon;\widetilde{\mathcal{X}})<\infty$, and moreover, that $\widetilde{\mathcal{X}}$ has a finite diameter $\Delta_{\widetilde{\mathcal{X}}}:={\max_{\widetilde{x},\widetilde{x}'\in\widetilde{\mathcal{X}}}\|\widetilde{x}-\widetilde{x}'\|_{2}}<\infty$.
For example, the unit hypersphere $\mathbb{S}^{d-1}$ clearly has diameter 2. Moreover, its covering number is bounded as follows.
\begin{claim}[Follows from Corollary 4.2.13 by \citet{vershynin2018high}]\label{claim:unit-hypersphere-covering-number}
For any $\varepsilon>0$, the unit hypersphere has covering number
$
{N(\varepsilon;\mathbb{S}^{d-1})}\le(\frac{4}{\varepsilon}+1)^{d}.
$
\end{claim}
Covering numbers provide a way to derive error bounds. Another way is to assume a known ``intrinsic dimension'' of the embedding space. We use the following notion of intrinsic dimension.

\medskip
\noindent
\textbf{Assumption $\mathbf{A}^{\normalfont \textit{intrinsic}}$.}
\emph{There exist a positive integer $d'>0$ (called the \emph{intrinsic dimension} of $\mathbb{P}_{\widetilde{\mathsf{X}}}$) and positive constants $C_{d'}$ and $r^*$ such that for any $\widetilde{x}\in\widetilde{\mathcal{X}}$ and $r\in(0,r^*]$, we have
$
\mathbb{P}_{\widetilde{\mathsf{X}}}(B(\widetilde{x},r))\ge C_{d'}r^{d'}.
$}

\medskip
When embedding vectors are on the unit hypersphere, then under a mild regularity condition, the embedding space has intrinsic dimension $d-1$.
\begin{claim}\label{claim:uniform-unit-hypersphere-intrinsic-dimension}
If $\mathbb{P}_{\widetilde{\mathsf{X}}}$ has a probability density function that is 0 outside of $\mathbb{S}^{d-1}$ and lower-bounded by a constant $c_{\min}>0$ over $\mathbb{S}^{d-1}$, then the embedding space has intrinsic dimension $d-1$.
\end{claim}
We remark that for embedding vectors on the hypersphere, one desirable property is that they are uniformly distributed \citep{wang2020understanding,liu2021learning}, which is of course a special case of Claim~\ref{claim:uniform-unit-hypersphere-intrinsic-dimension}.

\subsubsection{Relating Embedding Vectors to Survival and Censoring Times}
\label{sec:assumptions-emb-to-label}

Next, we impose a survival analysis assumption by \citet{chen2019nearest}, which is a slight variant on an earlier assumption by \citet{dabrowska1989uniform}.

\medskip
\noindent
\textbf{Assumption $\mathbf{A}^{\normalfont \textit{survival}}$.}
\emph{In addition to~Assumption $\mathbf{A}^{\text{technical}}$, we further assume that the conditional survival, censoring, and observed time distributions $\mathbb{P}_{\mathsf{T}|\widetilde{\mathsf{X}}}$, $\mathbb{P}_{\mathsf{C}|\widetilde{\mathsf{X}}}$, and $\mathbb{P}_{\mathsf{Y}|\widetilde{\mathsf{X}}}$ correspond to continuous random variables with PDF's $\widetilde{f}_{\mathsf{T}|\widetilde{\mathsf{X}}}$, $\widetilde{f}_{\mathsf{C}|\widetilde{\mathsf{X}}}$, and $\widetilde{f}_{\mathsf{Y}|\widetilde{\mathsf{X}}}$. The conditional survival function that we aim to predict is precisely $\widetilde{S}(t|\widetilde{x}):=\int_t^\infty \widetilde{f}_{\mathsf{T}|\widetilde{\mathsf{X}}}(s|\widetilde{x})ds$. To be able to estimate this function for any $\widetilde{x}\in\widetilde{\mathcal{X}}$, we assume that censoring does not almost surely happen for every $\widetilde{x}\in\widetilde{\mathcal{X}}$. Moreover, we assume the following:
\begin{itemize}[leftmargin=1.7em,itemsep=0.2em,parsep=0em,partopsep=0em]
\item[(a)] (Observed time distribution has large enough tails) For a user-specified time $t_{\text{horizon}}>0$, there exists a constant $\theta>0$ such that $\int_{t_{\text{horizon}}}^\infty f_{\mathsf{Y}|\widetilde{\mathsf{X}}}(s|\widetilde{x}) ds \ge\theta$ for all $\widetilde{x}\in\widetilde{\mathcal{X}}$.
\item[(b)] (Smoothness: embedding vectors that are close by should have similar survival time distributions and also similar censoring time distributions) PDF's $\widetilde{f}_{\mathsf{T}|\widetilde{\mathsf{X}}}$ and $\widetilde{f}_{\mathsf{C}|\widetilde{\mathsf{X}}}$ are assumed to be H\"{o}lder continuous with respect to embedding vectors, i.e., there exist constants $\lambda_{\mathsf{T}}>0$, $\lambda_{\mathsf{C}}>0$, and $\alpha>0$ such that for all $\widetilde{x},\widetilde{x}'\in\widetilde{\mathcal{X}}$,
\begin{align*}
|\widetilde{f}_{\mathsf{T}|\widetilde{\mathsf{X}}}(t|\widetilde{x})-\widetilde{f}_{\mathsf{T}|\widetilde{\mathsf{X}}}(t|\widetilde{x}')| \le\lambda_{\mathsf{T}}\|\widetilde{x}-\widetilde{x}'\|_{2}^{\alpha}, \\
\quad
|\widetilde{f}_{\mathsf{C}|\widetilde{\mathsf{X}}}(t|\widetilde{x})-\widetilde{f}_{\mathsf{C}|\widetilde{\mathsf{X}}}(t|\widetilde{x}')| \le\lambda_{\mathsf{C}}\|\widetilde{x}-\widetilde{x}'\|_{2}^{\alpha}.
\end{align*}
\end{itemize}}

\medskip
For any estimate $\widetilde{G}$ of the true conditional survival function~$\widetilde{S}$, our theory uses the generalization error
\[
\mathcal{L}_{\text{survival}}(\widetilde{G},\widetilde{S})
:=\mathbb{E}_{\widetilde{X}\sim\mathbb{P}_{\widetilde{\mathsf{X}}}}
  \Big[\frac{\int_0^{t_{\text{horizon}}}\big(\widetilde{G}(t|\widetilde{X})-\widetilde{S}(t|\widetilde{X})\big)^2 dt}{t_{\text{horizon}}}\Big].
\]
We only aim to accurately estimate $\widetilde{S}(t|\widetilde{x})$ up to the user-specified time $t_{\text{horizon}}$ (that appears in Assumption~$\mathbf{A}^{\text{survival}}$(a)). We shall plug in the survival kernet conditional survival function estimate $\widetilde{S}_{\widetilde{\mathcal{Q}}_{\varepsilon}}$ (from step~6 of the survival kernet procedure) in place of $\widetilde{G}$.

\subsubsection{Theoretical Guarantee on Prediction Accuracy\label{subsec:guarantee}}

We are now ready to state the main theoretical result of the paper, which states a generalization error bound for survival kernets.

\begin{theorem}
\label{thm:main-result}
Suppose that Assumptions $\mathbf{A}^{\text{technical}}$, $\mathbf{A}^{\text{intrinsic}}$, and $\mathbf{A}^{\text{survival}}$ hold, and we train a survival \mbox{kernet} with $\varepsilon=\beta\tau$ when constructing $\widetilde{\mathcal{Q}}_{\varepsilon}$, where $\beta\in(0,1)$ and $\tau>0$ are user-specified parameters, and the truncated kernel used is of the form in equation~\eqref{eq:truncated-kernel}. Let $\Psi=\min\{N({\textstyle \frac{(1-\beta)\tau}{2}};\widetilde{\mathcal{X}}), \frac{1}{C_{d'}((1-\beta)\tau)^{d'}}\}$, where $d'$ is the intrinsic dimension.
Then when $n\ge\mathcal{O}(\frac{K(0)}{K(\tau)((1-\beta)\tau)^{d'}})$,
\[
 \mathbb{E}_{\widetilde{X}_{1:n},Y_{1:n},D_{1:n}}[ \mathcal{L}_{\text{survival}}(\widetilde{S}_{\widetilde{\mathcal{Q}}_{\varepsilon}},\widetilde{S})]
 \le\widetilde{\mathcal{O}}\Big(\frac{1}{n}\cdot{}\frac{K^4(0)}{K^4(\tau)}\cdot\Psi\Big)
 + (1+\beta)^{2\alpha}\tau^{2\alpha}\cdot\mathcal{O}\Big(\frac{K^2(0)}{K^2(\tau)}\Big),
\]
where $\widetilde{\mathcal{O}}$ ignores log factors.
\end{theorem}
The error bound consists of two terms, which correspond to variance and bias respectively. When the embedding space has intrinsic dimension $d'$, this error bound is, up to a log factor, optimal in the case where $K(u)=1$ for $u\ge0$, and $\tau=\widetilde{\mathcal{O}}(n^{-1/(2\alpha+d')})$. In this scenario, the survival error bound is $\widetilde{\mathcal{O}}(n^{-2\alpha/(2\alpha+d')})$, which matches the lower bound for conditional CDF estimation by \citet{chagny2014adaptive} (this is a special case of the survival analysis setup in which there is no censoring).
Of course, in practice, $K$ is typically chosen to decay, so $\tau$ should be chosen to not be too large due to the error bound's dependence on the ratio $K(0)/K(\tau)$.

\paragraph{High-level proof overview}
We combine proof ideas of \citet{kpotufe2017time} and \citet{chen2019nearest}. What these authors consider to be the raw feature space is what we take to be the embedding space (the output space of the already trained neural net $\widehat{\phi}$). In this high-level proof overview, we only highlight some of the key ideas of our proof.

The bulk of the proof treats the test point $x$ as fixed. Its corresponding embedding vector is $\widetilde{x}=\widehat{\phi}(x)$. Recall that $\widetilde{\mathcal{Q}}_{\varepsilon}$ is the set of exemplars in the embedding space, and that step 4 of the survival kernet training procedure assigns each exemplar point $\widetilde{q}\in\widetilde{\mathcal{Q}}_{\varepsilon}$ to a subset $\mathcal{I}_{\widetilde{q}}\subset\{1,2,\dots,n\}$ of training points. In particular, every training point is assigned to exactly one exemplar so that
\[
\underbrace{\{1,2,\dots,n\}}_{\text{training data indices}}=\coprod_{\widetilde{q}\in\widetilde{Q}_{\varepsilon}}\mathcal{I}_{\widetilde{q}},
\]
where ``$\coprod$'' denotes a disjoint union. Following \citet{kpotufe2017time}, an initial key observation is that the only training points that can possibly contribute to the prediction for $\widetilde{x}$ (i.e., they have nonzero truncated kernel weight) are the ones with indices in
\[
\mathcal{N}_{\widetilde{Q}_{\varepsilon}}(\widetilde{x}):=\coprod_{\widetilde{q}\in\widetilde{Q}_{\varepsilon}\text{ s.t.~}\|\widetilde{x}-\widetilde{q}\|_{2}\le\tau}\mathcal{I}_{\widetilde{q}}.
\]
Kpotufe and Verma's analysis focuses on regression in which they then proceed to analyzing a kernel-weighted average of the labels of the training points in $\mathcal{N}_{\widetilde{Q}_{\varepsilon}}(\widetilde{x})$. In the survival analysis setting, the challenge is that prediction is more complicated than taking a kernel-weighted average of $Y_{i}$'s.

To address predicting the conditional survival function $\widetilde{S}(\cdot|\widetilde{x})$, the proof strategy by \citet{chen2019nearest} uses a Taylor expansion to decompose the predicted log of the estimated conditional survival function $\widetilde{S}_{\widetilde{\mathcal{Q}}_{\varepsilon}}(\cdot|\widetilde{x})$ into three terms:
\[
\log\widetilde{S}_{\widetilde{\mathcal{Q}}_{\varepsilon}}(t|\widetilde{x})=W_{1}(t|\widetilde{x})+W_{2}(t|\widetilde{x})+W_{3}(t|\widetilde{x}),
\]
where $W_{1}(t|\widetilde{x})$ is a kernel-weighted average of a hypothetical label that we do not directly observe (we provide more details on this shortly), $W_{2}(t|\widetilde{x})$ is an error term for how well we can approximate the CDF of the distribution of observed times specific to test embedding vector $\widetilde{x}$, and lastly $W_{3}(t|\widetilde{x})$ is a higher-order Taylor series error term. In more detail, $W_{1}(t|\widetilde{x})$ is a kernel-weighted average, where the $i$-th training point's label is taken to be $-\frac{D_{i}\ind\{Y_{i}\le t\}}{\widetilde{S}(Y_{i}|\widetilde{x})}$; note that this label knows the true conditional survival function $\widetilde{S}(\cdot|\widetilde{x})$. This hypothetical label only shows up in the proof and is not actually needed when implementing the model. Similarly, the CDF estimation problem that $W_{2}(t|\widetilde{x})$ relates to also only shows up in the proof and does not need to be implemented.

\citet{chen2019nearest} does not use a $\varepsilon$-net and his analysis could be thought of as the case where we take $\varepsilon\rightarrow0$ for the $\varepsilon$-net (so that every training embedding vector is its own exemplar). For this special case, Chen showed sufficient conditions that ensure that $W_{1}(t|\widetilde{x})\rightarrow\log\widetilde{S}(t|\widetilde{x})$, $W_{2}(t|\widetilde{x})\rightarrow0$, and $W_{3}(t|\widetilde{x})\rightarrow0$. Thus, $\log\widetilde{S}_{\widetilde{\mathcal{Q}}_{\varepsilon}}(t|\widetilde{x})\rightarrow\log\widetilde{S}(t|\widetilde{x})$, which can be turned into a statement on the difference ${\widetilde{S}_{\widetilde{\mathcal{Q}}_{\varepsilon}}(t|\widetilde{x})-\widetilde{S}(t|\widetilde{x})}$, without logs.

Extending Chen's analysis to account for $\varepsilon$-nets requires a bit of extra bookkeeping and, ultimately, the two main changes are as follows (that at a high-level relate to variance and bias of the survival kernet estimator):
\begin{enumerate}
\item Chen's original analysis (restated to use our notation) regularly involves the quantity $n\mathbb{P}_{\widetilde{\mathsf{X}}}(\mathcal{B}(\widetilde{x},\tau))$, which is the expected number of training points that land within distance $\tau$ of $\widetilde{x}$. Again, keep in mind that the setup Chen considers could be thought of as every training point being its own exemplar. Thus, the exemplars/training points that could contribute to the prediction of $\widetilde{x}$ is indeed any training point whose embedding vector lands within distance $\tau$ of $\widetilde{x}$. Thus, when the quantity $n\mathbb{P}_{\widetilde{\mathsf{X}}}(\mathcal{B}(\widetilde{x},\tau))$ is larger, then we should have more training points that contribute to the prediction of $\widetilde{x}$, thus reducing the variance in the estimated conditional survival function.

However, once we use $\varepsilon$-nets with $\varepsilon=\beta\tau$, the complication is that now, the training points that could contribute to the prediction of $\widetilde{x}$ need to be assigned to exemplars that are within distance $\tau$ of $\widetilde{x}$. It is possible that a training point $x'$ with embedding vector $\widetilde{x}'$ satisfies $\|\widetilde{x}'-\widetilde{x}\|\le\tau$ but $\widetilde{x}'$ could be assigned to an exemplar that is farther away from distance $\tau$ from $\widetilde{x}$ so that training point $x'$ does not actually contribute to the prediction for $\widetilde{x}$. Here, the key observation (which shows up in Kpotufe and Verma's regression analysis but is more generally a standard $\varepsilon$-net proof technique) is that it suffices to consider training points whose embedding vectors are within distance $(1-\beta)\tau$ of $\widetilde{x}$. These particular embedding vectors will for sure be assigned to exemplar(s) within distance $\tau$ of $\widetilde{x}$: any such embedding vector is within distance ${(1-\beta)\tau}$ from $\widetilde{x}$ and can be assigned to an exemplar that is at most an additional distance $\beta\tau$ away from $\widetilde{x}$ since the set of exemplars is an $(\beta\tau)$-net.
Thus, whereas Chen's original analysis depended on the quantity $n\mathbb{P}_{\widetilde{\mathsf{X}}}(\mathcal{B}(\widetilde{x},\tau))$, now we instead replace this quantity with $n\mathbb{P}_{\widetilde{\mathsf{X}}}(\mathcal{B}(\widetilde{x},(1-\beta)\tau))$ as we want the expected number of training embedding vectors within distance $(1-\beta)\tau$ (which get assigned to exemplar(s) within distance $\tau$ of $\widetilde{x}$) to be large.
\item The other main change is that in Chen's original analysis, there is a bias calculation involving H\"{o}lder's inequality. Restated using our notation, the idea is that a training embedding vector that is a distance $\tau$ away from $\widetilde{x}$ incurs a bias in estimating $W_{1}(t|\widetilde{x})$ that scales roughly as $\tau^{2\alpha}$ (specifically, see Lemma F.4 of \citet{chen2019nearest}; note that Chen uses sup norm error so the exponent is just $\alpha$ instead of $2\alpha$ whereas in this paper, we use squared error which adds the factor of 2). However, now that we are using an $\varepsilon$-net, the prediction for $\widetilde{x}$ could depend on an exemplar that is a distance $\tau$ away from $\widetilde{x}$, and this exemplar could have a training embedding vector assigned to it that is a worst case distance of $\tau+\beta\tau=(1+\beta)\tau$ away from $\widetilde{x}$. Thus, the worst-case bias incurred is now $((1+\beta)\tau)^{2\alpha}$ instead of just $\tau^{2\alpha}$. This same idea shows up in Kpotufe and Verma's regression analysis as well.
\end{enumerate}
Overall, we get that for large enough $n$, with high probability,
\[
\mathbb{E}_{\substack{\text{training}\\\text{data}}}\bigg[
\frac{1}{t_{\text{horizon}}}\int_{0}^{t_{\text{horizon}}}(\widetilde{S}_{\widetilde{\mathcal{Q}}_{\varepsilon}}(t|\widetilde{x})-\widetilde{S}(t|\widetilde{x}))^{2}dt\bigg]
\le \widetilde{\mathcal{O}}\Big(\frac{1}{n\mathbb{P}_{\widetilde{\mathsf{X}}}(\mathcal{B}(\widetilde{x},(1-\beta)\tau))} + ((1+\beta)\tau)^{2\alpha}\Big).
\]
We then further take the expectation of both sides over the randomness in sampling the test embedding vector $\widetilde{x}\sim\mathbb{P}_{\widetilde{\mathsf{X}}}$. Doing so, the first term on the right-hand side $\mathbb{E}_{\widetilde{x}\sim\mathbb{P}_{\widetilde{\mathsf{X}}}}[\frac{1}{n\mathbb{P}_{\widetilde{\mathsf{X}}}(\mathcal{B}(\widetilde{x},(1-\beta)\tau))}]$ can be upper-bounded by $\frac{1}{n} \cdot \Psi$ that shows up in the statement of Theorem~\ref{thm:main-result}. For details, see the full proof in Appendix~\ref{sec:pf-main-result}.

\subsection{Summary Fine-Tuning}
\label{sec:summary-fine-tuning}

Our theoretical analysis crucially relies on the summary functions being the ones stated in equation~\eqref{eq:unweighted-summaries}. However, in practice, it turns out that learning the summary functions improves prediction accuracy. We now discuss how to do this. Specifically, we add a ``summary fine-tuning'' step that refines the summary functions $\mathbf{D}_{\widetilde{q}}(\ell)$ and $\mathbf{R}_{\widetilde{q}}^+(\ell)$ that are used in the test-time hazard predictor (given in equation~\eqref{eq:kernel-netting-for-survival-analysis-hazard}).

To do summary fine-tuning, we treat everything in the survival kernet model as fixed except for the summary functions $\mathbf{D}_{\widetilde{q}}(\ell)$ and $\mathbf{R}_{\widetilde{q}}^+(\ell)$ for each exemplar $\widetilde{q}\in\widetilde{\mathcal{Q}}_{\varepsilon}$ and each time step $\ell\in\{1,\dots,m\}$, where as a reminder the time steps $t_1<t_2<\cdots<t_m$ are at the unique times of death in the training data. We now use the survival kernet predicted hazard function given in equation \eqref{eq:kernel-netting-for-survival-analysis-hazard}, which we plug into the deep kernel survival analysis loss $L_{\text{DKSA}}$ given in equation~\eqref{eq:full-dksa-loss}. Importantly, now when we minimize the loss $L_{\text{DKSA}}$, we are not learning the base neural net~$\phi$ from earlier (as it is treated as fixed); we are only learning the summary functions. Thus, the only additional details needed in explaining how summary fine-tuning works are in how we parameterize the summary functions. This requires a little bit of care to encode the constraints that the number of people at risk monotonically decreases over time and is always at least as large as the number of people who die over time.

\paragraph{Parameterizing summary functions for number of deaths}
We parameterize the summary function $\mathbf{D}_{\widetilde{q}}(\ell)$ for the number of deaths at time step $\ell$ specific to exemplar $\widetilde{q}$ as
\[
\mathbf{D}_{\widetilde{q}}(\ell) := \exp(\gamma_{\widetilde{q},\ell}) + \exp(\gamma^{\text{baseline}}_{\ell}),
\]
where $\gamma_{\widetilde{q},\ell}\in\mathbb{R}$ and $\gamma^{\text{baseline}}_{\ell}\in\mathbb{R}$ are unconstrained neural net parameters. We initialize $\mathbf{D}_{\widetilde{q}}(\ell)$ to be approximately equal to the value given in equation~\eqref{eq:unweighted-summaries} by using the initial values
\begin{align}
\gamma_{\widetilde{q},\ell}
&=\log\Big(\sum_{j\in\mathcal{I}_{\widetilde{q}}} \eventInd_j \ind\{Y_j=t_\ell\}\Big), \label{eq:gamma-parameter} \\
\gamma^{\text{baseline}}_{\ell}
&=\log(10^{-12})\approx-27.631. \label{eq:gamma-baseline-parameter}
\end{align}
Note that the $10^{-12}$ number in equation~\eqref{eq:gamma-baseline-parameter} is just an arbitrarily chosen small constant. Meanwhile, in computing equation~\eqref{eq:gamma-parameter}, to prevent numerical issues, if the summation inside the log is less than $10^{-12}$, then we also set $\gamma_{\widetilde{q},\ell}=\log(10^{-12})$.

\paragraph{Parameterizing summary functions for number of individuals at risk}
To parameterize the summary function for the number of individuals at risk, we first introduce a new variable for the number of individuals censored at time step $\ell$ specific to exemplar $\widetilde{q}$:
\[
\mathbf{C}_{\widetilde{q}}(\ell) := \exp(\omega_{\widetilde{q},\ell}) + \exp(\omega^{\text{baseline}}_{\ell}),
\]
where $\omega_{\widetilde{q},\ell}\in\mathbb{R}$ and $\omega^{\text{baseline}}_{\ell}\in\mathbb{R}$ are unconstrained neural net parameters. How these parameters are initialized works the same way as for the number of deaths and requires that we count how many people are censored at each time step $\ell$, which is straightforward to compute from the training data (the only minor complication is that the time steps are for unique observed times of death, and times of censoring could happen at other times---the simple fix is to consider the number of individuals censored at time step $\ell$ to be summed across all individuals censored at times within the interval $(t_{\ell-1},t_{\ell}]$).

Finally, it suffices to note that the number of individuals at risk at time step $\ell$ specific to exemplar $\widetilde{q}$ is given by the recurrence relation
\begin{equation}
\mathbf{R}_{\widetilde{q}}^+(\ell)=\mathbf{D}_{\widetilde{q}}(\ell)+\mathbf{C}_{\widetilde{q}}(\ell) + \mathbf{R}_{\widetilde{q}}^+(\ell+1),
\label{eq:at-risk-recurrence}
\end{equation}
where $\mathbf{R}_{\widetilde{q}}^+(m+1):=0$. In particular, for a specific exemplar $\widetilde{q}$, once we have computed $\mathbf{D}_{\widetilde{q}}(\ell)$ and $\mathbf{C}_{\widetilde{q}}(\ell)$ for all $\ell$, then we can easily compute $\mathbf{R}_{\widetilde{q}}(m)$, $\mathbf{R}_{\widetilde{q}}(m-1)$, $\mathbf{R}_{\widetilde{q}}(m-2)$, $\dots$, $\mathbf{R}_{\widetilde{q}}(1)$ using equation~\eqref{eq:at-risk-recurrence}---note that we compute these in reverse chronological order.

\paragraph{Final remarks on summary fine-tuning}
To summarize, summary fine-tuning computes refined estimates of the summary functions $\mathbf{D}_{\widetilde{q}}(\ell)$ and $\mathbf{R}_{\widetilde{q}}^+(\ell)$ by minimizing the deep kernel survival analysis loss $L_{\text{DKSA}}$ (equation~\eqref{eq:full-dksa-loss}) using the hazard function given in equation~\eqref{eq:kernel-netting-for-survival-analysis-hazard}. We treat the pre-trained base neural net $\widehat{\phi}$ and survival kernet clustering assignments as fixed. In fact, the only neural net parameters that we optimize over are the newly introduced variables $\gamma_{\widetilde{q},\ell},\gamma^{\text{baseline}}_{\ell},\omega_{\widetilde{q},\ell},\omega^{\text{baseline}}_{\ell}\in\mathbb{R}$ for $\widetilde{q}\in\widetilde{\mathcal{Q}}_{\varepsilon}$ and $\ell\in\{1,2,\dots,m\}$.

In terms of model interpretation, summary fine-tuning does not change survival kernet cluster assignments and only changes the summary functions $\mathbf{D}_{\widetilde{q}}(\ell)$ and $\mathbf{R}_{\widetilde{q}}^+(\ell)$ of the different clusters. For a specific exemplar $\widetilde{q}$, we could still compute a survival curve for the cluster corresponding to $\widetilde{q}$ using equation~\eqref{eq:cluster-KM}, where we plug in the refined summary functions (instead of the original summary functions stated in equation~\eqref{eq:unweighted-summaries}). The survival curve per cluster could still be interpreted as what the survival kernet model would predict for a data point that is purely explained by a single cluster and none of the other clusters. However, by using the refined summary functions, equation~\eqref{eq:cluster-KM} no longer corresponds to a Kaplan-Meier curve, so we can no longer use the 95\% confidence intervals that are meant for Kaplan-Meier survival curves.

We remark that simultaneously optimizing the summary functions \emph{and} the base neural net $\phi$ would be difficult in terms of how we set up the survival kernet framework because we only determine the exemplars after the base neural net $\phi$ has been learned and treated as fixed (the exemplars are chosen based on the embedding space). If we update the base neural net, the exemplars would change as would their summary functions, and in fact even the number of exemplars (and thus, the number of summary functions) could change. An alternating optimization could potentially be done (i.e., alternate between optimizing $\phi$, choosing which training points are the exemplars, and then optimizing the summary functions) but would be computationally costly; for simplicity, we do not do such an optimization.

\section{Scalable Tree Ensemble Warm-Start}
\label{sec:scalable-tree-ensemble-warm-start}

Whereas Section~\ref{sec:scaling} focused on constructing a test-time predictor that can scale to large datasets, we now turn to accelerating the training of the base neural net $\phi$, which happens in the very first step of survival kernet training. Specifically, we now present a warm-start strategy for initializing $\phi$ prior to optimizing the parameters of $\phi$ by minimizing the DKSA loss $L_{\text{DKSA}}$ given in equation~\eqref{eq:full-dksa-loss}.

Our warm-start strategy begins by learning a kernel function using a scalable tree ensemble approach such as \textsc{xgboost} \citep{chen2016xgboost}. Because a kernel function learned by a tree ensemble is not represented as a neural net, we then fit a neural net to the tree ensemble's kernel function via minibatch gradient descent. In short, we warm-start~$\phi$ using a \emph{Tree ensemble Under a Neural Approximation} (\textsc{tuna}). After the warm-start, we then fine-tune $\phi$ using the DKSA loss $L_{\text{DKSA}}$ (this is not to be confused with summary fine-tuning; here we in fact are fine-tuning the base neural net $\phi$ and we are \emph{not} optimizing over any sort of summary functions). Importantly, at the end of this section, we explain how \textsc{tuna} can accelerate neural architecture search. Note that our theory in Section~\ref{sec:theory} trivially supports~\textsc{tuna}: simply train base neural net~$\phi$ with \textsc{tuna} on the pre-training data.

\paragraph{Approximating a tree ensemble kernel}
The key idea behind \textsc{tuna} is that decision trees and their ensembles implicitly learn kernel functions \citep{breiman2000some}. Thus, by using any scalable decision tree or decision tree ensemble learning approach, we can learn an initial kernel function guess~$\mathbb{K}_0$. For example, a decision tree has $\mathbb{K}_{0}(x,x')=\ind\{x\text{ and }x'\text{ are in the same leaf}\}$. For a decision forest $\mathbb{K}_{0}(x,x')$ is equal to the fraction of trees for which $x$ and $x'$ are in the same leaf. For gradient tree boosting, the kernel function is a bit more involved and the general case is given by \citet[Section 7.1.3]{chen2018explaining}. For \textsc{xgboost} when we add one tree at a time and the final ensemble has equal weight across trees, then it suffices to set $\mathbb{K}_{0}(x,x')$ to be the fraction of trees for which~$x$ and~$x'$ are in the same leaf.

The kernel function $\mathbb{K}_0$ of a tree ensemble does not give us an embedding representation of the data, where the embedding function is a neural net. However, we can train the base neural net $\phi$ to approximate $\mathbb{K}_0$ by minimizing the mean-squared error
\begin{equation}
\frac{1}{n_{\pre}(n_{\pre}-1)/2}
\sum_{i=1}^{n_{\pre}}\sum_{j=i+1}^{n_{\pre}} \big(\underbrace{\mathbb{K}(X_i^{\pre},X_j^{\pre})}_{\text{to be learned}} - \underbrace{\mathbb{K}_0(X_i^{\pre},X_j^{\pre})}_{\substack{\text{given by the}\\\text{tree ensemble}}}\big)^2.
\label{eq:kernel-warm-start-mse-loss}
\end{equation}
As a reminder, ``$\pre$'' indicates that a quantity is part of pre-training data, and $\mathbb{K}$ depends on the base neural net~$\phi$: ${\mathbb{K}(x,x')}={K(\|\phi(x)-\phi(x')\|_2^2)}$, where for example $K(u)=\exp(-u^2)$ for a Gaussian kernel. To scale to large datasets, we minimize loss~\eqref{eq:kernel-warm-start-mse-loss} in minibatches.
The \textsc{tuna} warm-start strategy is as follows:
\begin{enumerate}[itemsep=0.2em,parsep=0em,partopsep=0em]

\item Train a scalable tree ensemble (e.g., \textsc{xgboost}) on the pre-training data; denote its learned kernel function by $\mathbb{K}_0$. For a subset $\mathcal{S}\subseteq\{X_1^{\pre},\dots,X_{n_{\pre}}^{\pre}\}$ of the pre-training feature vectors, let $\mathbf{K}_{\mathcal{S}}$ denote the $|\mathcal{S}|$-by-$|\mathcal{S}|$ Gram matrix formed so that the $(i,j)$-th entry is given by $\mathbb{K}_0(x_i^\circ, x_j^\circ)$ where $x_i^\circ$ and $x_j^\circ$ are the $i$-th and $j$-th pre-training feature vectors in~$\mathcal{S}$ (with elements of $\mathcal{S}$ ordered arbitrarily).

\item For each minibatch consisting of pre-training feature vectors $x_1^{\pre},\dots,x_b^{\pre}$ where $b$ is the batch size:
\begin{enumerate}[leftmargin=0.35em,itemsep=0.2em,parsep=0em,partopsep=0em,topsep=0.2em]

\item Compute the batch's tree ensemble Gram matrix $\mathbf{K}_{\{x_1^{\pre},\dots,x_b^{\pre}\}}$ (defined in step~1) using $\mathbb{K}_0$.

\item Compute the current neural net's Gram matrix estimate $\widehat{\mathbf{K}}_{\{x_1^{\pre},\dots,x_b^{\pre}\}}$, which has \mbox{$(i,j)$-th} entry given by $\mathbb{K}(x_i^{\pre},x_j^{\pre})=K(\|\phi(x_i^{\pre})-\phi(x_j^{\pre})\|_2^2)$.

\item Let this minibatch's loss be the MSE loss \eqref{eq:kernel-warm-start-mse-loss} restricted to feature vectors of the current minibatch, i.e., the MSE loss between $\widehat{\mathbf{K}}_{\{x_1^{\pre},\dots,x_b^{\pre}\}}$ and $\mathbf{K}_{\{x_1^{\pre},\dots,x_b^{\pre}\}}$. Update parameters of neural net $\phi$ based on the gradient of this minibatch's loss.
\end{enumerate}

\end{enumerate}
The existing warm-start strategy by \citet{chen2020deep} does not use minibatches and aims to solve a multidimensional scaling (MDS) \citep{borg2005modern} problem using the whole training dataset, which is computationally expensive to compute. For details on this MDS-based warm-start strategy, see Appendix~\ref{subsec:warm-start}.

\paragraph{Neural architecture search}
Step~2 of \textsc{tuna} can be run using different base neural nets (and step~1 only needs to be run once). Thus, we can search over different base neural net architectures and choose whichever one achieves the lowest average minibatch loss (the one described in step~2(c) above). Then when fine-tuning by minimizing the DKSA loss $L_{\text{DKSA}}$, we do not repeat a search over neural net architectures. Put another way, when searching over neural net architectures (whether using grid search or a more sophisticated approach to selecting architectures to try), we focus on minimizing a batched version of the loss~\eqref{eq:kernel-warm-start-mse-loss}.
We remark that step~2 of \textsc{tuna} does a simple least squares fit without accounting for any survival analysis problem structure.

\section{Experiments}
\label{sec:experiments}

Our experiments are designed to help us understand the empirical performance of survival kernets in terms of prediction accuracy (Section~\ref{sec:results-prediction-accuracy}) and computation time (Section~\ref{sec:results-computation-time}). We consider survival kernets with and without the summary fine tuning, and also with and without our \textsc{tuna} warm-start procedure. We also provide additional examples of cluster visualizations for different datasets that we use (Section~\ref{sec:cluster-visualizations}). Our code is publicly available.\footnote{\url{https://github.com/georgehc/survival-kernets}}

\paragraph{Datasets}
We benchmark survival kernets on four standard survival analysis datasets. We provide basic characteristics of these datasets in Table~\ref{tab:datasets}, where we have sorted the datasets in increasing number of data points. The first three datasets are on predicting time until death. The first dataset \textsc{rotterdam/gbsg} is for breast cancer patients and technically actually consists of two separate datasets: the Rotterdam tumor bank dataset \citep{foekens2000urokinase} and the German Breast Cancer Study Group dataset \citep{schumacher1994randomized}. However, these datasets share a common set of features and are frequently analyzed together (e.g., \citealt{katzman2018deepsurv,kvamme2019time,chen2020deep,zhong2021deep}), so for the rest of the paper, we denote these as the single dataset \textsc{rotterdam/gbsg}. The next dataset \textsc{support} is the one we already mentioned in Section~\ref{sec:model-interpretability} that is for hospitalized patients from the Study to Understand Prognoses, Preferences, Outcomes, and Risks of Treatment (\textsc{support}) \citep{knaus1995support}. The third dataset \textsc{unos} is for patients receiving heart transplants from the United Network for Organ Sharing.\footnote{We use the UNOS Standard Transplant and Analysis Research data from the Organ Procurement and Transplantation Network as of September 2019, requested at: \url{https://www.unos.org/data/}} The fourth dataset \textsc{kkbox} is the one on customer churn mentioned in Section~\ref{sec:intro}. These datasets have appeared in literature although not necessarily all at once in the same paper (e.g., \citealt{chapfuwa2018adversarial,giunchiglia2018rnn,katzman2018deepsurv,lee2018deephit,kvamme2019time,chen2020deep,nagpal2021deep,zhong2021deep}). Data preprocessing details are in Appendix~\ref{sec:datasets-preprocessing-details}. For the \textsc{rotterdam/gbsg} dataset, following \citet{katzman2018deepsurv}, we treat the Rotterdam data as the training data and the German Breast Cancer Study Group data as the test data. For the other datasets, we use a random 70\%/30\% train/test~split.

\begin{table}[t]
\centering
\caption{Dataset characteristics.\label{tab:datasets}}
\vspace{-.5em}
\small
\setlength{\tabcolsep}{6pt}
\begin{tabular}{cccc}
\toprule
Dataset & Number of data points & Number of features & Censoring rate \\
\midrule
\textsc{rotterdam/gbsg} & 2,232     & 7            & 43.2\%  \\
\textsc{support}                 & 8,873     & 14 (19$^*$)  & 31.97\% \\
\textsc{unos}                    & 62,644    & 49 (127$^*$) & 50.23\% \\
\textsc{kkbox}                   & 2,814,735 & 15 (45$^*$)  & 34.67\% \\
\bottomrule
\end{tabular} \\ \vspace{0.2em}
$^*$ number of features after preprocessing (note that \textsc{rotterdam/gbsg} dataset's preprocessing does not change the number of features)
\end{table}

\paragraph{Baseline survival models}
As baselines, we use an elastic-net-regularized Cox model (denoted as \textsc{elastic-net cox}), \textsc{xgboost} \citep{chen2016xgboost}, \textsc{deepsurv} \citep{katzman2018deepsurv}, \textsc{deephit} \citep{lee2018deephit}, Deep Cox mixtures (abbreviated as \textsc{dcm}) \citep{nagpal2021deep}, and \textsc{dksa} \citep{chen2020deep}. All neural net models use a multilayer perceptron as the base neural net. Hyperparameter grids and optimization details are in Appendix~\ref{sec:hyperparameter-grids-opt-details}. Note that our hyperparameter grid for \textsc{elastic-net cox} includes no regularization, which corresponds to the standard Cox model \citep{cox1972regression}. We remark that \textsc{deephit} has been previously reported to obtain state-of-the-art accuracy on the largest three datasets we have tested (cf., \citealt{lee2018deephit,kvamme2019time}; as these authors use different experimental setups, their accuracy numbers are not directly comparable to each other or to ours; however, our accuracy results yield trends similar to what was found by these authors in terms of how the baseline methods compare to each other).

\paragraph{Variants of survival kernets}
For survival kernets, as part of hyperparameter tuning, we try a number of variants. We always use the Gaussian kernel $K(u)=\exp(-u^2)$. We fix the truncation distance to be $\tau=\sqrt{2\log 10}\approx2.146$ (i.e., training points contributing to prediction for a test point must have kernel weight/similarity score with the test point that is at least $\exp(-\tau^2) = \exp(-(\sqrt{2\log 10})^2) = 0.01$), and we further only consider at most 128 approximate nearest neighbors, which are found using the HNSW algorithm~\citep{malkov2020efficient}. We comment on different choices of $\tau$ in Appendix~\ref{sec:different-threshold-distances}. For constructing $\varepsilon$-nets, we set $\varepsilon=\beta\tau$ and try $\beta\in\{1/2,1/4\}$. For the embedding space, motivated by Claims~\ref{claim:unit-hypersphere-covering-number} and \ref{claim:uniform-unit-hypersphere-intrinsic-dimension}, we constrain the embedding vectors to be on a hypersphere (in preliminary experiments, we found that leaving the embedding space unconstrained would occasionally lead to numerical issues during training and did not yield better validation accuracy scores than having a hypersphere constraint). We try standard neural net initalization \citep{he2015delving} with an exhaustive hyperparameter grid/neural architecture sweep vs our strategy \textsc{tuna} (denoted in tables as \textsc{kernet} vs \textsc{tuna-kernet}).

When dividing data so that pre-training and training data are not the same, we split the full training data into 60\%/20\%/20\% pre-training/training/validation sets (neural net training uses the pre-training and validation sets, with the latter used for early stopping). When we use this sample splitting, we add the suffix ``\textsc{(split)}'' to the method name in tables (e.g., ``\textsc{tuna-kernet (split)}'' refers to a survival kernet model trained with \textsc{tuna} warm-start and with pre-training and training sets disjoint). Otherwise, if the pre-training and training sets are set to be the same, then we split the full training data into 80\%/20\% training/validation sets similar to all the baselines, using the validation set for hyperparameter tuning (including early stopping during neural net training); in this case, we add the suffix ``\textsc{(no split)}'' to the method name in tables.

Finally, if summary fine-tuning (abbreviated \textsc{sft}) is used, then we also add the suffix ``\textsc{sft}''. For instance, ``\textsc{(no split, sft)}'' means that pre-training and training data are set to be the same, and summary fine-tuning is used.

\paragraph{Experimental setup}
For every dataset and for every model, we repeat the following basic experiment 5 times (except for a select few cases where the model is too computationally expensive to run):
\begin{enumerate}
\item We randomly split the dataset's training set into an 80\% ``proper'' training set and a 20\% validation set.
\item We train the model on the proper training set using different hyperparameters. The validation set is used to select the best hyperparameter according to an accuracy metric (we specifically use the time-dependent concordance index $C^{\text{td}}$ by \citet{antolini2005time}).
\item For the model achieving the best validation set accuracy, we use it to predict on the current dataset's test set. We record the test set accuracy score achieved (again using $C^{\text{td}}$ index).
\end{enumerate}
In our results later, we report the mean and standard deviation of the test set accuracy scores per dataset and per method across the 5 repetitions of the above basic experiment (we indicate the few cases where a model is too computationally expensive to run so that either we only ran the model once or the model could not even finish running a single time).

As a technical detail, for the first step above, we set random seeds so that per dataset, the~5 experimental repeats' have different random splits for the proper training and validation sets. However, we do \emph{not} use different random splits across the different models. For example, for the very first experimental repeat on \textsc{rotterdam/gbsg}, every model would use the exact same proper training set and also the exact same validation set. We set up our experiments in this manner so that every model gets access to the same random proper training/validation splits across the 5 experimental repeats.

\begin{table}[t]
\centering
\caption{Test set $C^{\text{td}}$ indices (mean $\pm$ standard deviation across 5 experimental repeats, except for a few cases that run into computational issues).\label{tab:main-benchmark}}
\vspace{-.5em}
\setlength{\tabcolsep}{5pt}
\adjustbox{max width=\textwidth}{%
\begin{tabular}{ccccc}
\toprule
\multirow{2}{*}[-2.5pt]{Model} & \multicolumn{4}{c}{Dataset} \\
\cmidrule{2-5}
& \textsc{rotterdam/gbsg} & \textsc{support} & \textsc{unos} & \textsc{kkbox}\\
\midrule
\textsc{elastic-net cox} & 0.6660 $\pm$ 0.0045 & 0.6046 $\pm$ 0.0013 & 0.5931 $\pm$ 0.0011 & 0.8438 $\pm$ 0.0001 \\
\textsc{xgboost} & 0.6703 $\pm$ 0.0128 & 0.6281 $\pm$ 0.0031 & 0.6028 $\pm$ 0.0009 & 0.8714 $\pm$ 0.0000 \\
\textsc{deepsurv} & {\bftab 0.6850} $\pm$ 0.0160 & 0.6155 $\pm$ 0.0032 & 0.5941 $\pm$ 0.0021 & 0.8692 $\pm$ 0.0003 \\
\textsc{deephit} & 0.6792 $\pm$ 0.0121 & 0.6354 $\pm$ 0.0047 & 0.6170 $\pm$ 0.0016 & {\bftab 0.9148} $\pm$ 0.0001 \\
\textsc{dcm} & 0.6763 $\pm$ 0.0104 & 0.6289 $\pm$ 0.0047 & 0.6101 $\pm$ 0.0023 & 0.8830$^*$ \\
\textsc{dksa} & 0.6570 $\pm$ 0.0139 & 0.6316 $\pm$ 0.0080 & out of memory & out of memory \\
\cmidrule(lr){1-5}
\textsc{kernet (split)} & 0.6450 $\pm$ 0.0086 & 0.5934 $\pm$ 0.0073 & 0.5936 $\pm$ 0.0039 & 0.8933$^*$ \\
\textsc{kernet (split, sft)} & 0.6478 $\pm$ 0.0140 & 0.6007 $\pm$ 0.0040 & 0.5984 $\pm$ 0.0039 & 0.9027$^*$ \\
\textsc{kernet (no split)} & 0.6599 $\pm$ 0.0190 & 0.6244 $\pm$ 0.0026 & 0.6033 $\pm$ 0.0039 & 0.8942$^*$ \\
\textsc{kernet (no split, sft)} & 0.6621 $\pm$ 0.0191 & 0.6291 $\pm$ 0.0059 & 0.6071 $\pm$ 0.0039 & 0.9029$^*$ \\
\textsc{tuna-kernet (split)} & 0.6510 $\pm$ 0.0212 & 0.6220 $\pm$ 0.0026 & 0.6028 $\pm$ 0.0032 & 0.8952 $\pm$ 0.0002 \\
\textsc{tuna-kernet (split, sft)} & 0.6544 $\pm$ 0.0239 & 0.6287 $\pm$ 0.0050 & 0.6105 $\pm$ 0.0046 & 0.9049 $\pm$ 0.0004 \\
\textsc{tuna-kernet (no split)} & 0.6694 $\pm$ 0.0163 & 0.6385 $\pm$ 0.0038 & 0.6130 $\pm$ 0.0029 & 0.8957 $\pm$ 0.0005 \\
\textsc{tuna-kernet (no split, sft)} & 0.6719 $\pm$ 0.0135 & {\bftab 0.6426} $\pm$ 0.0045 & {\bftab 0.6211} $\pm$ 0.0025 & 0.9057 $\pm$ 0.0003 \\
\bottomrule
\end{tabular}} \\
\smallskip
{\small $^*$ we only ran one experimental repeat due to how computationally expensive \textsc{dcm} and the non-\textsc{tuna} survival kernet variants are, so standard deviations are unavailable for these}
\end{table}

\subsection{Results on Prediction Accuracy}
\label{sec:results-prediction-accuracy}
We report test set $C^{\text{td}}$ indices in Table~\ref{tab:main-benchmark} (mean $\pm$ standard deviation across 5 experimental repeats unless some computational issue was encountered, as described above). The main findings are as follows:
\begin{itemize}
\item[(a)] On the largest three datasets (\textsc{support}, \textsc{unos}, and \textsc{kkbox}), the \textsc{tuna-kernet (no split, sft)} model (i.e., a survival kernet trained with \textsc{tuna}, pre-training data set to be the same as training data, and summary fine-tuning) achieves the highest mean test set $C^{\text{td}}$ indices on \textsc{support} and \textsc{unos} and the second highest on \textsc{kkbox} (second to \textsc{deephit}).

\item[(b)] On the smallest dataset \textsc{rotterdam/gbsg}, \textsc{tuna-kernet (no split, sft)} has higher accuracy than the other survival kernet variants but the mean $C^{\text{td}}$ score it achieves is lower than those of several baselines (\textsc{deepsurv}, \textsc{deephit}, \textsc{dcm}---of which \textsc{deepsurv} achieves the highest mean $C^{\text{td}}$ score). \textsc{tuna-kernet (no split, sft)} does not appear significantly less accurate than these baselines though since the test set $C^{\text{td}}$ standard deviations for \textsc{rotterdam/gbsg} across all models are high. For example, for the test set $C^{\text{td}}$ indices of \textsc{tuna-kernet (no split, sft)}, the interval that is within one standard deviation of the mean (namely, $0.6719 \pm 0.0135$) contains the mean test set $C^{\text{td}}$ indices achieved by \textsc{deepsurv}, \textsc{deephit}, and \textsc{dcm}.\footnote{In fact, when we look at the 5 experimental repeats (as we described in the experimental setup above, for a specific experimental repeat, all models use the same proper training set and the same validation set), \textsc{deepsurv} achieved higher test set $C^{\text{td}}$ index compared to \textsc{tuna-kernet (no split, sft)} only in 3 out of the 5 experiments: specifically, \textsc{deepsurv} achieved test set $C^{\text{td}}$ indices {\bftab 0.693}, {\bftab 0.697}, {\bftab 0.699}, 0.673, and 0.663 whereas \textsc{tuna-kernet (no split, sft)} achieved 0.673, 0.649, 0.674, {\bftab 0.679}, and {\bftab 0.685}, respectively.}

\item[(c)] Our experimental results include two variants of survival kernets that do correspond to our theoretical analysis (namely \textsc{kernet (split)} and \textsc{tuna-kernet (split)}, of which the latter generally outperforms the former). We point out that \textsc{tuna-kernet (split)} does outperform a few baselines: it achieves a higher mean test set $C^{\text{td}}$ index than \textsc{elastic-net cox} and \textsc{deepsurv} on \textsc{support}, \textsc{unos}, and \textsc{kkbox}. For the \textsc{kkbox} dataset, \textsc{tuna-kernet (split)} also outperforms \textsc{xgboost} and \textsc{dcm}.

\item[(d)] Focusing only on the survival kernet variants, we observe the following general trends:
\begin{itemize}
\item[i.] By making the pre-training and training sets disjoint as required by our theory, the accuracy decreases (this trend generally holds for any variant of survival kernets where the only difference is whether pre-training and training sets are disjoint vs set to be the same).
\item[ii.] Using summary fine-tuning improves accuracy compared to setting the summary functions based on equation~\eqref{eq:unweighted-summaries} (this trend generally holds for any variant of survival kernets where the only difference is whether summary fine-tuning is used).
\item[iii.] Using the \textsc{tuna} warm-start strategy improves accuracy compared to standard neural net initialization (this trend generally holds for any variant of survival kernets where the only difference is whether \textsc{tuna} is used).
\end{itemize}
\end{itemize}
As a minor remark, our \textsc{kkbox} $C^{\text{td}}$ index is significantly higher for \textsc{deephit} than previously reported in literature (e.g., \citealt{kvamme2019time}) because we also try setting the time discretization grid to be all unique times of customer churn in the dataset, which turns out to significantly improve \textsc{deephit}'s accuracy for the \textsc{kkbox} dataset (when we discretize to 64 or 128 time steps, the $C^{\text{td}}$ indices we get are on par with what Kvamme et al.~get).

\subsection{Results on Computation Time}
\label{sec:results-computation-time}
We next report wall clock computation times for training survival kernet variants and the baselines in Table~\ref{tab:timing-benchmark}. These training times are inclusive of hyperparameter tuning (which includes searching over neural net architectures). Our computation time results aim to show how much our \textsc{tuna} warm-start strategy accelerates training and also to give a sense of the time needed to train the different models under consideration. All code was run under similar conditions on identical Ubuntu 20.04.2 LTS instances, each with an Intel Core i9-10900K CPU (3.70 GHz, 10 cores, 20 threads), 64GB RAM, and an Nvidia Quadro RTX 4000 (8GB GPU RAM). These compute instances accessed code and data that were centrally stored on a single network-attached storage system.

\begin{table}[t]
\centering
\caption{Total training time (mean $\pm$ standard deviation across 5 experimental repeats, except for a few cases that run into computational issues), which includes training using different hyperparameters. Note that \textsc{tuna-kernet (split)} and \textsc{tuna-kernet (split, sft)} include the time needed to train an \textsc{xgboost} warm-start model only using \emph{pre-training data}, whereas \textsc{tuna-kernet (no split)} and \textsc{tuna-kernet (no split, sft)} include time to train \textsc{xgboost} on the \emph{full training data}.
\label{tab:timing-benchmark}}
\vspace{-.5em}
\begingroup
\setlength{\tabcolsep}{2pt}
\adjustbox{max width=\textwidth}{%
\begin{tabular}{crlrlrlrl}
\toprule
\multirow2{*}[-2.5pt]{Model} & \multicolumn{8}{c}{Total training time (minutes)} \\
\cmidrule{2-9}
& \multicolumn{2}{c}{\makebox[0pt]{\textsc{rotterdam/gbsg}}} & \multicolumn{2}{c}{\textsc{support}} & \multicolumn{2}{c}{\textsc{unos}} & \multicolumn{2}{c}{\textsc{kkbox}} \\
\midrule
\textsc{elastic-net cox} & 0.294 &$\pm$ 0.047 & 0.762 &$\pm$ 0.031 & 3.075 &$\pm$ 0.169 & 60.171 &$\pm$ 0.155 \\
\textsc{xgboost} & 9.369 &$\pm$ 0.209 & 17.885 &$\pm$ 0.537 & 73.503 &$\pm$ 0.420 & 292.358 &$\pm$ 1.943 \\
\textsc{deepsurv} & 0.140 &$\pm$ 0.009 & 0.323 &$\pm$ 0.016 & 2.784 &$\pm$ 0.209 & 48.621 &$\pm$ 0.136 \\
\textsc{deephit} & 5.083 &$\pm$ 0.343 & 9.984 &$\pm$ 0.373 & 467.479 &$\pm$ 5.621 & 1573.683 &$\pm$ 5.333 \\
\textsc{dcm} & 6.739 &$\pm$ 0.539 & 20.667 &$\pm$ 0.829 & 1111.717 &$\pm$ 25.459 & \multicolumn{2}{c}{38263.568$^*$} \\
\textsc{dksa} & 2.756 &$\pm$ 0.184 & 31.135 &$\pm$ 0.269 & \multicolumn{2}{c}{out of memory} & \multicolumn{2}{c}{out of memory} \\
\cmidrule(lr){1-9}
\textsc{kernet (split)} & 8.917 &$\pm$ 0.464 & 47.581 & $\pm$ 1.354 & 841.923 & $\pm$ 80.509 & \multicolumn{2}{c}{10323.367$^*$} \\
\textsc{kernet (split, sft)} & 9.074 &$\pm$ 0.480 & 49.025 &$\pm$ 1.149 & 851.257 &$\pm$ 86.945 & \multicolumn{2}{c}{10515.057$^*$} \\
\textsc{kernet (no split)} & 23.964 &$\pm$ 0.588 & 128.579 &$\pm$ 3.241 & 1579.229 &$\pm$ 58.450 & \multicolumn{2}{c}{14323.791$^*$} \\
\textsc{kernet (no split, sft)} & ~24.103 &$\pm$ 0.583~ & ~129.694 &$\pm$ 3.094~ & ~1595.272 &$\pm$ 65.490~ & \multicolumn{2}{c}{14673.548$^*$} \\
\textsc{tuna-kernet (split)} & 7.210 &$\pm$ 0.296 & 17.840 &$\pm$ 1.317 & 153.803 &$\pm$ 10.966 & ~1581.292 &$\pm$ 26.261~ \\
\textsc{tuna-kernet (split, sft)} & 7.471 &$\pm$ 0.373 & 19.250 &$\pm$ 1.739 & 162.893 &$\pm$ 8.254 & 1924.505 &$\pm$ 17.151 \\
\textsc{tuna-kernet (no split)} & 12.207 &$\pm$ 0.706 & 30.052 &$\pm$ 1.035 & 312.922 &$\pm$ 33.438 & 2083.887 &$\pm$ 85.976 \\
\textsc{tuna-kernet (no split, sft)} & ~~~12.386 &$\pm$ 0.674~~~ & 31.602 &$\pm$ 1.809 & 320.748 &$\pm$ 31.228 & ~2555.702 &$\pm$ 65.595~ \\
\bottomrule
\end{tabular}} \\
\medskip
{\small $^*$ we only ran one experimental repeat due to how computationally expensive \textsc{dcm} and the non-\textsc{tuna} survival kernet variants are, so standard deviations are unavailable for these}
\endgroup
\end{table}

Note that the timing results in Table~\ref{tab:timing-benchmark} are fair in comparing between the different variants of survival kernets but \emph{not} fair in comparing between a survival kernet variant and a baseline or between two different baselines. The reason for this is that, as aforementioned, we report training times that are inclusive of hyperparameter tuning. Different models can have drastically different hyperparameters and hyperparameter grid sizes. We can easily increase or decrease the number of hyperparameters we try for a model, which can significantly increase or decrease the overall training time of the model as a result. Because we use the exact same hyperparameter grid across survival kernet variants, comparing computation times across them is fair. However, comparing the training times of a survival kernet variant and of a baseline (or even of two different baselines) does not lead to a fair comparison due to the hyperparameter grids being different.\footnote{Note that reporting ``per epoch'' training times does not give a fair comparison either since, for instance, training the \textsc{tuna-kernet (no split, sft)} model involves three different kinds of neural net minibatch gradient descent optimizations (the first is for approximating the \textsc{xgboost} tree ensemble kernel, the second is for fine-tuning the base neural net using the DKSA loss, and the third is for summary fine-tuning) whereas, for instance, the baselines \textsc{deepsurv} and \textsc{deephit} each only involve one kind of neural net minibatch gradient descent optimization. Reporting ``per hyperparameter setting'' training times is not straightforward since, for instance, the \textsc{tuna} warm-start strategy could be viewed as further expanding the hyperparameter grid to include the hyperparameters of the base tree ensemble model being trained while at the same time never doing an exhaustive grid search.} That said, we present the overall training times across models anyways in Table~\ref{tab:timing-benchmark} (mean $\pm$ standard deviation across the 5 experimental repeats unless some computational issue was encountered).

The main takeaways from Table~\ref{tab:timing-benchmark} are as follows:
\begin{itemize}
\item[(a)] For each \textsc{tuna-kernet} variant, when we compare it to its corresponding variant that does not use \textsc{tuna}, then we consistently find that using \textsc{tuna} dramatically reduces computation time, with savings of 17.7\% to 85.5\% depending on the specific variant of survival kernets used and on the dataset. The savings are larger on the larger datasets; for instance, when only considering the largest three datasets, then the savings actually range from 60.7\% to 85.5\%.
\item[(b)] Using summary fine-tuning clearly increases training time compared to not using it, with an increase of 0.58\% to 22.6\% additional time depending on the specific variant of survival kernets used and on the dataset.
\item[(c)] Even though survival kernets with \textsc{tuna} are significantly faster to train than survival kernets without \textsc{tuna}, survival kernets with \textsc{tuna} are still relatively expensive to train when compared to baselines tested (aside from \textsc{dksa} and \textsc{dcm}, both of which appear to have issues scaling to larger datasets) at least for the hyperparameter grids that we have used.\footnote{For \textsc{dcm} in particular, the maximum number of clusters we attempted to use was~6 (which was also the maximum used by the original authors \citep{nagpal2021deep}) and even so, the total training time clearly grows at a rate that makes it impractical for use on large datasets (as shown in Table~\ref{tab:timing-benchmark}, a single experiment on \textsc{kkbox} took over 26 days to run). The issue is that the default behavior in the code by Nagpal et al.~is that every 10 minibatches, it does posterior inference across the entire proper training set (not just the current minibatch) to update every cluster's survival curve (represented as a spline).}
\end{itemize}
We reiterate that there are trivial ways to reduce or increase the computation time of either survival kernets or any of the baseline models by simply removing or adding hyperparameter settings to try. For example, we suspect that for the survival kernet variants and for \textsc{deephit}, we could remove some of the hyperparameter settings that we had tried while retaining a similar test set accuracy.
Note that the hyperparameter grid we used for every survival kernet variant is a superset of the one we used for \textsc{deephit}, so we expected the computation times to be higher for survival kernet variants than \textsc{deephit}. It turns out though that all the survival kernet variants with \textsc{tuna} have lower overall training times than \textsc{deephit} on the \textsc{unos} dataset for the hyperparameter grids we used.

As an illustrative example to provide more insight on how much time is spent on different parts of survival kernet training, we provide a timing breakdown for the \textsc{tuna-kernet (no split, sft)} model trained on the \textsc{kkbox} dataset in Table~\ref{tab:tuna-kernet-nosplit-sft-kkbox-breakdown}. Note that in this table, we have separated the computation times for using different $\varepsilon$ values in $\varepsilon$-net construction (as a reminder, we set $\varepsilon=\beta\tau$ where the threshold distance is set to be $\tau=\sqrt{2\log(10)}\approx2.146$, and we sweep over $\beta\in\{1/2,1/4\}$). For example, we see that if we only tried using $\beta=1/4$ and did not try $\beta=1/2$, then we would reduce the computation time by about 31\% in this case. In Table~\ref{tab:tuna-kernet-nosplit-sft-kkbox-breakdown}, we also provide some detail on the hyperparameters that we sweep over for the different steps during training. As a reminder, for the \textsc{tuna} warm-start strategy, the neural net architecture is treated as fixed after we approximate the \textsc{xgboost} kernel with the base neural net (in particular, trying different neural net architectures happens during the step labeled ``Approximate \textsc{xgboost} kernel with base neural net'' in Table~\ref{tab:tuna-kernet-nosplit-sft-kkbox-breakdown}).

\begin{table}
\centering
\caption{Time breakdown in training the \textsc{tuna-kernet (no split, sft)} model on the \textsc{kkbox} dataset for our experimental setup (mean $\pm$ standard deviation across 5 experimental repeats). The time breakdown for training \textsc{tuna-kernet (no split)} is the same without the final summary fine-tuning time. Note that our \textsc{tuna} warm-start strategy corresponds to the combination of the first two tasks listed below. Also, note that every task listed below involves tuning different hyperparameters. For details on hyperparameters, see Appendix~\ref{sec:hyperparameter-grids-opt-details}. For a slightly more detailed time breakdown that separates the fine-tuning of the base neural net further to include the different numbers of time steps to discretize to, see Appendix~\ref{sec:more-detailed-time-breakdown}.}
\label{tab:tuna-kernet-nosplit-sft-kkbox-breakdown}
\vspace{-.5em}
\small
\setlength{\tabcolsep}{6pt}
\adjustbox{max width=\textwidth}{
\begin{tabular}{ccc}
\toprule
Task & Hyperparameter settings & Time (minutes) \\
\midrule
Train \textsc{xgboost} model & 192$^*$ & 292.358 $\pm$ 1.943 \\
Approximate \textsc{xgboost} kernel with base neural net & 18$^\dagger$ & 107.271 $\pm$ 0.481 \\
Fine-tune base neural net with DKSA loss ($\beta=1/2$) & 36$^\ddagger$ & 793.884 $\pm$ 52.890 \\
Fine-tune base neural net with DKSA loss ($\beta=1/4$) & 36$^\ddagger$ & 890.374 $\pm$ 34.188 \\
Summary fine-tuning & 2$^\text{\S}$ & 471.814 $\pm$ 37.123 \\
\midrule
Total & & 2555.702 $\pm$ 65.595 \\
\bottomrule
\end{tabular}} \\
\smallskip
$^*$ corresponds to the full \textsc{xgboost} hyperparameter grid that we use \\
$^\dagger$ base neural net hyperparameters (number of hidden layers, nodes per hidden layer), learning rate \\
$^\ddagger$ survival loss hyperparameters ($\eta$ and $\sigma_{\text{rank}}$ from equation~\eqref{eq:full-dksa-loss}), number of time steps to discretize to, learning rate \\
$^\text{\S}$ learning rate
\end{table}

\subsection{Cluster Visualizations}
\label{sec:cluster-visualizations}

We now present cluster visualizations of \textsc{tuna-kernet (no split, sft)} models. For simplicity, we only show visualizations for the first experimental repeat.\footnote{As is standard in machine learning literature, when we talk about model interpretability, we are talking about interpreting a specific trained model. For example, consider topic models (such as latent Dirichlet allocation \citep{blei2003latent} or any neural topic model \citep{zhao2021topic}) or prototypical part networks \citep{chen2019looks}. When one re-trains such a topic or prototypical part model using different random seeds (but leave model hyperparameters fixed), the topics/prototypes learned can vary across random seeds. Even so, each specific learned model (per experimental repeat) can be interpreted.} For ease of exposition, we begin with the \textsc{support} dataset since we had already presented a figure for it in Section~\ref{sec:model-interpretability} (Figure~\ref{fig:support}). We now also discuss some modified versions of the two plots in Figure~\ref{fig:support}, first to how incorporate summary fine-tuning and then, separately, to discuss how to summarize \emph{all} the clusters found and not just a few. After providing cluster visualizations for the \textsc{support} dataset, we then provide visualizations for the \textsc{rotterdam/gbsg}, \textsc{unos}, and \textsc{kkbox} datasets.

\paragraph{SUPPORT dataset}
For the final \textsc{tuna-kernet (no split)} model trained on the \textsc{support} dataset that we used to evaluate test set accuracy, we visualize the largest 5 clusters in Figure~\ref{fig:support} using our visualization strategies from Section~\ref{sec:model-interpretability}. In Section~\ref{sec:model-interpretability}, we had already pointed out how one could interpret the different clusters, for instance finding the rightmost cluster in Figure~\ref{fig:support}\subref{subfig:support-heatmap} to be of mostly young, cancer-free patients. Meanwhile, the Kaplan-Meier curves in Figure~\ref{fig:support}\subref{subfig:support-km} quickly give us a sense of which clusters appear to be better or worse off than others over time. For this dataset, the largest 5 clusters contain 60.5\% of the proper training data. There are a total of 73 clusters, many of which are small. Of course, more clusters can be visualized at once; the plots get more cluttered though.

If instead summary fine-tuning is used and we want to visualize the largest 5 clusters of the final \textsc{tuna-kernet (no split, sft)} model, then we would make the following changes. First, note that the final \textsc{tuna-kernet (no split, sft)} model is actually identical to the final \textsc{tuna-kernet (no split)} model except that summary fine-tuning is used at the very end, which does not change the survival kernet cluster assignments from the $\varepsilon$-net. Thus, in terms of cluster visualization, the only main change would be that instead of computing the survival curve for each cluster using a Kaplan-Meier plot as in Figure~\ref{fig:support}\subref{subfig:support-km}, we would instead compute the survival curve using equation~\eqref{eq:cluster-KM} with the learned summary functions. Next, based on these newly estimated survival curves, we can look at when they cross probability 1/2 to obtain median survival time estimates. The survival curves from summary fine-tuning in this case are shown in Figure~\ref{fig:support-sft-survival-curves}; note that these curves are in the exact same ordering and correspond to the same clusters as in Figure~\ref{fig:support} (the same color means the same cluster although the median survival times change), but as we can readily see, they look slightly different from the Kaplan-Meier curves from Figure~\ref{fig:support}\subref{subfig:support-km}. Since the survival curves in Figure~\ref{fig:support-sft-survival-curves} are no longer the standard Kaplan-Meier ones, we do not have confidence intervals for them. Meanwhile, the heatmap in Figure~\ref{fig:support}\subref{subfig:support-heatmap} remains the same except that the median survival times at the top would be updated to the new values computed (and that appear in the legend of Figure~\ref{fig:support-sft-survival-curves}); note that the columns may have to be reordered if we still want them to be sorted by median survival time (the new median survival time estimates based on summary fine-tuning do not have to be in the same order as the original median survival time estimates based on Kaplan-Meier curves).

We have found that for some clusters, the survival curve from summary fine-tuning is relatively close to that of the Kaplan-Meier estimate without summary fine-tuning whereas for other clusters, there could be a larger difference (such as the green cluster shifting upward at earlier times in Figure~\ref{fig:support-sft-survival-curves} compared to in Figure~\ref{fig:support}\subref{subfig:support-km}). Some possible explanations could be that the green cluster is in a region of the embedding space that is close to other clusters and its survival function that gets learned using summary fine-tuning could be thought of as being estimated from a larger region of the embedding space (rather than just getting averaged over points assigned to the green cluster by the $\varepsilon$-net), or that even among the points in the cluster, calculating the survival curve by weighting every point equally (as in the standard Kaplan-Meier survival curve estimate) does not make as much sense as using some sort of unequal/nonuniform weighting.

\begin{figure}[!t]
\centering
\includegraphics[scale=.42]{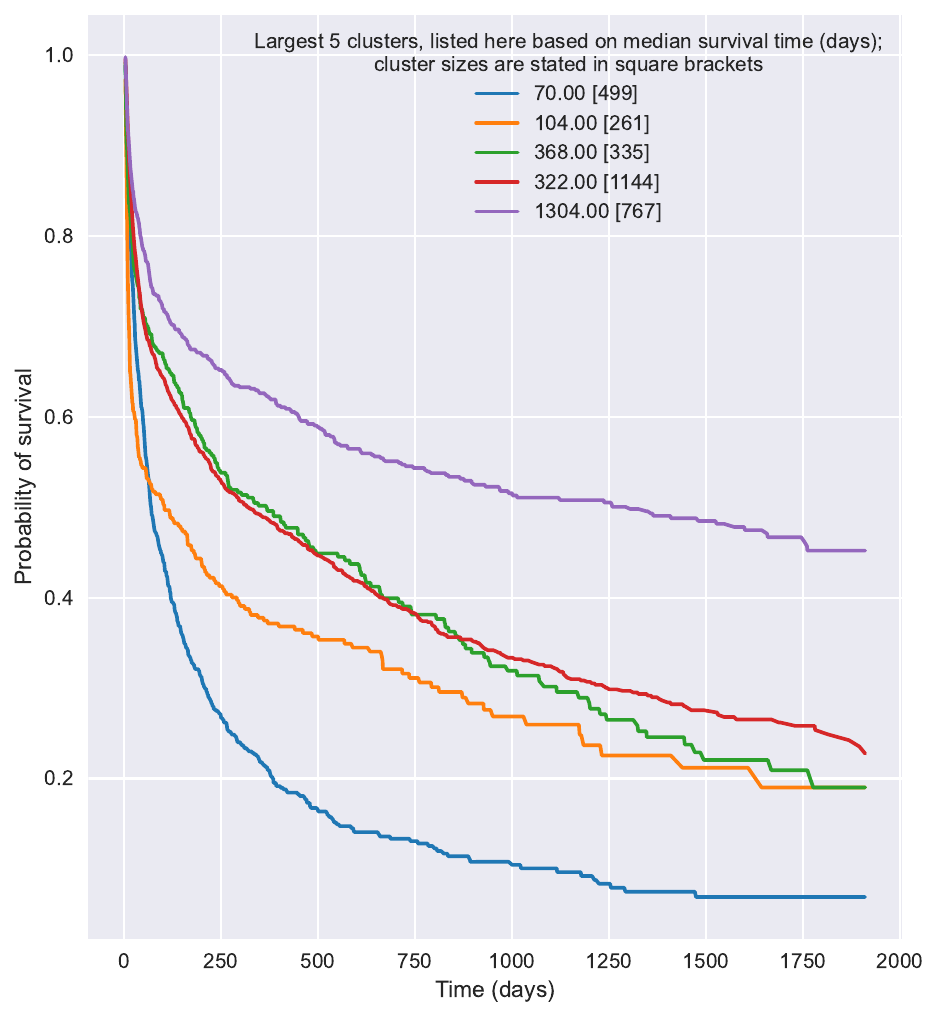}
\caption{Survival curves for the largest 5 clusters found by the final \textsc{tuna-kernet (no split, sft)} model trained on the \textsc{support} dataset; the x-axis measures the number of days since a patient entered the study. Note that the green curve has a higher median survival time estimate than the red curve; this is not a typo in that we are ordering the clusters the exact same way as in Figure~\ref{fig:support}\subref{subfig:support-km}. In particular, the median survival time estimates using summary fine-tuning do \emph{not} have to be ordered the same way as the median survival time estimates from the Kaplan-Meier estimator.}
\label{fig:support-sft-survival-curves}
\end{figure}

Although there are a total of 73 clusters, we have shown only 5 of them thus far. While we could visualize all 73 at once using the same visualization ideas as above, the heatmap and survival curve plots would get cluttered. One way to still visualize information from all 73 clusters is to use agglomerative clustering \citep[Section~14.3.12]{hastie2009elements} to merge many clusters into larger ``superclusters''. For the new superclusters, we can make visualizations similar to Figure~\ref{fig:support}. For instance, using standard complete-linkage agglomerative clustering, we partition the 73 clusters into 10 superclusters in Figure~\ref{fig:support-superclusters}. The resulting feature heatmap in Figure~\ref{fig:support-superclusters}\subref{subfig:support-superclusters-heatmap} gives a more comprehensive view of how features vary across the entire training dataset. In this case, the second-to-rightmost and third-to-rightmost columns/clusters in the heatmap each only have one data point, suggesting that they are outliers among the training data. The rest of the superclusters, however, exhibit some of the same trends we pointed out earlier when only looking at the largest 5 clusters, such as the rightmost supercluster (with the highest median survival time) tending to be for younger, cancer-free patients who often have no comorbidities. Among the leftmost superclusters, we get different common patterns in patient characteristics that are associated with low survival times (e.g., the leftmost four superclusters have unusually high creatinine levels possibly indicative of severe kidney disorder, the supercluster with a median survival time of 80 days corresponds to patients tending to have metastatic cancer and at least one comorbidity). Meanwhile, the survival curves in Figure~\ref{fig:support-superclusters}\subref{subfig:support-superclusters-survival-curves} are more spread out compared to those in Figure~\ref{fig:support-sft-survival-curves}.

Note that when we merge clusters into a supercluster, there is a question of how to compute an aggregated survival curve for the supercluster. If we are not using summary fine-tuning, then the aggregation is straightforward: we continue to use the standard Kaplan-Meier survival curve estimator but instead now use data from all patients in a supercluster to compute that supercluster's survival curve. However, when using summary fine-tuning, there is no standard way to aggregate information from different clusters' learned summary functions. For simplicity, to come up with the survival curves in Figure~\ref{fig:support-superclusters}\subref{subfig:support-superclusters-survival-curves}, for the clusters that are being merged into a supercluster, we take these clusters' survival curves and then just take a weighted average to obtain the survival curve for the supercluster (the weights are proportional to how many points are in each cluster belonging to the supercluster). The median survival times are then estimated by finding the times in which these survival curves cross probability 1/2.

\begin{figure}[!p]
\centering
\begin{subfigure}[b]{\linewidth}
\centering
\includegraphics[scale=.35]{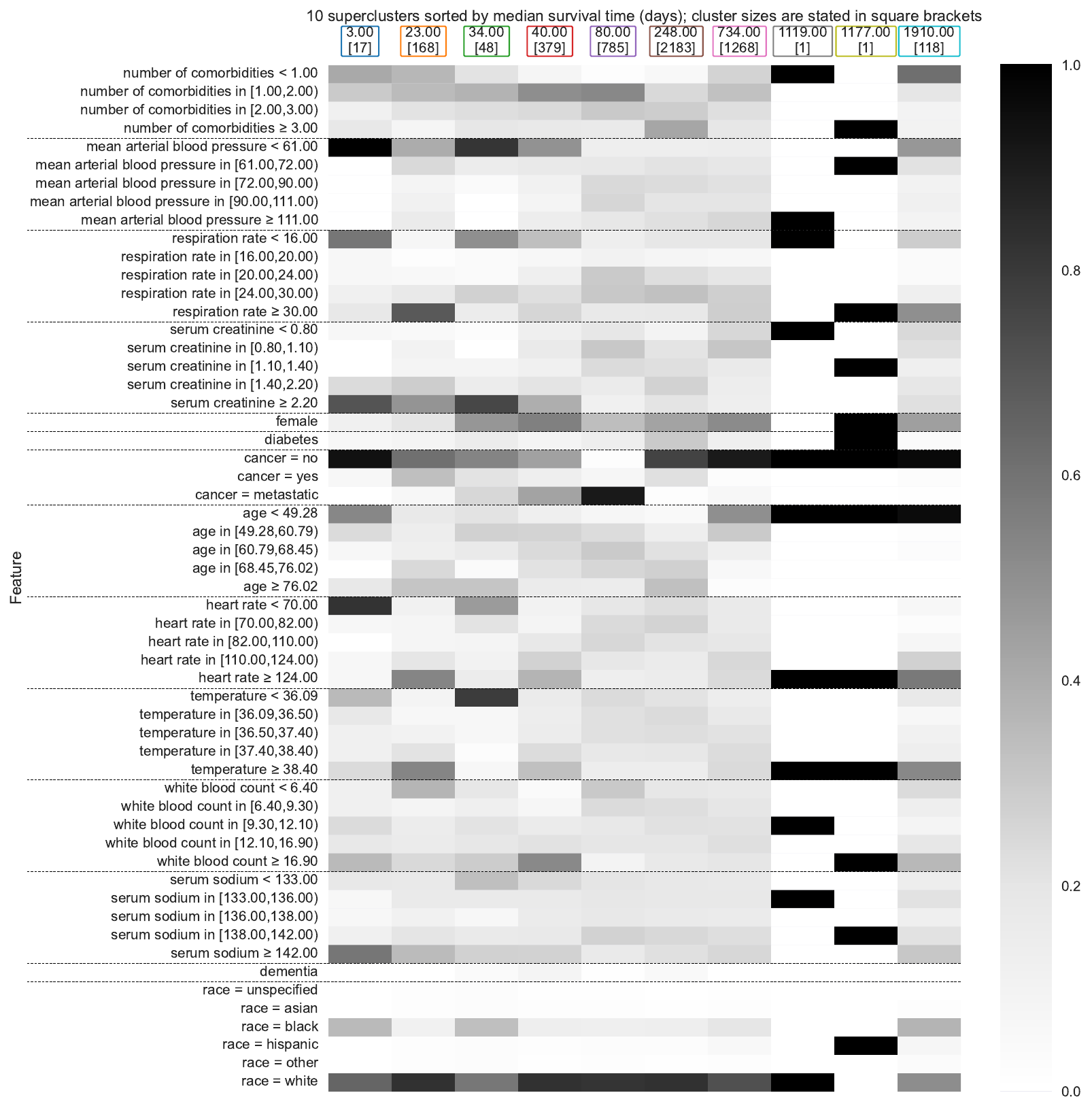}
\caption{}
\label{subfig:support-superclusters-heatmap}
\end{subfigure} \\ \vspace{.5em}
\begin{subfigure}[b]{\linewidth}
\centering
\includegraphics[scale=.35]{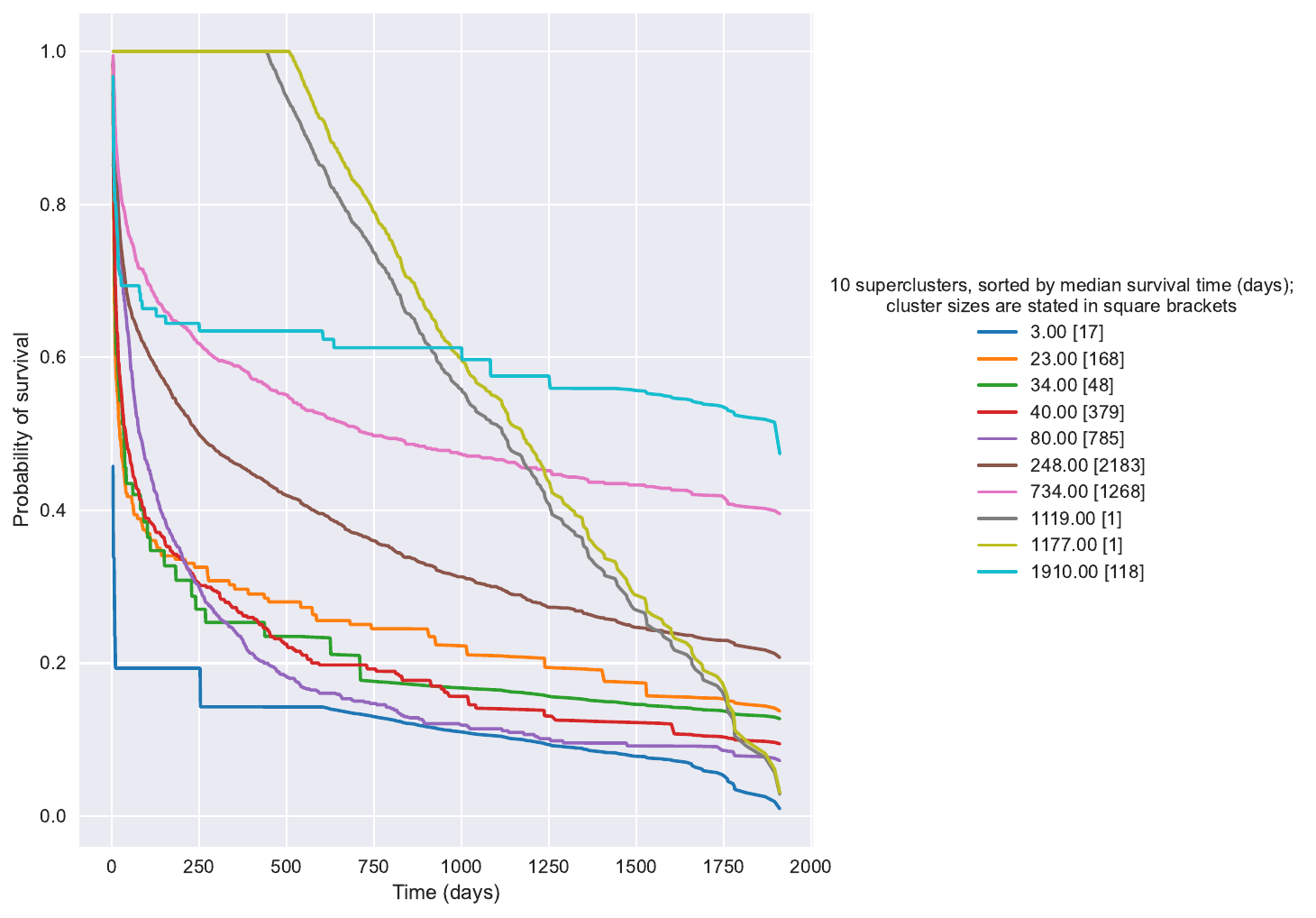}
\vspace{-.5em}
\caption{}
\label{subfig:support-superclusters-survival-curves}
\end{subfigure}
\caption{Visualization of 10 superclusters for the final \textsc{tuna-kernet (no split, sft)} model trained on the \textsc{support} dataset. These 10 superclusters summarize all 73 clusters found by the \textsc{tuna-kernet (no split, sft)} model by merging clusters using complete-linkage agglomerative clustering. Panel (a) shows a feature heatmap visualization. Panel (b) shows survival curves for the same superclusters as in panel (a); the x-axis measures the number of days since a patient entered the study. The second-to-rightmost and third-to-rightmost columns/clusters each only have one data point, suggesting that they are outliers.}
\label{fig:support-superclusters}
\end{figure}

\paragraph{Rotterdam-GBSG dataset}
Next, we plot the feature heatmap and survival curves for the final \textsc{tuna-kernet (no split, sft)} model trained on the \textsc{rotterdam/gbsg} dataset (technically trained only on the Rotterdam portion of the data) in Figure~\ref{fig:rotterdam-gbsg}. In this case the final model only has a total of 10 clusters, so we plot all 10 of them. This number of clusters found is also small enough that using the supercluster idea above that we discussed for the \textsc{support} dataset is unnecessary.

From looking at the heatmap in Figure~\ref{fig:rotterdam-gbsg}\subref{subfig:rotterdam-gbsg-heatmap}, we immediately see a few key patterns: lower survival times are associated with higher numbers of positive nodes (these are lymph nodes with cancer) and higher tumor sizes. The rightmost cluster has the highest median survival time ($>$ 64 months); note that the reason the median survival time is not a precise estimate here and is only a lower bound is that the estimated survival curve for this cluster never crosses the probability 1/2 threshold by the last observed time. For this rightmost cluster, we see that the number of positive nodes is low and also the age distribution of patients in the cluster also tends to be younger than those of the other clusters. These qualitative findings agree with previous literature on breast cancer \citep{sopik2018relationship}.

By looking instead at the survival curves in Figure~\ref{fig:rotterdam-gbsg}\subref{subfig:rotterdam-gbsg-survival-curves}, we immediately notice a pattern in the survival curves: they roughly look like powers of each other, which happens when the proportional hazards assumption holds! This observation agrees with the prediction accuracy results from earlier where we saw that the \textsc{deepsurv} model (which uses a proportional hazards assumption) is the best-performing model we evaluated for \textsc{rotterdam/gbsg}. We suspect that the survival kernet model is essentially too flexible of model in this case (arguably \textsc{deephit} and \textsc{dcm} are also too flexible for this dataset).

\begin{figure}[!t]
\centering
\begin{subfigure}[b]{\linewidth}
\centering
\includegraphics[scale=.42]{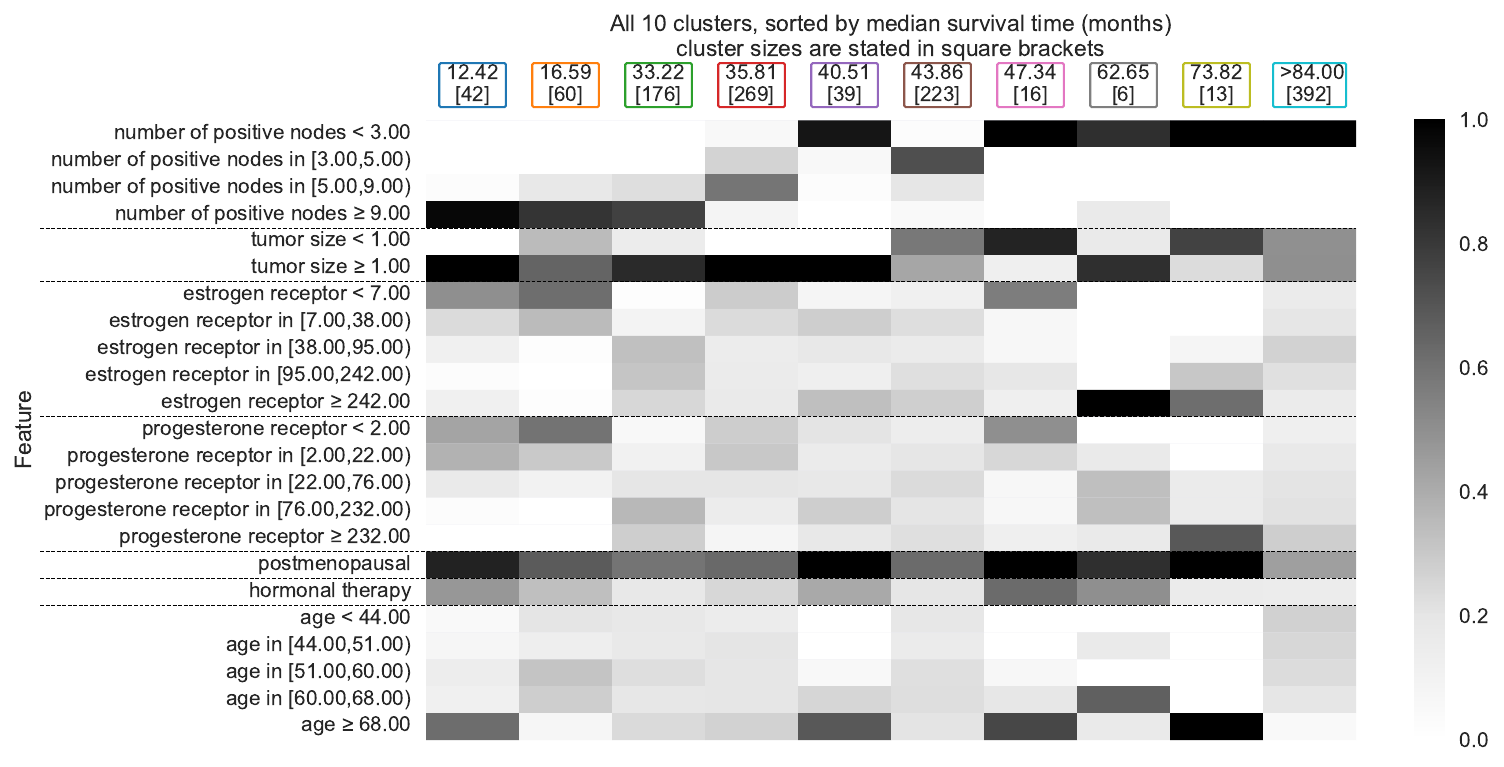}
\caption{}
\label{subfig:rotterdam-gbsg-heatmap}
\end{subfigure}
\\ \vspace{.5em}
\begin{subfigure}[b]{\linewidth}
\centering
\hspace{1em}
\includegraphics[scale=.42]{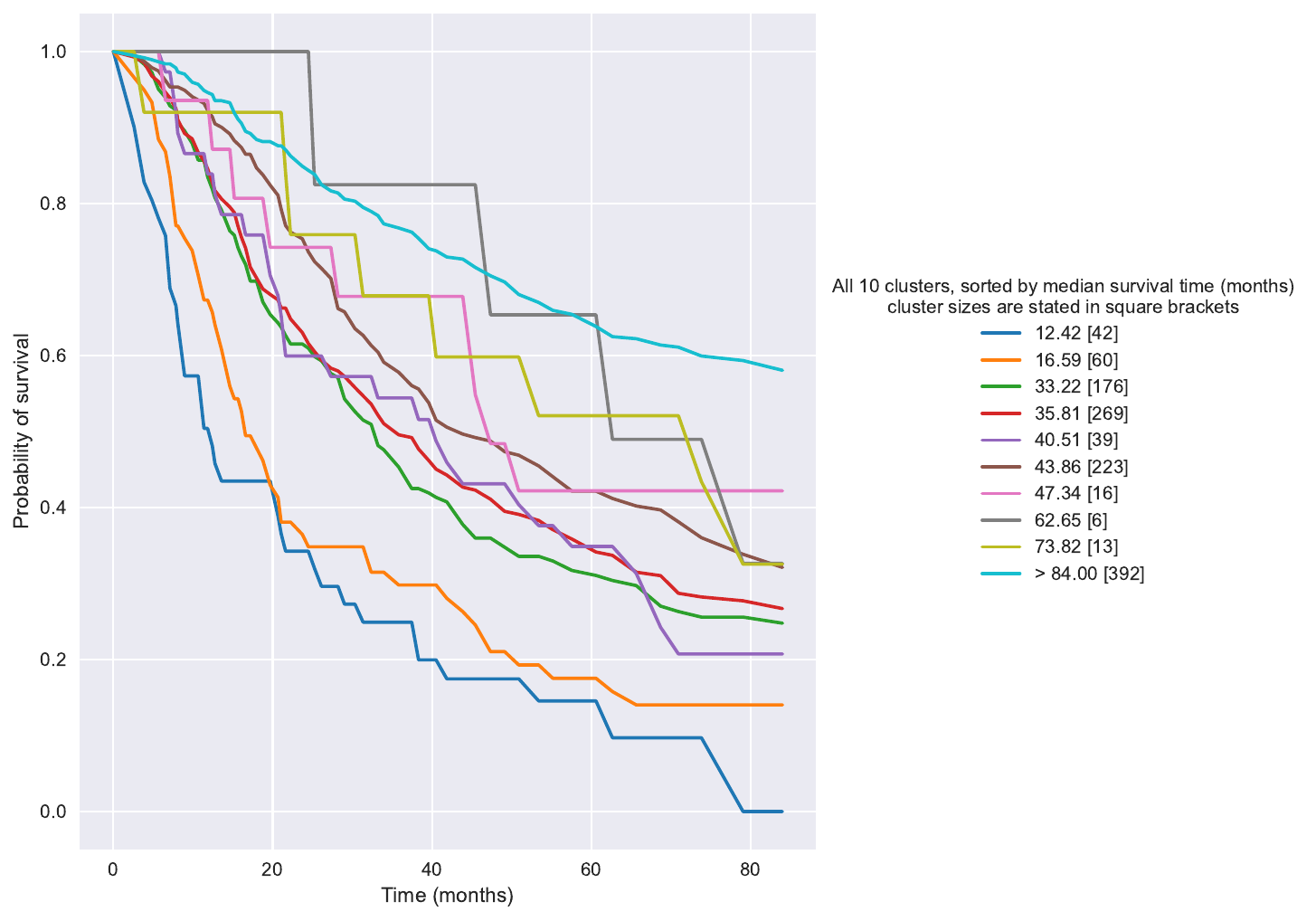}\vspace{-.5em}
\caption{}
\label{subfig:rotterdam-gbsg-survival-curves}
\end{subfigure}
\caption{Visualization of all 10 clusters found by the final  \textsc{tuna-kernet (no split, sft)} model trained on the \textsc{rotterdam/gbsg} dataset (technically trained only on the Rotterdam portion of the data). Panel (a) shows a heatmap visualization that readily provides information on how the clusters are different, highlighting feature values that are prominent for specific clusters; the dotted horizontal lines separate features that correspond to the same underlying variable. Panel (b) shows survival curves (estimated from learned summary functions) for the same clusters as in panel (a); the x-axis indicates recurrence free survival time in months.}
\label{fig:rotterdam-gbsg}
\end{figure}

\paragraph{UNOS dataset}
For the \textsc{unos} dataset, automatic hyperparameter tuning led to a final \textsc{tuna-kernet (no split, sft)} model with a total of 30 clusters that altogether contain 35,080 training points. The largest 5 clusters only account for 87.9\% of the proper training data. For these largest 5 clusters, we plot their feature heatmap and survival curvesin Figure~\ref{fig:unos}. Note that for the heatmap, to prevent it from getting cluttered, we only show 60 rows in the heatmap even though there are a total of 190 rows after feature discretization (as mentioned in Section~\ref{sec:model-interpretability}, we sort features largest to smallest based on the maximum minus minimum intensity value per row although we keep the same variable together, e.g., the different discretized values for ``age'' remain in sequence together).

\begin{figure}[!t]
\centering
\begin{subfigure}[b]{0.48\linewidth}
\centering
\includegraphics[scale=.42]{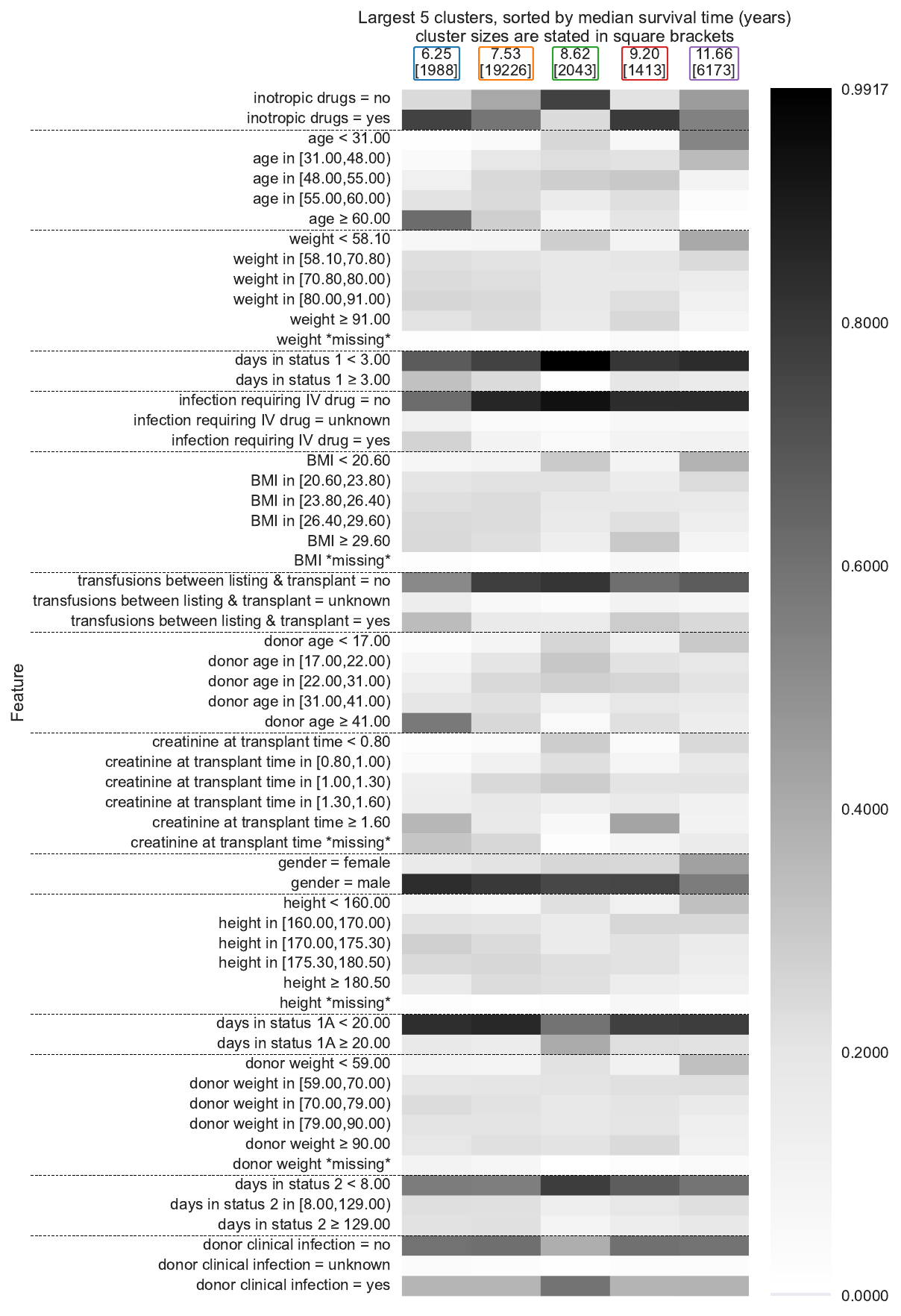}
\caption{}
\label{subfig:unos-heatmap}
\end{subfigure}
~~~
\begin{subfigure}[b]{0.48\linewidth}
\centering
\hspace{1em}
\includegraphics[scale=.42]{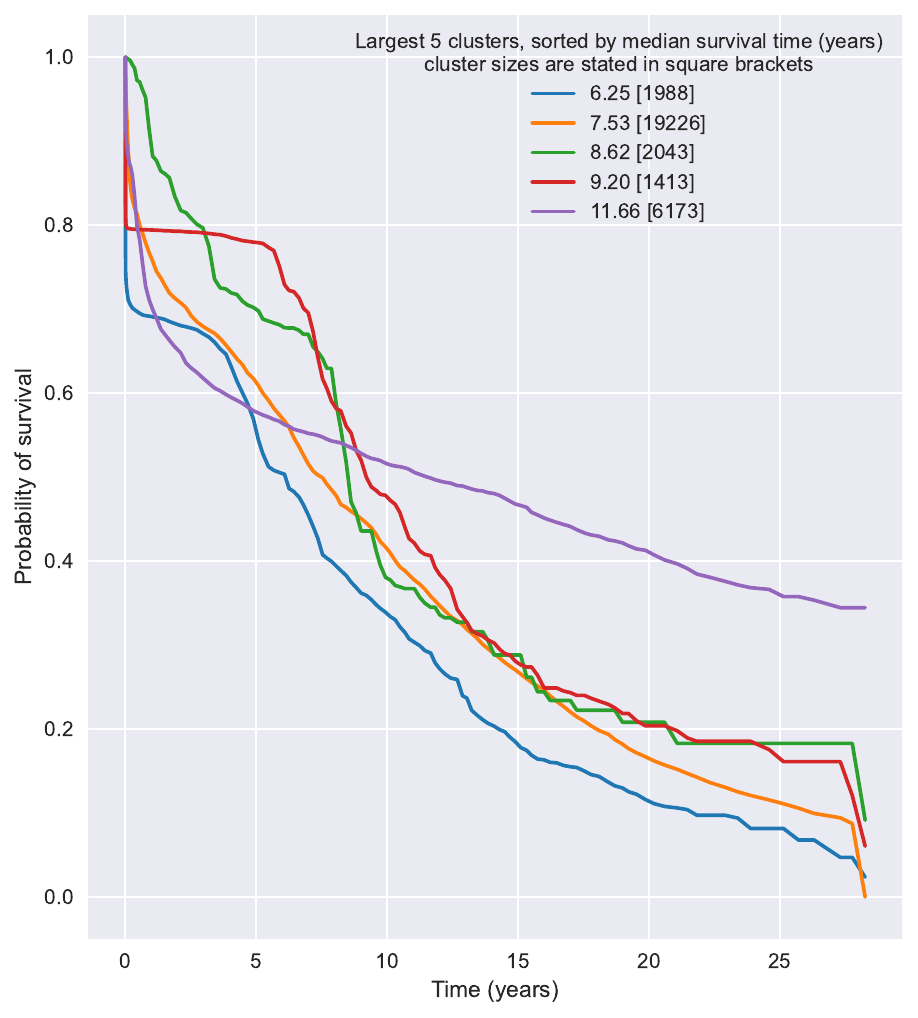}
\caption{}
\label{subfig:unos-survival-curves}
\end{subfigure}
\caption{Visualization of the largest 5 clusters found by the final  \textsc{tuna-kernet (no split, sft)} model trained on the \textsc{unos} dataset. Panel (a) shows a heatmap visualization that readily provides information on how the clusters are different, highlighting feature values that are prominent for specific clusters; the dotted horizontal lines separate features that correspond to the same underlying variable. Panel (b) shows survival curves (estimated from learned summary functions) for the same clusters as in panel (a); the x-axis measures the number of years since a patient received a heart transplant.}
\label{fig:unos}
\end{figure}

From the heatmap in Figure~\ref{fig:unos}\subref{subfig:unos-heatmap}, we can readily tease apart some trends. For instance, of the 5 largest clusters, the one with the lowest median survival time contains patients who are older and whose heart donors are also older. In this cluster, there are also many patients who are either overweight or obese (BMI value of 25 kg/m$^2$ or higher), which is a known risk factor for survival of cardiac transplant patients \citep{russo2010effect}. The rightmost cluster with the highest median survival time contains mostly young patients with low weight, low BMI, and low donor age.

From the survival curves in Figure~\ref{fig:unos}\subref{subfig:unos-survival-curves}, we can tease out some trends among the five clusters shown. For example, the blue cluster (median survival time 6.25 years) is nearly always the worst off among the five clusters shown. Meanwhile, the purple cluster (median survival time 11.66 years) tends to be among the worst off initially (along with the blue cluster) but then has a higher survival probability than the other four clusters shown after around 9 years.

Visualizing all 30 clusters simultaneously would result in a raw feature heatmap and survival curve plot that are too cluttered. We again use complete-linkage agglomerative clustering, this time partitioning the 30 clusters into 10 superclusters, which we then visualize in Figure~\ref{fig:unos-superclusters}. The resulting feature heatmap in Figure~\ref{fig:unos-superclusters}\subref{subfig:unos-superclusters-heatmap} provides a more complete picture of how features vary across the proper training dataset compared to the earlier heatmap in Figure~\ref{fig:unos}\subref{subfig:unos-heatmap}. Note that the rightmost supercluster (with the highest median survival time of 25.13 years) corresponds to a single patient who can be considered an outlier. The five superclusters/columns in the middle of the heatmap (with median survival times 5.08, 5.88, 7.87, 8.62, and 9.20 years) exhibit trends quite similar to what we already saw in Figure~\ref{fig:unos}. On the other hand, the other five superclusters (with median survival times 0.01, 0.04, 15.74, 16.74, and 25.13 years) are all for young patients who receive heart transplants typically from young donors (in fact, we could use a finer-grain discretization of continuous features into more than 5 evenly spaced quantiles to reveal that many patients in these clusters are pediatric patients). Meanwhile, the survival curves in Figure~\ref{fig:unos-superclusters}\subref{subfig:unos-superclusters-survival-curves} show much larger differences compared to the earlier plot in Figure~\ref{fig:unos}\subref{subfig:unos-survival-curves} that was only of the 5 largest clusters.

\begin{figure}[!p]
\centering
\begin{subfigure}[b]{\linewidth}
\centering
\includegraphics[scale=.35]{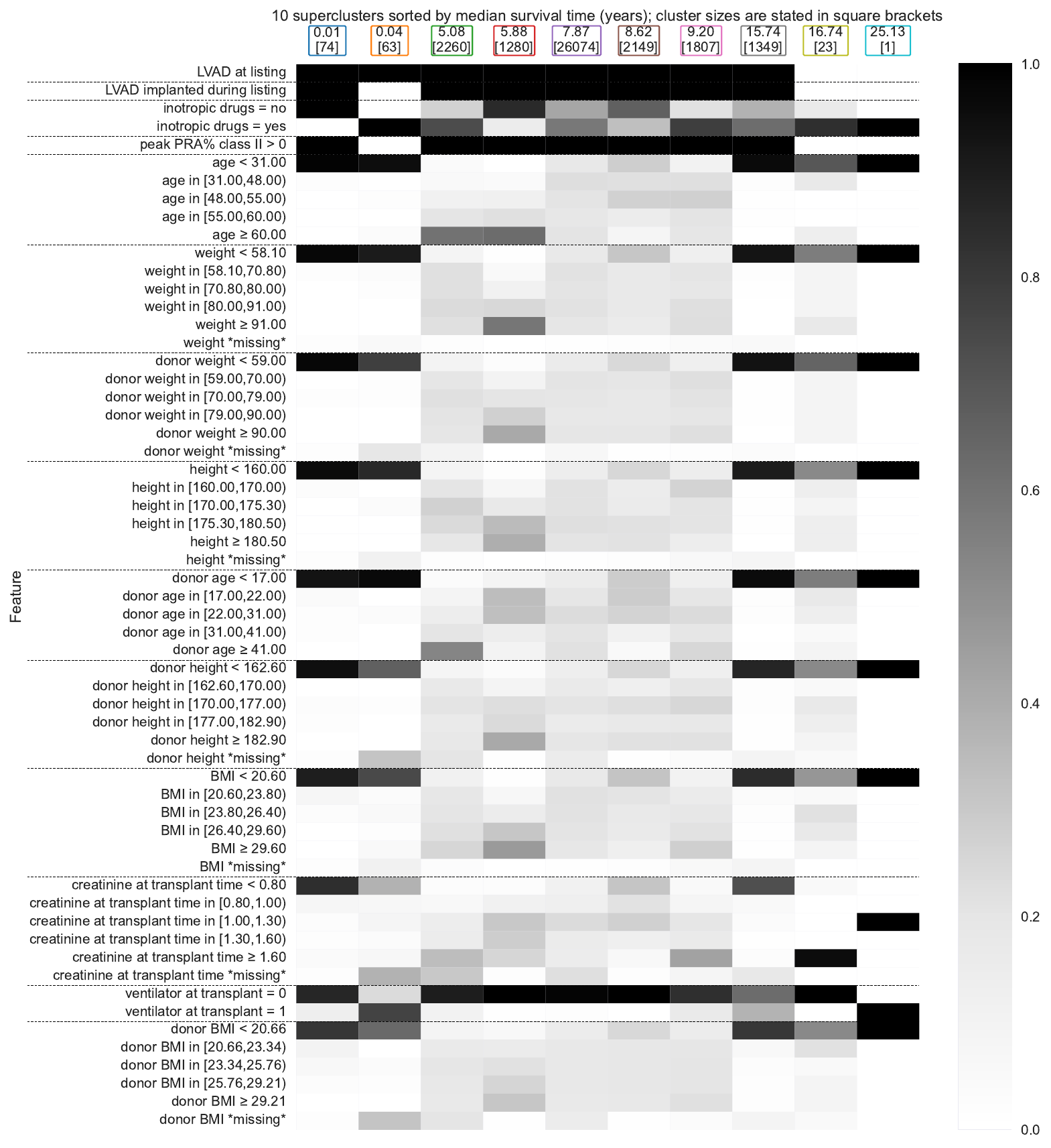}
\caption{}
\label{subfig:unos-superclusters-heatmap}
\end{subfigure} \\ \vspace{.5em}
\begin{subfigure}[b]{\linewidth}
\centering
\hspace{1em}
\includegraphics[scale=.35]{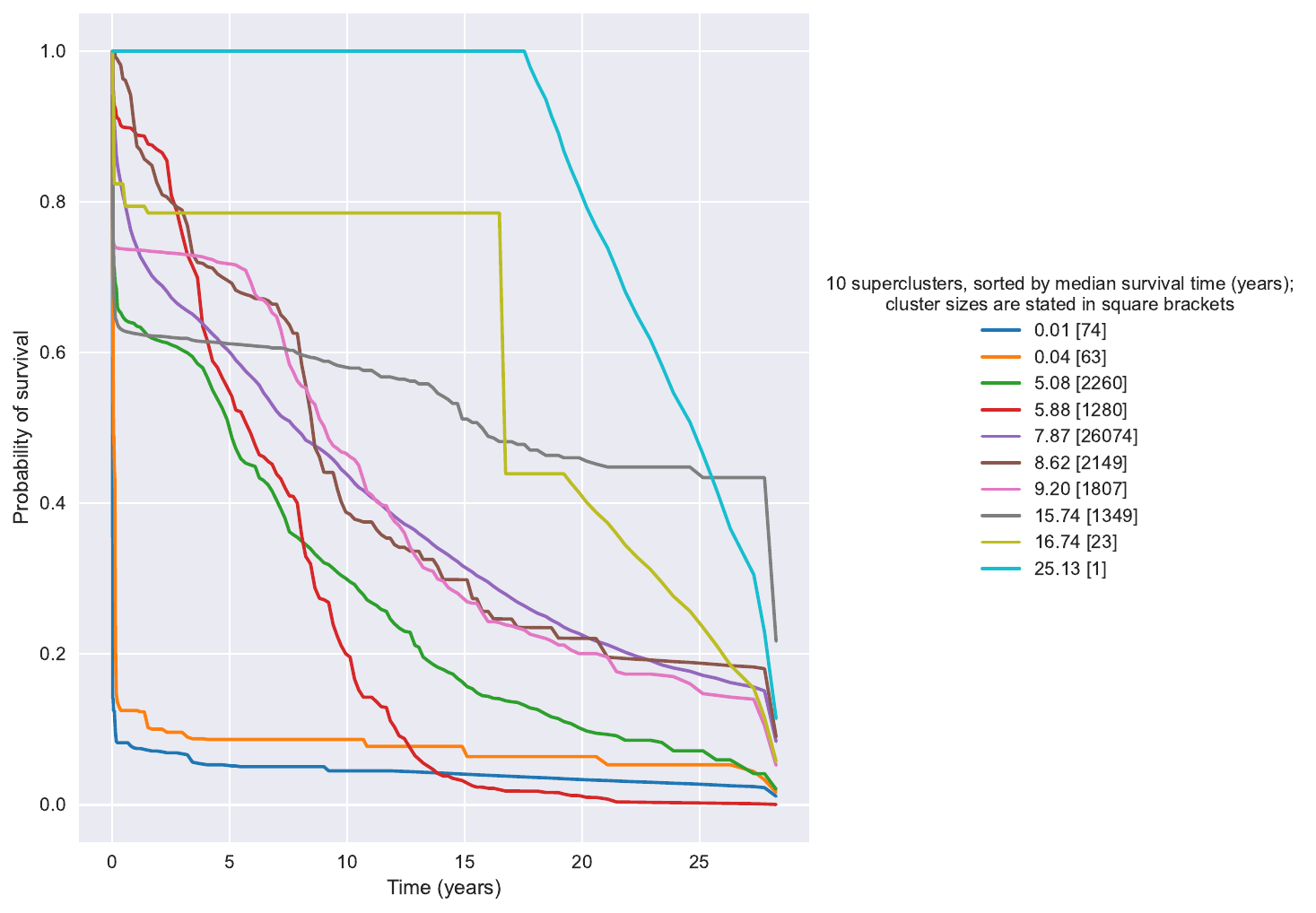}
\vspace{-.5em}
\caption{}
\label{subfig:unos-superclusters-survival-curves}
\end{subfigure}
\caption{Visualization of 10 superclusters (summarizing 30 clusters) for the final \textsc{tuna-kernet (no split, sft)} model trained on the \textsc{unos} dataset. Panel (a) shows a feature heatmap visualization (only 59 rows of features are shown out of 190). Panel (b) shows survival curves for the same superclusters as in panel (a); the x-axis measures the number of years since a patient received a heart transplant.}
\label{fig:unos-superclusters}
\end{figure}

If we focus only on the five superclusters corresponding to young patients and re-plot the raw feature heatmap and survival curves only for these five superclusters, we obtain the plots in Figure~\ref{fig:unos-superclusters-young}. As a reminder, our strategy for ranking raw features in our heatmap visualization depends on which specific (super)clusters are being displayed, which is why the raw features are ranked differently between Figure~\ref{fig:unos-superclusters}\subref{subfig:unos-superclusters-heatmap} and Figure~\ref{fig:unos-superclusters-young}\subref{subfig:unos-superclusters-young-heatmap}. From Figure~\ref{fig:unos-superclusters-young}\subref{subfig:unos-superclusters-young-heatmap}, we see that the differences in these five superclusters could be explained in part by features such as whether a Left Ventricular Assist Device (LVAD) was already in place or implanted at the time of listing, whether inotropic drugs were used, creatinine level at transplant time, and whether a ventilator was used at transplant time; in fact, these features are known to be relevant for survival of pediatric cardiac transplantation \citep{dipchand2018current}.

\begin{figure}[!p]
\centering
\begin{subfigure}[b]{\linewidth}
\centering
\includegraphics[scale=.35]{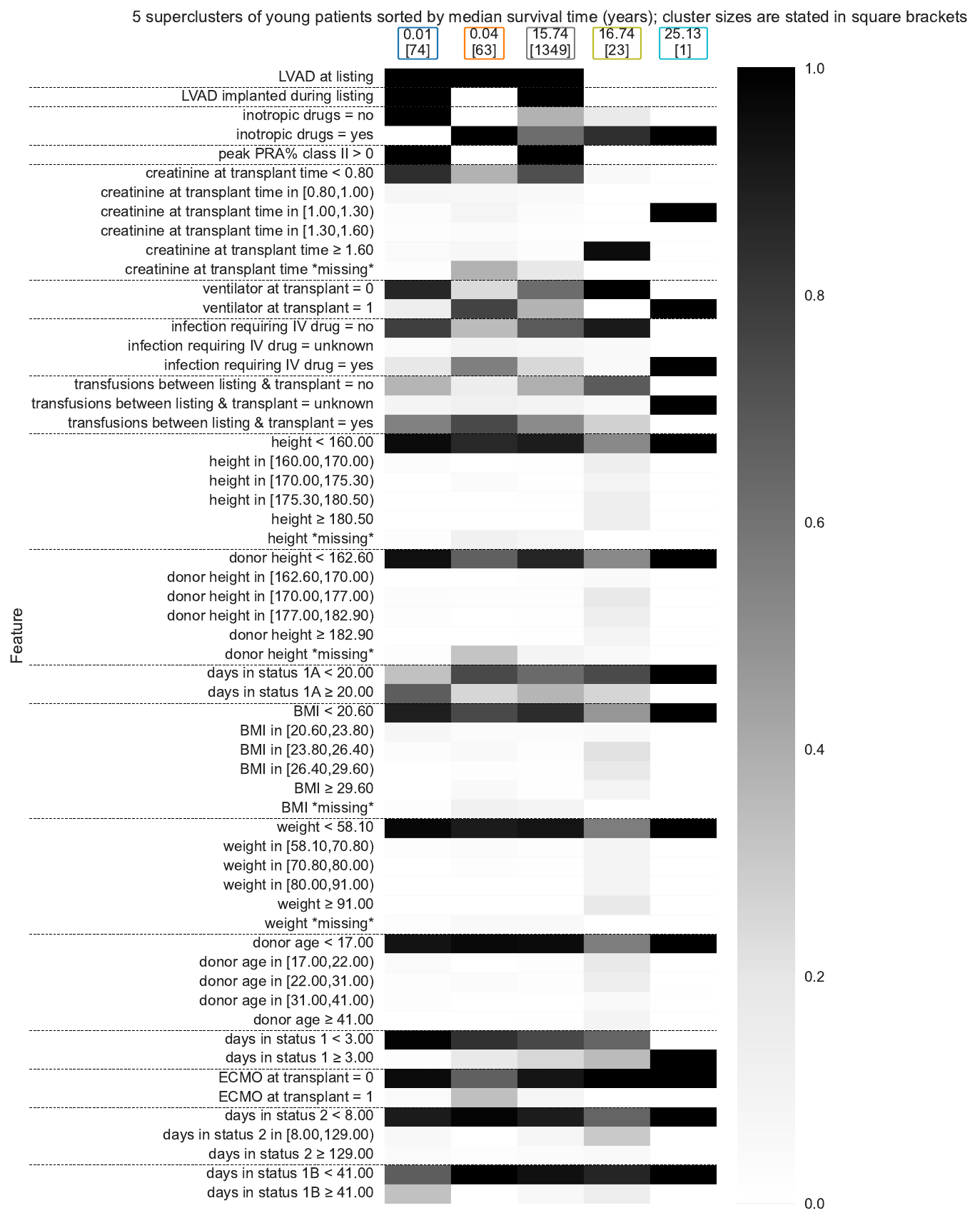}
\caption{}
\label{subfig:unos-superclusters-young-heatmap}
\end{subfigure} \\ \vspace{.5em}
\begin{subfigure}[b]{\linewidth}
\centering
\hspace{1em}
\includegraphics[scale=.35]{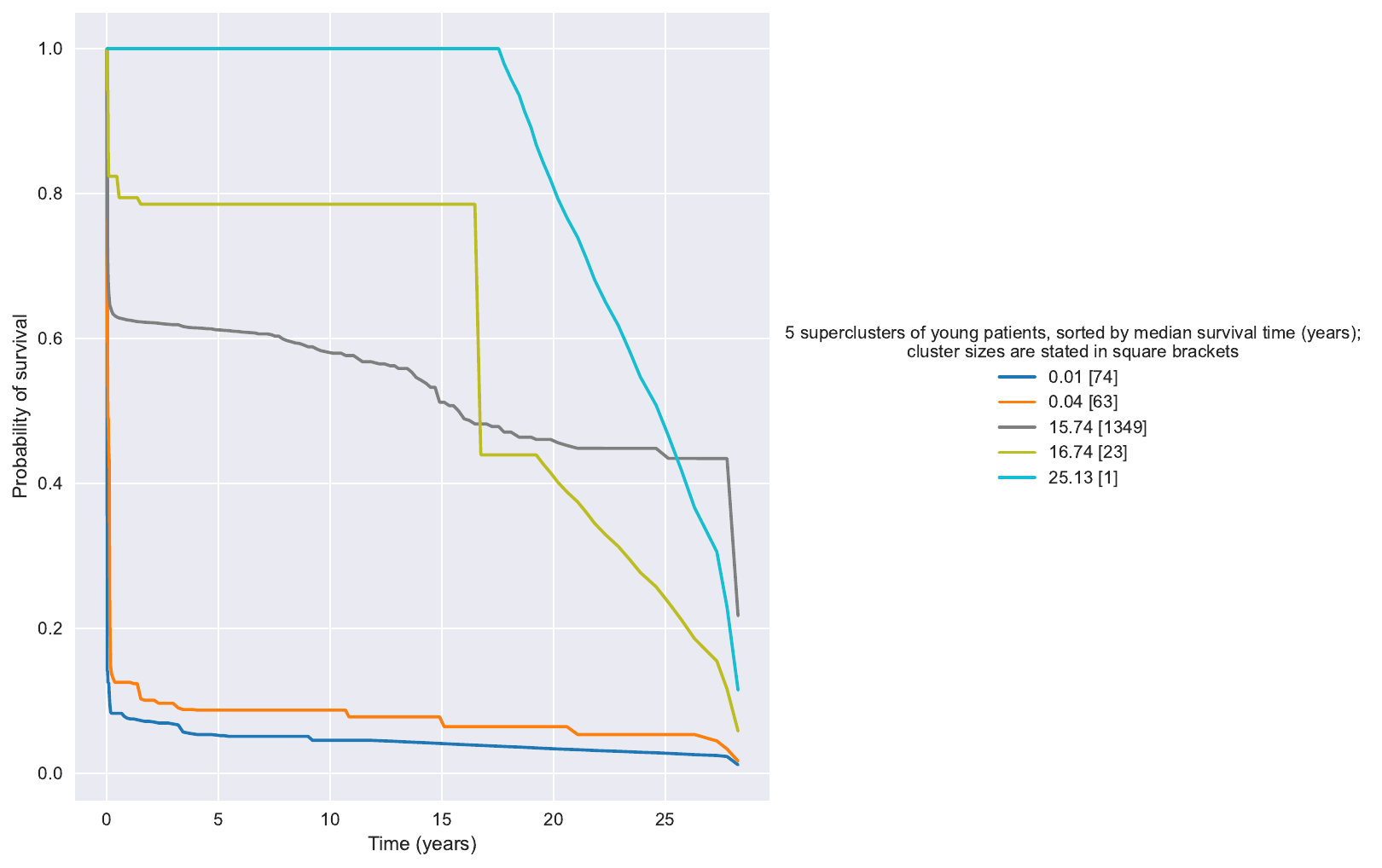}
\vspace{-.5em}
\caption{}
\label{subfig:unos-superclusters-young-survival-curves}
\end{subfigure}
\caption{For the \textsc{unos} dataset, among the 10 superclusters from Figure~\ref{fig:unos-superclusters}, we now only display the superclusters with median survival times 0.01, 0.04, 15.74, 16.74, and 25.13 years corresponding to young patients. Panel (a) shows a feature heatmap visualization (only 59 rows of features are shown out of 190), which has the raw features sorted differently from Figure~\ref{fig:unos-superclusters}\subref{subfig:unos-superclusters-heatmap} to emphasize how these young patients' superclusters differ. Panel (b) shows survival curves for the same superclusters as in panel (a); as before, the x-axis measures the number of years since a patient received a heart transplant.}
\label{fig:unos-superclusters-young}
\end{figure}

\paragraph{KKBOX dataset}
We can repeat the same visualization ideas for the \textsc{kkbox} dataset. As the ideas are the same, we provide only a summary of findings. Using the same supercluster idea presented previously, we visualize 10 superclusters in Figure~\ref{fig:kkbox-superclusters} for the final \textsc{tuna-kernet (no split, sft)} model trained on the \textsc{kkbox} dataset. These superclusters summarize all 108,050 clusters found by the final \textsc{tuna-kernet (no split, sft)} model and contain 1,576,251 data points (the proper training data). We also provide a visualization of just the largest 10 clusters (containing 28.7\% of the proper training data) in Appendix~\ref{sec:additional-visualizations}.
Among superclusters corresponding to users who subscribed to the music streaming service for less than a month, there's a higher fraction of users who are less than 19 years old and who have the lowest amount of payment. Among superclusters corresponding to the longest subscription times, there tends to be fewer previous churns for these users. Many of the survival curves show a steep drop in survival probability at around the 30-day mark, corresponding to a one-month promotion period.

\begin{figure}[!p]
\centering
\begin{subfigure}[b]{\linewidth}
\centering
\includegraphics[scale=.37]{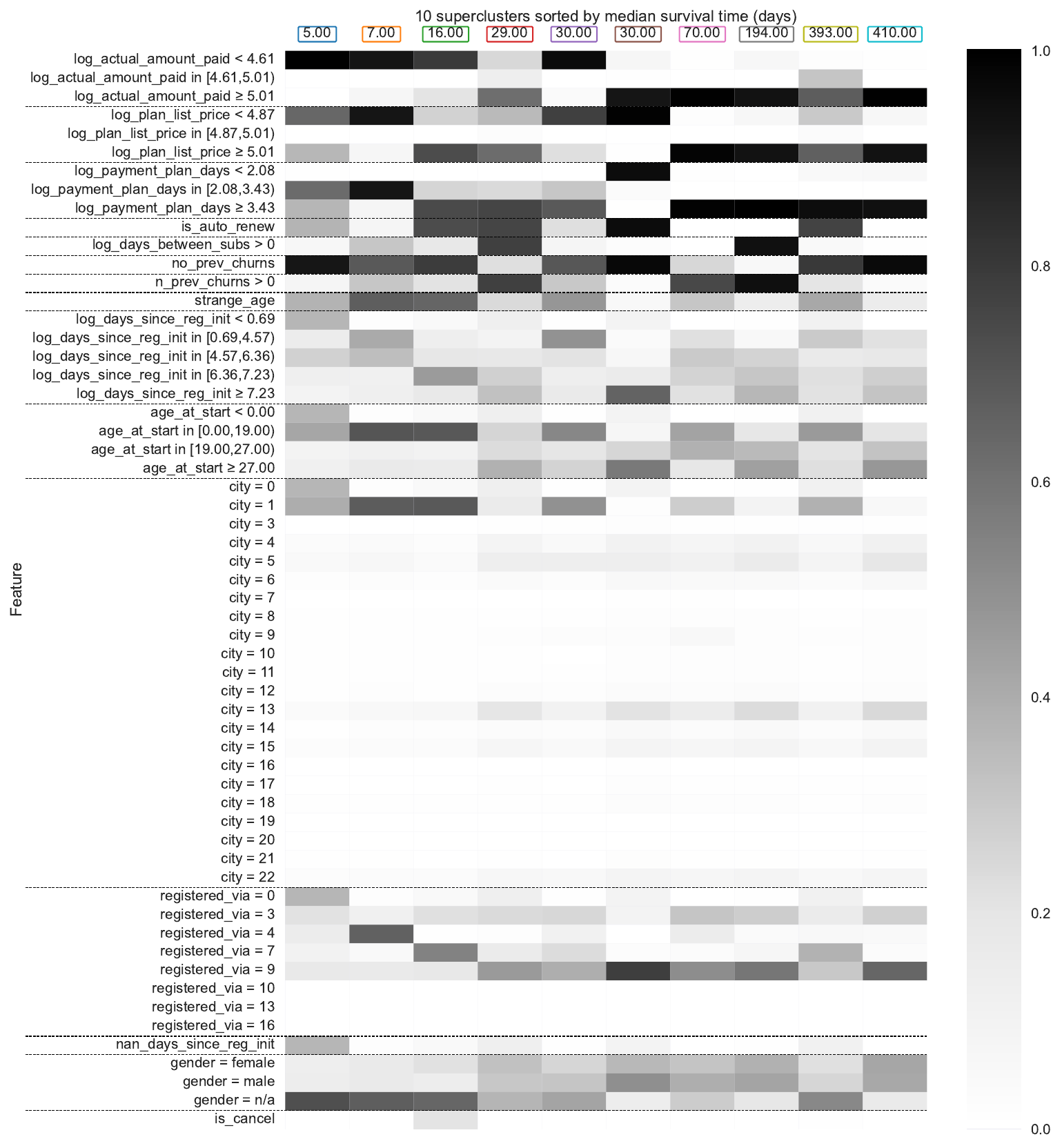}
\caption{}
\label{subfig:kkbox-superclusters-heatmap}
\end{subfigure} \\ \vspace{.5em}
\begin{subfigure}[b]{\linewidth}
\centering
\hspace{1em}
\includegraphics[scale=.37]{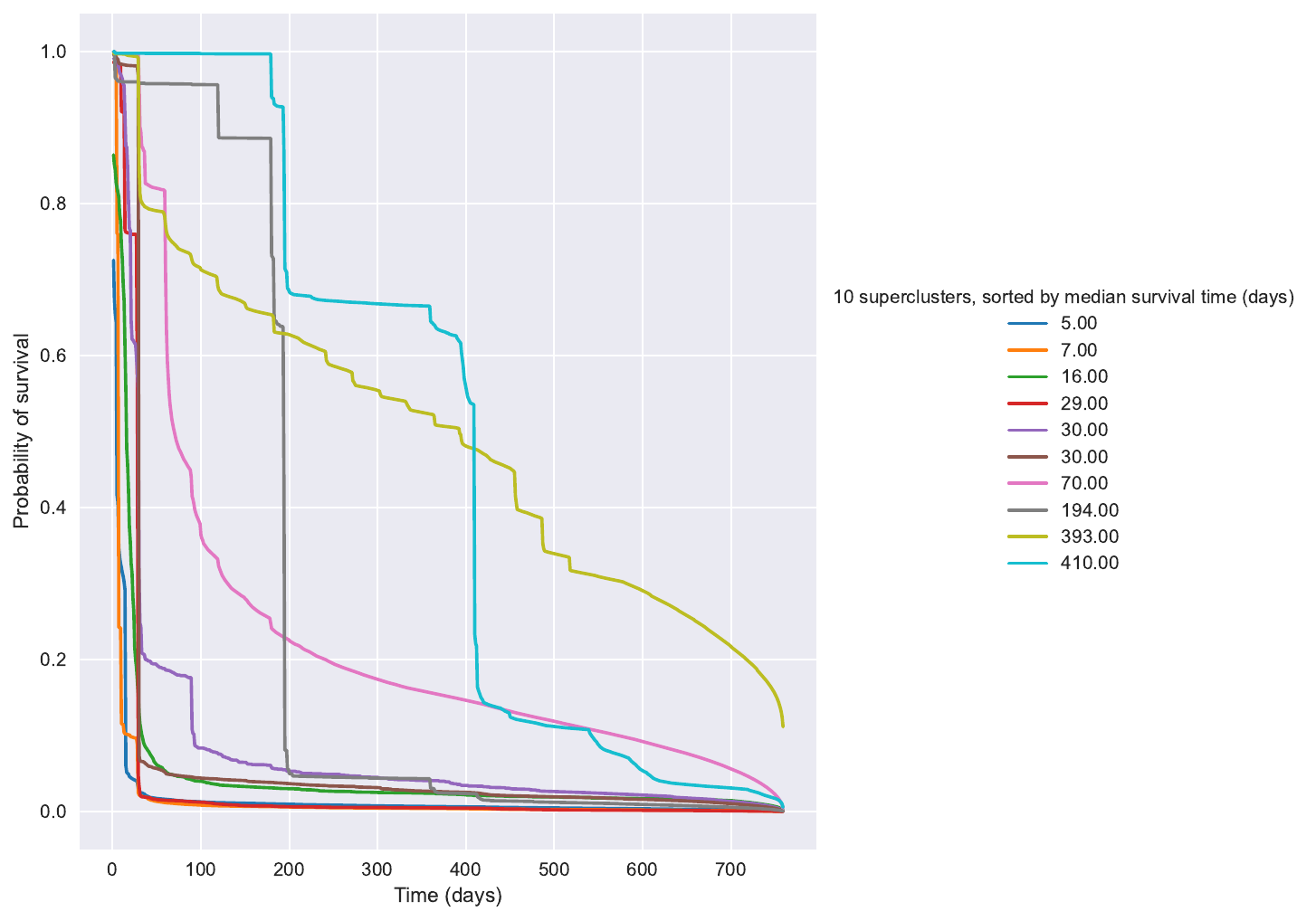}
\vspace{-.5em}
\caption{}
\label{subfig:kkbox-superclusters-survival-curves}
\end{subfigure}
\caption{Visualization of 10 superclusters for the final \textsc{tuna-kernet (no split, sft)} model trained on the \textsc{kkbox} dataset. These 10 superclusters summarize all 108,050 clusters found by the \textsc{tuna-kernet (no split, sft)} model by merging clusters using complete-linkage agglomerative clustering. Panel (a) shows a feature heatmap visualization. Panel (b) shows survival curves for the same superclusters as in panel (a); the x-axis measures the number of days since an individual subscribed to the music streaming service.}
\label{fig:kkbox-superclusters}
\end{figure}

\paragraph{Final remarks on visualization}
We end this section with a reminder that as pointed out at the end of Section~\ref{sec:model-interpretability}, a survival kernet model represents the hazard of any feature vector as a weighted combination of clusters. We can determine which clusters have nonzero weight for any given feature vector and only visualize these particular clusters. This visualization could be helpful to provide as ``forecast evidence'' or to assist model debugging. As an example, we can find test data with predictions that are inconsistent with the ground truth (e.g., if the test data point is not censored and its observed survival time is far from the predicted median survival time, or if the test data point is censored and the predicted median survival time is much lower than the observed time). For these test data that the model has difficulty with, we could examine which clusters have high weight, what features are prominent for these clusters, and what the clusters' survival curves are. After all, the predictions are made with these clusters' summary functions.

\section{Discussion\label{sec:discussion}}

In high-stakes applications such as healthcare, for survival models to be deployed in practice and producing time-to-event predictions using large (potentially live) streams of data in years to come, we believe that these models should be accurate, scalable, robust, interpretable, and theoretically sound. Our proposed survival kernet model achieves many of these properties although some only partially. Specifically, survival kernets are accurate, scalable, interpretable, and for a special case of the model has a theoretical guarantee. We discuss a number of limitations of our work next.

\paragraph{Theoretical analysis of the best-performing survival kernet variant}
As our experimental results show, the best-performing variant of survival kernets sets pre-training and training data to be the same and also uses summary fine-tuning. However, neither of these are covered by our theoretical analysis.

From a technical standpoint, sample splitting (i.e., having pre-training and training data be disjoint and moreover appear i.i.d.)~is crucial in our theoretical analysis. Specifically, our proof of Theorem~\ref{thm:main-result} routinely uses the fact that the training embedding vectors $\widetilde{X}_{1}=\widehat{\phi}(X_{1}),\dots,\widetilde{X}_{n}=\widehat{\phi}(X_{n})$ appear i.i.d. If the pre-training and training data are actually the same, then the learned base neural net $\widehat{\phi}$ is a function of $X_{1},\dots,X_{n}$ (along with their observed times and event indicators), so the embedding vectors are no longer guaranteed to be independent. We do not currently know of a good way to resolve this technical issue. We expect that a proof technique that can address this issue would be quite broadly applicable beyond analyzing survival kernets.\footnote{The idea of using sample splitting to get guarantees really is not limited to our problem setup. Another example of using sample splitting has been in decision forest regression, where a commonly used proof technique (e.g., \citealt{biau2012analysis,denil2014narrowing,wager2018estimation,athey2019generalized})---that translates into how the decision trees in the forest are trained---is to split the training data into two portions: the first portion is used to decide on branching rules for the decision trees, whereas the second portion is used to decide on predicted labels at the leaves of the trees. In practice, random forests are rarely trained using this sort of sample splitting and, in fact, this sample splitting empirically can worsen the model's prediction accuracy in some circumstances (e.g., see the UCI dataset experiments by~\citet{denil2014narrowing}).}

A key reason for why our theoretical analysis requires sample splitting is that our analysis is agnostic to the choice of base neural network $\phi$ used. We leave the choice of $\phi$ up to the modeler since in practice, this is indeed how neural survival analysis models are used. For example, depending on the format of raw inputs, different neural net architectures could be used (e.g., using a multilayer perceptron for tabular data, using a convolutional neural network or vision transformer for images, using a recurrent neural network for variable-length time series). Even for the same class of base neural networks, such as multilayer perceptrons, the modeler could choose to try different options (such as whether to use batch norm). In contrast, existing neural survival analysis models with theoretical guarantees \citep{zhong2021deep,zhong2022deep} require the base neural network to be a multilayer perceptron with a number of constraints that are not typically used by practitioners ($\ell_0$ and sup norm constraints) and, moreover, their theoretical guarantees assume that the neural network training finds the global minimum of the training loss, which is typically an impractical assumption.

As for summary fine-tuning, we suspect that if one treats the learned kernel function as fixed, then theoretical analysis should be possible although the theory would not say anything about the learned kernel function. Of course, the theory that we have presented does not say anything about the learned kernel function either since we treat the base neural net as a black box. Future work could consider the case when the base neural net is chosen from a specific family (e.g., a multilayer perceptron with ReLU activations), which could lead to a more nuanced theoretical guarantee (such as what is done by \citet{zhong2021deep,zhong2022deep} for other deep survival analysis models, although as we already pointed out, their analysis requires some impractical assumptions on the neural net family).

A separate issue related to the theory is whether it is possible to come up with diagnostics that could be practically computed for any learned embedding space to assess to what extent it satisfies the theoretical assumptions relating the embedding space to survival and censoring times. As far as we are aware, there is no existing way to do this even if the neural net is the identity function so that the embedding space is just the raw feature space.

\paragraph{Better handling small datasets}
Survival kernets are inherently nonparametric. They make predictions using all the training data, possibly with some compression. We saw in the experimental results that survival kernets do not work as well as several baselines on the smallest dataset tested (\textsc{rotterdam/gbsg}), for which one possible explanation is that for this dataset, the proportional hazards assumption is reasonable. A future research direction could be to figure out how to combine survival kernets with a parametric model so that when there are too few data points, we mostly use the parametric model (which could, for instance, use a proportional hazards assumption or an accelerated failure time assumption), and as more data become available, we gradually transition to taking a nonparametric approach. Figuring out a seamless way to do this transition could be an interesting research direction to explore.

\paragraph{Better handling low-dimensional structure}
The \textsc{kkbox} dataset has a somewhat peculiar structure not present in the other datasets we considered: the observed survival times for \textsc{kkbox} are always integers in the set $\{1,2,\dots,759\}$; even though there are close to 3 million data points, the number of unique observed times is exactly 759. The number of unique observed times divided by the total number of data points is equal to 0.000270. In contrast, this ratio is significantly higher for all the other datasets considered: the ratio is equal to 0.551 for the \textsc{rotterdam/gbsg} dataset, 0.193 for the \textsc{support} dataset, and 0.149 for the \textsc{unos} dataset. Put another way, the \textsc{kkbox} dataset has a significantly smaller time grid needed to represent observed times, which could be thought of as some low-dimensional structure in time. Of course, there could separately be low-dimensional structure present in the feature vectors themselves.

Currently, survival kernets do not have any sort of special procedure for exploiting low-dimensional structure. As already stated above, survival kernets are inherently nonparametric and compute predictions based on similarity scores to training points. If training feature vectors are noisy, then if we do not denoise the feature vectors somehow (e.g., using singular value thresholding as is done in principal component regression \citep{agarwal2021robustness}), then we suspect that survival kernets could struggle to learn a good model compared to a simpler parametric model. While using an $\varepsilon$-net to cluster the training data can help denoise, right now we do not have an efficient manner that quickly determines what value of $\varepsilon$ should be used. Moreover, this $\varepsilon$-net construction is only done after the base neural net $\phi$ has already been trained, which means relying on the $\varepsilon$-net to denoise would not affect neural net training itself (for which how we are currently learning the base neural net using the infinite-bandwidth kernel is not explicitly incorporating any denoising mechanism).

\paragraph{Improving the warm-start strategy and hyperparameter tuning}
Another possible explanation for why survival kernets does not achieve the best accuracy on the \textsc{kkbox} dataset could be that for \textsc{kkbox}, the \textsc{xgboost} initialization is not very good. Note that even when using \textsc{xgboost} for survival analysis, there are a number of hyperparameters that can be set, and for simplicity, currently we are just using Cox regression as \textsc{xgboost}'s objective. At the time of writing, \textsc{xgboost} also supports using accelerated failure time models with Gaussian, logistic, and extreme value distributions. Sweeping over more of these options could improve \textsc{xgboost}'s prediction accuracy which could in turn improve the accuracy of a survival kernet model when using the \textsc{tuna} warm-start procedure with \textsc{xgboost}.

Separately, as already pointed out by \citet{chen2020deep} in the original DKSA paper, using a tree ensemble to warm-start neural net training is not required. An alternative approach is to use the base neural net learned from, for instance, \textsc{deephit} (possibly removing a subset of the final layers of the neural net) and then fine-tuning using the DKSA loss.

Ultimately, although we have demonstrated that our \textsc{tuna} warm-start strategy is clearly better than using standard random neural net parameter initialization, figuring out the best warm-start strategy to use for survival kernets (which might be dataset-dependent) remains an interesting open question. For example, one research direction could be to see whether contrastive representation learning (e.g., see the survey by \citet{le2020contrastive}) could be used to warm-start survival kernet training and, if so, how well it works.

More generally, the \textsc{tuna} warm-start strategy also aims to save time with hyperparameter selection by having some hyperparameters tuned as part of the warm-start procedure and then treated as fixed afterward. We leave a more thorough investigation of how to optimize hyperparameters for future work. For simplicity, when sweeping over hyperparameters, we used grid search. We did not explore other strategies such as Bayesian hyperparameter optimization (e.g., \citealt{akiba2019optuna}) nor did we carefully try to determine if some hyperparameters simply do not need to be tuned at all (in that there is a reasonable default value that can be used).

\paragraph{Impact of clustering on accuracy and interpretability}
For survival kernets, one could easily replace the $\varepsilon$-net-based clustering with another clustering approach. We chose the $\varepsilon$-net-based clustering for theoretical convenience: $\varepsilon$-nets can easily be constructed in the greedy manner we had mentioned, and importantly, the resulting clusters come with theoretical properties. Consequently, $\varepsilon$-nets have been used to prove that ``compressed'' versions of nearest neighbor or kernel regression are consistent \citep{kpotufe2017time,kontorovich2017nearest,hanneke2021universal}. In contrast, had we used, for instance, \mbox{k-means} for clustering, then even though k-means at this point also has a lot of known theory (e.g., \citealt{rakhlin2006stability,ben2007stability,kumar2010clustering,von2010clustering,franti2019much}), it can be quite unstable in practice and the resulting clusters do not in general come with theoretical assurances. As such, we suspect that using k-means clustering with survival kernets would lead to a test-time predictor that is harder to theoretically analyze.

Putting aside theoretical convenience, by empirically evaluating how a wider range of clustering algorithms work with our survival kernet framework, we could potentially find that some of these lead to better prediction accuracy or better model interpretability compared to using $\varepsilon$-net clustering. We remark that especially as our experiments use embedding vectors that are on a hypersphere, then hyperspherical clustering methods could be used such as fitting a mixture of von Mises-Fisher distributions \citep{banerjee2005clustering}. We leave an empirical study of using different clustering methods with survival kernets for future work.

\paragraph{Handling outliers}
Separately, we point out that even with our $\varepsilon$-net clustering approach, in practice many clusters could have very few data points assigned to them. Currently we are not removing such small clusters, although it might make sense to do so or to somehow flag these clusters as outliers and treat them a bit differently. Even pointing these clusters out to the user and visualizing them using the cluster visualization approach we had presented could be useful for model debugging purposes. In fact, this problem of addressing outliers would also arise if we replace the clustering method with, for instance, DBSCAN \citep{ester1996density}, which automatically flags various data points as outliers as part of the clustering procedure.

\paragraph{Calibration}
When a survival model's predicted number of critical events within a time interval closely resembles the observed number, then the model is considered well-calibrated \citep{haider2020effective}. This could be thought of as a different notion of accuracy than the one we used in our experiments (namely time-dependent concordance index by \citet{antolini2005time}). We have not considered calibration in our paper, which in practice can be important. A straightforward way to incorporate calibration is to introduce an additional loss term such as the X-CAL loss by \citet{goldstein2020x}. We leave a thorough investigation of the calibration properties of survival kernets, with and without the addition of the X-CAL loss, to future work.

\paragraph{Robustness}
We have not discussed robustness of survival kernets thus far although in some sense they already have a limited amount of robustness built-in.
Specifically, survival kernets do not assume a parametric form for the underlying survival distribution; the base neural net is used only to parameterize the kernel function, which is then plugged into the nonparametric conditional Kaplan-Meier estimator. By being nonparametric, survival kernets should be more robust to the data coming from different survival distributions. In fact, our theoretical guarantee works under a fairly broad range of settings. However, survival kernets are currently not designed to handle changes in the data distribution such as covariate shift, where test feature vectors come from a different distribution than that of training but how feature vectors relate to survival and censoring times remains unchanged at test time. Modifying survival kernets to better accommodate test data appearing different from training data is an important future research direction to explore.

\section*{Acknowledgments}
This work was supported in part by NSF CAREER award \#2047981, and by Health Resources and Services Administration contract 234-2005-370011C. The content is the responsibility of the author alone and does not necessarily reflect the views or policies of the National Science Foundation or the Department of Health and Human Services, nor does mention of trade names, commercial products, or organizations imply endorsement by the U.S.~Government. We thank Shu Hu for pointing out some typos in an earlier draft, and we thank the anonymous reviewers for very helpful feedback.

\appendix

\section{Proofs}
\label{sec:proofs}

We present the proofs for Claim~\ref{claim:unit-hypersphere-covering-number} (unit hypersphere's covering number) and Claim~\ref{claim:uniform-unit-hypersphere-intrinsic-dimension} (intrinsic dimension of the unit hypersphere) in Appendices~\ref{sec:pf-prop-unit-hypersphere-covering-number} and~\ref{sec:pf-prop-uniform-unit-hypersphere-intrinsic-dimension}, respectively. The proof of our main theoretical result Theorem~\ref{thm:main-result} (survival kernet error bound) is in Appendix~\ref{sec:pf-main-result}.

\subsection{Proof of Claim~\ref{claim:unit-hypersphere-covering-number}}
\label{sec:pf-prop-unit-hypersphere-covering-number}

Let $\{\zeta_{1},\zeta_{2},\dots,\zeta_{N}\}$ be an $\varepsilon/2$-covering of the $d$-dimensional unit Euclidean ball (denoted by $\mathcal{B}_{d})$ such that $N$ is the smallest value possible, i.e., $N={N(\varepsilon/2;\mathcal{B}_{d})}$; note that the unit Euclidean ball includes the interior of the ball whereas the unit hypersphere $\mathbb{S}^{d-1}$ is only the outer shell. We recall a standard result of covering numbers (Corollary 4.2.13 of \citet{vershynin2018high}) that
\begin{equation}
N=N(\varepsilon/2;\mathcal{B}_{d})\le\Big(\frac{4}{\varepsilon}+1\Big)^{d}.\label{eq:basic-euclidean-ball-delta-2-covering-number}
\end{equation}
Then since $\mathbb{S}^{d-1}\subseteq\mathcal{B}_{d}$, for every $v\in\mathbb{S}^{d-1}$, there exists some $\zeta_j$ such that ${\|v-\zeta_{j}\|}\le\varepsilon/2$. This implies that there exists a subset $\Xi$ of $\{\zeta_{1},\zeta_{2},\dots,\zeta_{N}\}$ such that every point in $\mathbb{S}^{d-1}$ is at most a distance $\varepsilon/2$ from some point in $\Xi$. Let $\Xi^{*}$ be such a subset that has the smallest size, and denote its unique elements as $\xi_{1},\xi_{2},\dots,\xi_{|\Xi^{*}|}$.

We next show how to construct an $\varepsilon$-cover for $\mathbb{S}^{d-1}$ using $\xi_{1},\xi_{2},\dots,\xi_{|\Xi^{*}|}$. For each point $\xi_{i}$, note that there exists some $s_{i}\in\mathbb{S}^{d-1}$ such that $s_{i}\in B(\xi_{i},\varepsilon/2)$ (if such an $s_{i}$ does not exist, then $\xi_{i}$ would not have been included in $\Xi^{*}$, i.e., $\Xi^{*}$ does not actually have the smallest size possible while still covering $\mathbb{S}^{d-1}$). Next, we observe that for any $v\in B(\xi_{i},\varepsilon/2)$, the triangle inequality implies that
\[
\|s_{i}-v\|_{2}=\|s_{i}-\xi_{i}+\xi_{i}-v\|_{2}\le\underbrace{\|s_{i}-\xi_{i}\|_{2}}_{\le\varepsilon/2}+\underbrace{\|\xi_{i}-v\|_{2}}_{\le\varepsilon/2},
\]
which means that $v$ is also in $B(s_{i},\varepsilon)$. Hence,
\[
B(s_{i},\varepsilon)\supseteq B(\xi_{i},\varepsilon/2).
\]
This means that $\cup_{i=1}^{|\Xi^{*}|}B(s_{i},\varepsilon)\supseteq\cup_{i=1}^{|\Xi^{*}|}B(\xi_{i},\varepsilon/2)$, and since the latter covers $\mathbb{S}^{d-1}$, so does the former. We thus conclude that $\{s_{1},s_{2},\dots,s_{|\Xi^{*}|}\}$ forms a valid $\varepsilon$-cover of $\mathbb{S}^{d-1}$. The $\varepsilon$-cover of $\mathbb{S}^{d-1}$ that has the smallest size possible is thus upper-bounded by $|\Xi^{*}|\le N\le\big(\frac{4}{\varepsilon}+1\big)^{d}$ using inequality \eqref{eq:basic-euclidean-ball-delta-2-covering-number}.$\hfill\blacksquare$

\subsection{Proof of Claim~\ref{claim:uniform-unit-hypersphere-intrinsic-dimension}}
\label{sec:pf-prop-uniform-unit-hypersphere-intrinsic-dimension}

Let $f_{\widetilde{\mathsf{X}}}$ denote the PDF of $\mathbb{P}_{\widetilde{\mathsf{X}}}$, where we assume that $f_{\widetilde{\mathsf{X}}}(\widetilde{z})\ge c_{\min}>0$ for all $\widetilde{z}\in\mathbb{S}^{d-1}$, and $f_{\widetilde{\mathsf{X}}}(\widetilde{z})=0$ for all $\widetilde{z}\notin\mathbb{S}^{d-1}$. Let $\widetilde{x}\in\widetilde{\mathcal{X}}$ and $r\in(0,1]$. Then
\[
\mathbb{P}_{\widetilde{\mathsf{X}}}(B(\widetilde{x},r))
= \int_{B(\widetilde{x}, r)} f_{\widetilde{\mathsf{X}}}(\widetilde{z}) d\widetilde{z}
\ge
\int_{B(\widetilde{x}, r)} c_{\min} \cdot d\widetilde{z}
= c_{\min} \cdot \text{area}\big(\underbrace{\mathbb{S}^{d-1}\cap B(\widetilde{x},r)}_{\mathcal{C}(\widetilde{x},r):=}\big).
\]
Note that $\mathcal{C}(\widetilde{x},r)$ is a hyperspherical cap. We make a geometric observation:
\[
\text{area}\big(\mathcal{C}(\widetilde{x},r)\big)\ge\text{volume}\big(%
\underbrace{\mathcal{C}(\widetilde{x},r)\text{ projected down to }\mathbb{R}^{d-1}\text{ in the direction of }\widetilde{x}}_{\mathcal{A}:=}\big),
\]
where $\mathcal{A}$ is depicted as a solid shaded green circle in Figure~\ref{fig:unit-hypersphere-intrinsic-dimension-pf-projection}. The core of the remaining analysis is in determining the volume of $\mathcal{A}$. With a bit of trigonometry, we can derive that $\mathcal{A}$ is precisely equal to a ball centered at the origin with radius $\frac{1}2r\sqrt{4-r^2}$; see Figure~\ref{fig:unit-hypersphere-intrinsic-dimension-pf-trig}. Since $r\in(0,1]$, the radius of $\mathcal{A}$ satisfies the bound $\frac{1}2r\sqrt{4-r^2}\ge\frac{1}2r\sqrt{3}$. Thus, the volume of $\mathcal{A}$ is lower-bounded by the volume of a ball in $\mathbb{R}^{d-1}$ with radius $\frac{\sqrt{3}}2r$, which is $\frac{\pi^{(d-1)/2}}{\Gamma(\frac{d-1}2+1)}(\frac{\sqrt{3}}2r)^{d-1}$, where $\Gamma(z):=\int_{0}^{\infty}u^{z-1}e^{-u}du$. Putting everything together,
\begin{align*}
\mathbb{P}_{\widetilde{\mathsf{X}}}(B(\widetilde{x},r))
&\ge c_{\min}\cdot\text{area}(\mathcal{C}(\widetilde{x},r)) \\
&\ge c_{\min}\cdot\text{volume}(\mathcal{A}) \\
&\ge c_{\min}\cdot\frac{\pi^{(d-1)/2}}{\Gamma(\frac{d-1}2+1)}\Big(\frac{\sqrt{3}}2r\Big)^{d-1} \\
&= \Big[c_{\min}\cdot\frac{\pi^{(d-1)/2}}{\Gamma(\frac{d-1}2+1)}\Big(\frac{\sqrt{3}}2\Big)^{d-1}\Big] r^{d-1},
\end{align*}
which establishes that the embedding space has intrinsic dimension $d-1$.

\begin{figure}[!t]
\centering
\includegraphics[scale=0.55]{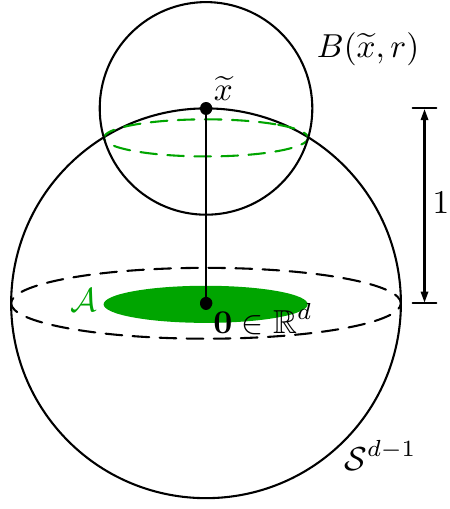}
\caption{This diagram shows the geometric observation in the proof for the case of $d=3$ dimensional Euclidean space. The intersection of the closed ball $B(\widetilde{x},r)$ with the unit hypersphere $\mathbb{S}^{d-1}$ is a hyperspherical cap, with the ``top'' of the cap at $\widetilde{x}$, and the bottom is shown in the dotted green line. The surface area of this cap is lower-bounded by the area of the projection (the solid green circle $\mathcal{A}$; in higher dimensions, this is a ball in $\mathbb{R}^{d-1}$).\label{fig:unit-hypersphere-intrinsic-dimension-pf-projection}}
\end{figure}

\begin{figure}[!t]
\centering
\begin{subfigure}[b]{.48\linewidth}
\centering
\includegraphics[scale=0.55]{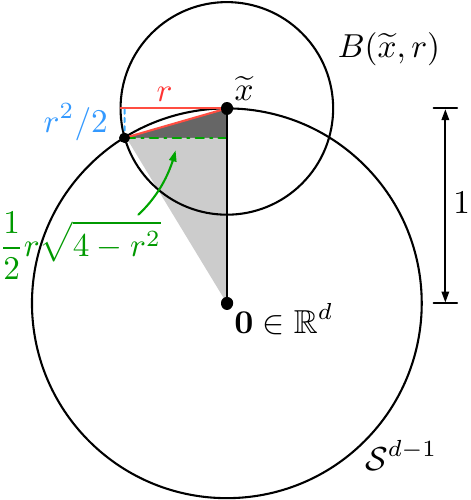}
\caption{}
\end{subfigure}
~
\begin{subfigure}[b]{.48\linewidth}
\centering
\includegraphics[scale=0.55]{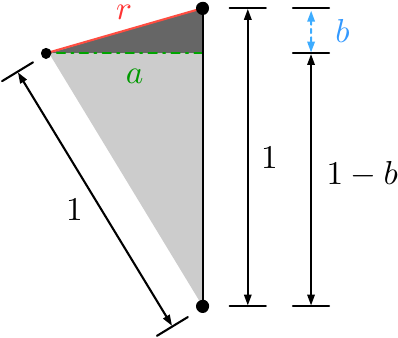}
\caption{}
\end{subfigure}
\caption{The radius of both the dotted and solid green circles in Figure~\ref{fig:unit-hypersphere-intrinsic-dimension-pf-projection} can be computed to be equal to $\frac{1}{2}r\sqrt{4-r^2}$. Specifically, in panel (a), we show a 2D view of Figure~\ref{fig:unit-hypersphere-intrinsic-dimension-pf-projection}. To verify that the lengths specified in red, blue, and green in panel (a) are correct, we show the two shaded right triangles from panel (a) with more detail in panel (b), where we denote the green and blue lengths as $a$ and $b$ respectively: the Pythagorean theorem says that for the darker triangle, we have $a^2+b^2=r^2$, and for the lighter triangle, we have $a^2+(1-b)^2=1^2$; by solving this system of two equations, we get $a=\frac{1}{2}r\sqrt{4-r^2}$ and $b=\frac{r^2}{2}$. \label{fig:unit-hypersphere-intrinsic-dimension-pf-trig}}
\end{figure}

\subsection{Proof of Survival Kernet Error Bound (Theorem~\ref{thm:main-result})}
\label{sec:pf-main-result}

To keep the notation from getting cluttered, we use different notation than what is in the main paper. We drop tildes and instead write the embedding space as $\mathcal{U}\subset\mathbb{R}^d$. The training (and not pre-training) embedding vectors (treated as random variables) are written as $U_{1},U_{2},\dots,U_{n}\in\mathcal{U}$. Just as in the main paper, we use the notation $U_{1:n}=(U_{1},\dots,U_{n})\in\mathcal{U}^{n}$. The random test data point's embedding vector is denoted by the random variable $U$. The marginal distribution of embedding vectors is denoted by $\mathbb{P}_{\mathsf{U}}$. We denote the $\varepsilon$-net as $Q$ (treated as a set), which is a subsample of the indices in $U_{1:n}$ (treated as a sequence). Just as in the main paper, we denote $\mathcal{I}_{q}\subseteq\{1,2,\dots,n\}$ to be the indices of the training points assigned to $q\in Q$.
We denote the true conditional survival, censoring, and observed time distributions as $\mathbb{P}_{\mathsf{T}|\mathsf{U}}$, $\mathbb{P}_{\mathsf{C}|\mathsf{U}}$, and $\mathbb{P}_{\mathsf{Y}|\mathsf{U}}$. We write their PDFs as $f_{\mathsf{T}|\mathsf{U}}(t|u)$, $f_{\mathsf{C}|\mathsf{U}}(t|u)$, and $f_{\mathsf{Y}|\mathsf{U}}(t|u)$; their CDFs as $F_{\mathsf{T}|\mathsf{U}}(t|u)$, $F_{\mathsf{C}|\mathsf{U}}(t|u)$, and $F_{\mathsf{Y}|\mathsf{U}}(t|u)$; and their tail functions as $S_{\mathsf{T}|\mathsf{U}}(t|u)=1-F_{\mathsf{T}|\mathsf{U}}(t|u)$, $S_{\mathsf{C}|\mathsf{U}}(t|u)=1-F_{\mathsf{C}|\mathsf{U}}(t|u)$, and $S_{\mathsf{Y}|\mathsf{U}}(t|u)=1-F_{\mathsf{Y}|\mathsf{U}}(t|u)$. For a test embedding vector~$u$, the conditional survival function we aim to predict is precisely $S_{\mathsf{T}|\mathsf{U}}(t|u)$ for $t\ge0$.

For survival kernets, the only training points that can possibly contribute to the prediction for test embedding vector~$u$ (i.e., they have nonzero truncated kernel weight) are the ones with indices in
\[
\mathcal{N}_Q(u):=\coprod_{q\in Q\text{ s.t.~}\|u-q\|_{2}\le\tau}\mathcal{I}_{q},
\]
where ``$\coprod$'' denotes disjoint union. This union is disjoint since the survival kernet training procedure assigns each training point to exactly one exemplar in~$Q$. The set $\mathcal{N}_Q(u)$ could be thought of as the ``neighbors'' of $u$.

For $Y_{\max}(u):=\max_{j\in\mathcal{N}_Q(u)}Y_{j}$, the survival kernet hazard function estimate (equation~\eqref{eq:kernel-netting-for-survival-analysis-hazard}) can be written as follows, using the new notation:
\begin{align}
\widehat{g}_{Q}(t|u) & :=\begin{cases}
{\displaystyle \frac{\sum_{j\in\mathcal{N}_Q(u)}K(\|u-q_j\|_2)\eventInd_{j}\ind\{Y_{j}=t\}}{\sum_{j\in\mathcal{N}_Q(u)}K(\|u-q_j\|_{2})\ind\{Y_{j}\ge t\}}}\\
\phantom{0}\qquad\text{if }\mathcal{N}_Q(u)\ne\emptyset\text{ and }0\le t\le Y_{\max}(u),\\
0\qquad\text{otherwise}.
\end{cases}
\label{eq:hazard-estimate-appendix-form}
\end{align}
Let $\mathcal{T}$ be the set of (unique) times until death within the training data, i.e., $\mathcal{T}:=\{t'\in[0,\infty):\exists j\in\{1,\dots,n\}\text{ s.t.~}Y_{j}=t'\text{ and }\eventInd_{j}=1\}$. Then the survival kernet conditional survival function estimator can be written as
\[
\widehat{S}_{\mathsf{T}|\mathsf{U}}(t|u):=\prod_{t'\in\mathcal{T}}(1-\widehat{g}_{Q}(t'|u))^{\ind\{t'\le t\}}.
\]
In the main paper, this estimator is written as $\widetilde{S}_{\widetilde{\mathcal{Q}}_{\varepsilon}}(t|\widetilde{x})$. Using the new notation, we state the survival kernet error guarantee (Theorem~\ref{thm:main-result}) in the following proposition (note that this proposition is more detailed than Theorem~\ref{thm:main-result} and includes constants that we have not tried to optimize).

\begin{proposition}
\label{prop:deep-kernel-netting-survival}
Suppose that assumptions $\mathbf{A}^{\text{technical}}$, $\mathbf{A}^{\text{intrinsic}}$, and $\mathbf{A}^{\text{survival}}$ hold, and that for training a survival kernet, we choose $\varepsilon=\beta\tau$ when constructing the $\varepsilon$-net $Q$ using $U_{1:n}$, where $\beta\in(0,1)$ and $\tau>0$ are user-specified parameters. Let $$\Psi=\min\left\{N\Big(\frac{(1-\beta)\tau}{2};\widetilde{\mathcal{X}}\Big), \frac{1}{C_{d'}((1-\beta)\tau)^{d'}}\right\}.$$
Then for $$n\ge\frac{2}{C_{d'}((1-\beta)\tau)^{d'}}\max\left\{\frac{K(0)}{K(\tau)\theta},\Big(\frac{9K^{2}(0)}{144K^{2}(\tau)}(2+\sqrt{2})\Big)^{1/3}\right\},$$
we have
\begin{align*}
 & \mathbb{E}_{Y_{1:n},\eventInd_{1:n},U_{1:n},U}\Big[\frac{\int_{0}^{t_{\text{horizon}}}(\widehat{S}_{\mathsf{T}|\mathsf{U}}(t|U)-S_{\mathsf{T}|\mathsf{U}}(t|U))^{2}dt}{t_{\text{horizon}}}\Big]\\
 & \;\,\le\frac{1}{n}\cdot\frac{346K^{4}(0)\log(n^{3}\frac{144K^{2}(\tau)}{9K^{2}(0)})}{\theta^{4}K^{4}(\tau)}\Psi\\
 & \;\,\quad+(1+\beta)^{2\alpha}\tau^{2\alpha}\bigg[\frac{t_{\text{horizon}}^{2}}{\theta^{2}}\Big(\lambda_{\mathsf{T}}^2+\frac{3\lambda_{\mathsf{T}}f_{\mathsf{T}}^{*}\lambda_{\mathsf{C}}t_{\text{horizon}}}{4}
 +\frac{3(f_{\mathsf{T}}^{*})^{2}\lambda_{\mathsf{C}}^{2}t_{\text{horizon}}^{2}}{20}\Big)\\
 & \;\,\quad\qquad\qquad\qquad\quad\;+\frac{12K^{2}(0)}{K^{2}(\tau)\theta^{4}}(\lambda_{\mathsf{T}}+\lambda_{\mathsf{C}})^{2}t_{\text{horizon}}^{2}\bigg],
\end{align*}
where
\begin{equation*}
f_{\mathsf{T}}^*:=\max_{s\in[0,t_{\text{horizon}}],u\in\mathcal{U}}f_{\mathsf{T}|\mathsf{U}}(s|u).
\end{equation*}
\end{proposition}
In our theoretical analysis, we assume that the $\varepsilon$-net~$Q$ is a random variable that depends on~$U_{1:n}$, where if we condition on~$U_{1:n}$, then~$Q$ becomes deterministic. Note that as stated previously, we treat~$U_{1:n}$ as a sequence (so ordering matters) whereas~$Q$ is treated as a set. Thus, an algorithm that constructs~$Q$ from~$U_{1:n}$ need not produce the exact same~$Q$ if the ordering of~$U_{1:n}$ is shuffled. If a randomized algorithm is used to construct~$Q$, we assume that it has a random seed that can be fixed, making the algorithm behave deterministically given~$U_{1:n}$.

A key lemma used in proving Proposition \ref{prop:deep-kernel-netting-survival} is a lower bound on the number of neighbors of a test embedding vector~$u$.
\begin{lemma}
\label{lem:lower-bound-on-good-neighbors-deep-kernel-netting-regression}
\emph{(Number of training points contributing to prediction -- lower bound)}
Let $\beta\in(0,1)$. Suppose that we train a survival kernet with threshold distance $\tau>0$ and compression parameter $\varepsilon=\beta\tau$, i.e., training embedding vectors $U_{1:n}$ are used to construct a $\beta\tau$-net $Q$. For a fixed choice of embedding vector $u\in\mathcal{U}$, let $U_{1:n}(u;\beta,\tau)$ denote the multiset of training data that land within the closed ball $B(u,(1-\beta)\tau)$. We have the inclusion relationship
\begin{equation}
U_{1:n}(u;\beta,\tau)
\subseteq \bigcup_{j\in\mathcal{N}_Q(u)} \underbrace{[U_j]}_{\text{multiset with single element}}.
\label{eq:epsilon-net-key-subset-relationship}
\end{equation}
In particular, taking the cardinality of both sides, we have the inequality
\begin{equation}
\Xi_u:=|U_{1:n}(u;\beta,\tau)| \le |\mathcal{N}_Q(u)|.
\label{eq:Xi}
\end{equation}
\end{lemma}
\begin{proof}
Let $u'\in U_{1:n}(u;\beta,\tau)$. Since $U_{1:n}(u;\beta,\tau)=U_{1:n}\cap B(u,(1-\beta)\tau)$, we have $\|u'-u\|_2\le(1-\beta)\tau$. Next, let $q\in Q$ be the exemplar point that $u'$ is assigned to by step~4 of the survival kernet training procedure. By how step~4 works, we have $\|u'-q\|_2\le\beta\tau$. We can then bound the distance between $q$ and $u$ using the triangle inequality:
\[
\|q-u\|_{2}\le\underbrace{\|q-u'\|_{2}}_{\le\beta\tau}+\underbrace{\|u'-u\|_{2}}_{\le(1-\beta)\tau}\le\tau.
\]
This means that $u'$ is a training embedding vector that is assigned to an exemplar point $q$ that satisfies $\|u-q\|_{2}\le\tau$, i.e., $u'$ corresponds to a training point in $\mathcal{N}_Q(u)$. This establishes the subset relationship~\eqref{eq:epsilon-net-key-subset-relationship}.
\end{proof}
Throughout our theoretical analysis, we assume that $Q$ is constructed to be an $\varepsilon$-net of $U_{1:n}$ with $\varepsilon=\beta\tau$ just as in Lemma~\ref{lem:lower-bound-on-good-neighbors-deep-kernel-netting-regression}, i.e., $\beta\in(0,1)$ and $\tau>0$ is the threshold distance for the truncated kernel.

We also regularly use the variable~$\Xi_u$ defined in equation~\eqref{eq:Xi}. Importantly, since $U_{1:n}$ are sampled i.i.d.~from $\mathbb{P}_{\mathsf{U}}$, then $\Xi_u$ is binomially distributed with parameters $n$ and $\mathbb{P}_{\mathsf{U}}(B(u,(1-\beta)\tau))$. We can ensure that test embedding vector $u$ has enough neighbors with high probability by asking that $\Xi_u$ be sufficiently large (for which Lemma~\ref{lem:lower-bound-on-good-neighbors-deep-kernel-netting-regression} guarantees that $|\mathcal{N}_Q(u)|\ge\Xi_u$). Specifically, define the good event
\begin{equation}
\mathcal{E}_{\beta}(u):=\Big\{\Xi_u>\frac{1}{2}n\mathbb{P}_{\mathsf{U}}(B(u,(1-\beta)\tau))\Big\}.
\label{eq:event-good-neighbors}
\end{equation}
We denote the complement of this event as $\mathcal{E}_{\beta}^c(u)$.
\begin{lemma}
\label{lem:prob-event-good-neighbors}
\emph{(Sufficiently many number of neighbors)}
Suppose that Assumption $\mathbf{A}^{\text{technical}}$ holds and $U_{1:n}$ are sampled i.i.d.~from~$\mathbb{P}_{\mathsf{U}}$. Let $\beta\in(0,1)$, $\tau>0$, and $u\in\mathcal{U}$. Then the probability that good event $\mathcal{E}_\beta(u)$ does not happen is bounded above as
\[
\mathbb{P}(\mathcal{E}_{\beta}^c(u)) \le \exp\Big(-\frac{n\mathbb{P}_{\mathsf{U}}(B(u,(1-\beta)\tau))}{8}\Big) \le \frac{8}{n\mathbb{P}_{\mathsf{U}}(B(u,(1-\beta)\tau))}.
\]
\end{lemma}
\begin{proof}
The claim readily follows from a multiplicative \mbox{Chernoff} bound of the binomial distribution and noting that $\exp(-z)\le1/z$ for all $z>0$.
\end{proof}
Another argument used in the proof is that if the \mbox{$j$-th} training point is a neighbor of embedding vector $u$---i.e., $j\in\mathcal{N}_Q(u)$---then their embedding vectors cannot be too far apart, i.e., $\|U_j-u\|_{2}$ is not too large.

\begin{lemma}
\label{lem:neighbor-embedding}
\emph{(Neighbor embedding vectors are close)}
Suppose that we train a survival kernet as in Lemma~\ref{lem:lower-bound-on-good-neighbors-deep-kernel-netting-regression}. Let $u\in\mathcal{U}$. If $j\in\mathcal{N}_Q(u)$, then
\[
\|U_{j}-u\|_{2}\le(1+\beta)\tau.
\]
\end{lemma}
\begin{proof}
Let $j\in\mathcal{N}_Q(u)$. This means that the exemplar that the $j$-th training point is assigned to, namely $q_{j}\in Q$, must be within distance $\tau$ of $u$. Thus,
\[
\|U_{j}-u\|_{2}\le\underbrace{\|U_{j}-q_{j}\|_{2}}_{\le\beta\tau}+\underbrace{\|q_{j}-u\|_{2}}_{\le\tau}\le(1+\beta)\tau,
\]
where we have used the fact that $Q$ is a $\beta\tau$-net, and $\mathcal{N}_Q(u)$ is defined to consider all exemplars within distance $\tau$ of~$u$.
\end{proof}
We remark that proving the special case when $\varepsilon = 0$ is easier precisely because Lemmas~\ref{lem:lower-bound-on-good-neighbors-deep-kernel-netting-regression}, \ref{lem:prob-event-good-neighbors}, and \ref{lem:neighbor-embedding} simplify --- we no longer need $\varepsilon$-net arguments. Instead of Lemma~\ref{lem:lower-bound-on-good-neighbors-deep-kernel-netting-regression}, we directly consider training points within distance $\tau$ of fixed embedding vector~$u$. Instead of Lemma~\ref{lem:prob-event-good-neighbors}, we directly consider points landing in the ball $B(u,\tau)$. In Lemma~\ref{lem:neighbor-embedding}, there is no multiplicative approximation error $(1+\beta)$ to worry about. In short, we get the same conclusions as these lemmas except plugging in $\beta=0$.

We present proof of Proposition \ref{prop:deep-kernel-netting-survival} next.

\subsubsection{Proof of Proposition~\ref{prop:deep-kernel-netting-survival}}

We focus on proving ``pointwise'' bounds first, which in this case means that we fix both the test embedding vector $u\in\mathcal{U}$ and time $t\in[0,t_{\text{horizon}}]$. (At the end of the proof, we take an expectation over the test embedding vector and integrate over time.)

First off, we have the trivial bound ${|\widehat{S}_{\mathsf{T}|\mathsf{U}}(t|u)-S_{\mathsf{T}|\mathsf{U}}(t|u)|}\le1$. Under some good events holding, we can obtain a nontrivial error bound. The first good event is the same as for regression $\mathcal{E}_{\beta}(u)=\{\Xi_u>\frac12 n\mathbb{P}_{\mathsf{U}}(B(u,(1-\beta)\tau))\}$. We introduce two additional good events. We define the second good event as
\begin{equation}
\mathcal{E}_{\text{horizon}}(u):=\Big\{\underbrace{\sum_{j\in\mathcal{N}_Q(u)}\ind\{Y_{j}>t_{\text{horizon}}\}}_{d^{+}(u,t_{\text{horizon}}):=}>\frac{|\mathcal{N}_Q(u)|\theta}{2}\Big\}.
\label{eq:event-horizon}
\end{equation}
Note that for any $s\ge0$, $d^{+}(u,s)$ is the number of neighbors of $u$ that have observed times $Y_{j}$'s exceeding $s$. When event $\mathcal{E}_{\text{horizon}}(u)$ holds, we have $t_{\text{horizon}}\le Y_{\max}(u)$. We define the third and final good event later when we relate estimating the tail function $S_{\mathsf{T}|\mathsf{U}}(t|u)$ (that is for survival times $T_{j}$'s that we do not always get to observe in the training data), to estimating the tail function $S_{\mathsf{Y}|\mathsf{U}}(t|u)$ (that is of observed times $Y_{j}$'s that we see all of in the training data).

Importantly, when both $\mathcal{E}_{\beta}(u)$ and $\mathcal{E}_{\text{horizon}}(u)$ hold, then for every time $t\in[0,t_{\text{horizon}}]$, the hazard estimator $\widehat{g}_{Q}(t|u)$ takes on a nontrivial value given by the first case of equation~\eqref{eq:hazard-estimate-appendix-form}, rather than just being 0. As we show later (Lemma~\ref{lem:prob-survival-good-events}), good events $\mathcal{E}_{\beta}(u)$ and $\mathcal{E}_{\text{horizon}}(u)$ simultaneously hold with high probability as $n$ grows large. These good events holding enable us to write the survival kernet conditional survival function estimate $\widehat{S}_{\mathsf{T}|\mathsf{U}}(t|u)$ in a form that will be easier to work with.

\begin{lemma}
\label{lem:survival-estimator-convenient-form}
\emph{(Convenient form of the kernel survival estimator)}
Suppose that Assumptions $\mathbf{A}^{\text{technical}}$ and $\mathbf{A}^{\text{survival}}$ hold, and that the kernel function $K$ is of the form described in equation~\eqref{eq:truncated-kernel}. Consider training a survival kernet with parameter $\varepsilon=\beta\tau$, where $\tau>0$ is the truncated kernel threshold distance, and $\beta\in(0,1)$ is a user-specified parameter. Fix $u\in\mathcal{U}$, and condition on events $\mathcal{E}_{\beta}(u)$ and $\mathcal{E}_{\text{horizon}}(u)$ occurring. Recall that $q_j\in Q$ is the exemplar that the $j$-th training point is assigned to during survival kernet training. With probability 1, for all $t\in[0,t_{\text{horizon}}]$,
\begin{equation}
 \widehat{S}_{\mathsf{T}|\mathsf{U}}(t|u)=\prod_{i\in\mathcal{N}_Q(u)}\Big(\frac{d_{K}^{+}(u,Y_{i})}{K_{i}+d_{K}^{+}(u,Y_{i})}\Big)^{\eventInd_{i}\ind\{Y_{i}\le t\}},
\label{eq:survival-estimator-convenient-form}
\end{equation}
where $K_{i}:=K(\|u-q_{i}\|_{2})\ind\{\|u-q_i\|_2\le\tau\}$, and for any $s\ge0$,
\begin{align*}
d_{K}^{+}(u,s) & := \sum_{j\in\mathcal{N}_Q(u)}K_j\ind\{Y_{j}>s\}.
\end{align*}
Moreover, since $\mathcal{E}_{\text{horizon}}(u)=\{d^{+}(u,t_{\text{horizon}})>\frac{|\mathcal{N}_Q(u)|\theta}{2}\}$ holds, we are guaranteed that $\widehat{S}_{\mathsf{T}|\mathsf{U}}(t|u)>0$ for $t\in[0,t_{\text{horizon}}]$.
\end{lemma}
Note that $d_{K}^{+}(u,s)$ is not the same as $d^{+}(u,s)$ defined earlier (in equation~\eqref{eq:event-horizon}, which does not use kernel weights). These quantities are related via the bound $d_{K}^{+}(u,s)\ge K(\tau)d^{+}(u,s)$.

\begin{proof}
Fix $u\in\mathcal{U}$ and $t\in[0,t_{\text{horizon}}]$. First off, due to events $\mathcal{E}_{\beta}(u)$ and $\mathcal{E}_{\text{horizon}}(u)$ holding, $\widehat{g}_{Q}(t|u)$ is nonzero. Next, Assumption $\mathbf{A}^{\text{survival}}$ implies that ties in $Y_{j}$'s happen with probability 0. Thus, with probability 1, there are no ties in $Y_{j}$'s, so
\begin{align}
\widehat{S}_{\mathsf{T}|\mathsf{U}}(t|u) & =\prod_{t'\in\mathcal{T}}(1-\widehat{g}_{Q}(t'|u))^{\ind\{t'\le t\}} =\prod_{i=1}^{n}(1-\widehat{g}_{Q}(Y_{i}|u))^{\eventInd_{i}\ind\{Y_{i}\le t\}}.
\label{eq:survival-estimator-convenient-form-helper1}
\end{align}
Next, since there are no ties in $Y_{j}$'s, the expression for $\widehat{g}_{Q}(Y_{i}|u)$ simplifies as
\begin{align*}
\widehat{g}_{Q}(Y_{i}|u)
&=\frac{K_i\eventInd_i}{\sum_{j\in\mathcal{N}_Q(u)}K_j\ind\{Y_{j}\ge Y_i\}}
=\frac{K_i\eventInd_i}{K_i + \sum_{j\in\mathcal{N}_Q(u)}K_j\ind\{Y_{j}> Y_i\}}
=\frac{K_i\eventInd_i}{K_i + d_K^+(u,Y_i)}.
\end{align*}
Hence,
\begin{align*}
 (1-\widehat{g}_{Q}(Y_{i}|u))^{\eventInd_{i}\ind\{Y_{i}\le t\}}
 & =\Big(1-\frac{K_{i}\eventInd_{i}}{K_{i}+d_{K}^{+}(u,Y_{i})}\Big)^{\eventInd_{i}\ind\{Y_{i}\le t\}}\\
 & =\Big(1-\frac{K_{i}}{K_{i}+d_{K}^{+}(u,Y_{i})}\Big)^{\eventInd_{i}\ind\{Y_{i}\le t\}}\\
 & =\Big(\frac{d_{K}^{+}(u,Y_{i})}{K_{i}+d_{K}^{+}(u,Y_{i})}\Big)^{\eventInd_{i}\ind\{Y_{i}\le t\}}.
\end{align*}
Combining this with equation~\eqref{eq:survival-estimator-convenient-form-helper1}, we get
\[
\widehat{S}_{\mathsf{T}|\mathsf{U}}(t|u)=\prod_{i=1}^{n}\Big(\frac{d_{K}^{+}(u,Y_{i})}{K_{i}+d_{K}^{+}(u,Y_{i})}\Big)^{\eventInd_{i}\ind\{Y_{i}\le t\}}.
\]
In the RHS, if $K_{i}=0,$ then the $i$-th factor is 1 and thus does not affect the overall product. Note that $K_{i}>0$ precisely for $i\in\mathcal{N}_Q(u)$. Thus, it suffices to only take the product over $i\in\mathcal{N}_Q(u)$, which establishes equality \eqref{eq:survival-estimator-convenient-form}.
\end{proof}
When we can write $\widehat{S}_{\mathsf{T}|\mathsf{U}}(t|u)$ using the expression~\eqref{eq:survival-estimator-convenient-form}, then using the Taylor expansion ${\log(1+z)}=\sum_{\ell=1}^{\infty}\frac{1}{\ell}(\frac{z}{z+1})^{\ell}$ that is valid for any $z>0$, we get
\begin{align*}
 \log\widehat{S}_{\mathsf{T}|\mathsf{U}}(t|u)
 & =\sum_{i\in\mathcal{N}_Q(u)}\!\eventInd_{i}\ind\{Y_{i}\le t\}\log\Big(\frac{d_{K}^{+}(u,Y_{i})}{K_{i}+d_{K}^{+}(u,Y_{i})}\Big) \\
 & =-\sum_{i\in\mathcal{N}_Q(u)}\!\eventInd_{i}\ind\{Y_{i}\le t\}\log\Big(1+\frac{K_{i}}{d_{K}^{+}(u,Y_{i})}\Big) \\
 & =-\sum_{i\in\mathcal{N}_Q(u)}\!\frac{\eventInd_{i}\ind\{Y_{i}\le t\}K_{i}}{K_{i}+d_{K}^{+}(u,Y_{i})}
 -\sum_{i\in\mathcal{N}_Q(u)}\!\eventInd_{i}\ind\{Y_{i}\le t\}\sum_{\ell=2}^{\infty}\frac{1}{\ell[1+(d_{K}^{+}(u,Y_{i})/K_{i})]^{\ell}} \\
 & =W_{1}(t|u)+W_{2}(t|u)+W_{3}(t|u),
\end{align*}
where
\begin{align*}
 W_{1}(t|u)
 &:=\sum_{i\in\mathcal{N}_Q(u)}\Big(\frac{K_{i}}{\sum_{j\in\mathcal{N}_Q(u)}K_{j}}\Big)\Big(-\frac{\eventInd_{i}\ind\{Y_{i}\le t\}}{S_{\mathsf{Y}|\mathsf{U}}(Y_{i}|u)}\Big),\\
 W_{2}(t|u)
 &:=-\sum_{i\in\mathcal{N}_Q(u)}\frac{K_{i}\eventInd_{i}\ind\{Y_{i}\le t\}\big[\frac{\sum_{\ell\in\mathcal{N}_Q(u)}K_{\ell}}{K_{i}+d_{K}^{+}(u,Y_{i})}-\frac{1}{S_{\mathsf{Y}|\mathsf{U}}(Y_{i}|u)}\big]}{\sum_{j\in\mathcal{N}_Q(u)}K_{j}},\\
 W_{3}(t|u)
 &:=-\sum_{i\in\mathcal{N}_Q(u)}\eventInd_{i}\ind\{Y_{i}\le t\}\sum_{\ell=2}^{\infty}\frac{1}{\ell[1+(d_{K}^{+}(u,Y_{i})/K_{i})]^{\ell}}.
\end{align*}
The high-level idea is to show that $W_{1}(t|u)\rightarrow\log S_{\mathsf{T}|\mathsf{U}}(t|u)$, $W_{2}(t|u)\rightarrow0$, and $W_{3}(t|u)\rightarrow0$.

As previously pointed out by~\citet{chen2019nearest}, the term $W_{1}(t|u)$ could be thought of as a kernel regression estimate that averages over hypothetical labels of the form $-\frac{\eventInd_{i}\ind\{Y_{i}\le t\}}{S_{\mathsf{Y}|\mathsf{U}}(Y_{i}|u)}$ across~$i$, which knows the true tail function $S_{\mathsf{Y}|\mathsf{U}}$ (of observed times and not survival times). Of course this true tail function is not actually known.

The term $W_{2}(t|u)$ focuses on how accurately we can estimate the tail function $S_{\mathsf{Y}|\mathsf{U}}$. Note that estimating the CDF of the observed times $Y_j$'s is simpler than estimating the CDF of the survival times $T_j$'s since observed times are known for all training data whereas survival times are only known for uncensored training data (although the censored data provide information on survival times as well since censored times are lower bounds on survival times).

Lastly, the term $W_{3}(t|u)$ vanishes when the number of neighbors of $u$ with observed times $Y_{j}$'s exceeding $t_{\text{horizon}}$ is sufficiently large.

For now, we assume that good events $\mathcal{E}_{\beta}(u)$ and $\mathcal{E}_{\text{horizon}}(u)$ hold, along with Assumptions $\mathbf{A}^{\text{technical}}$ and $\mathbf{A}^{\text{survival}}$ before introducing the third and final good event that comes into play in showing that $W_{2}(t|u)\rightarrow0$. Lemma~\ref{lem:survival-estimator-convenient-form} guarantees that under these conditions, $\widehat{S}_{\mathsf{T}|\mathsf{U}}(t|u)\ne0$ (recall that $t\in[0,t_{\text{horizon}}]$). Moreover, using Assumption $\mathbf{A}^{\text{survival}}$(a) and since $S_{\mathsf{T}|\mathsf{U}}(\cdot|u)$ monotonically decreases, we have $S_{\mathsf{T}|\mathsf{U}}(t|u)\ge S_{\mathsf{T}|\mathsf{U}}(t_{\text{horizon}}|u)\ge\theta>0$. Using the fact that for any $a,b\in(0,1]$, we have $|a-b|\le|\log a-\log b|$, we get
\begin{align*}
 & \mathbb{E}_{Y_{1:n},\eventInd_{1:n}|U_{1:n}}\big[(\widehat{S}_{\mathsf{T}|\mathsf{U}}(t|u)-S_{\mathsf{T}|\mathsf{U}}(t|u))^{2}\big]\\
 & \quad\le\mathbb{E}_{Y_{1:n},\eventInd_{1:n}|U_{1:n}}\big[(\log\widehat{S}_{\mathsf{T}|\mathsf{U}}(t|u)-\log S_{\mathsf{T}|\mathsf{U}}(t|u))^{2}\big]\\
 & \quad=\mathbb{E}_{Y_{1:n},\eventInd_{1:n}|U_{1:n}}\big[(W_{1}(t|u)-\log S_{\mathsf{T}|\mathsf{U}}(t|u)
 +W_{2}(t|u)+W_{3}(t|u))^{2}\big].
\end{align*}
Next, recall that for any $a_{1},a_{2},\dots,a_{\ell}\in\mathbb{R}$, a consequence of Jensen's inequality is that
$
(\sum_{i=1}^{\ell}a_{i})^{2}\le\ell\sum_{i=1}^{\ell}a_{i}^{2}
$. By applying this inequality to the RHS above, we get
\begin{align}
 & \mathbb{E}_{Y_{1:n},\eventInd_{1:n}|U_{1:n}}\big[(\widehat{S}_{\mathsf{T}|\mathsf{U}}(t|u)-S_{\mathsf{T}|\mathsf{U}}(t|u))^{2}\big]\nonumber\\
 & \quad\le\mathbb{E}_{Y_{1:n},\eventInd_{1:n}|U_{1:n}}\big[(W_{1}(t|u)-\log S_{\mathsf{T}|\mathsf{U}}(t|u)
 +W_{2}(t|u)+W_{3}(t|u))^{2}\big] \nonumber\\
 & \quad\le3\underbrace{\mathbb{E}_{Y_{1:n},\eventInd_{1:n}|U_{1:n}}[(W_{1}(t|u)-\log S_{\mathsf{T}|\mathsf{U}}(t|u))^{2}]}_{\text{term I}}\nonumber
 +3\underbrace{\mathbb{E}_{Y_{1:n},\eventInd_{1:n}|U_{1:n}}[(W_{2}(t|u))^{2}]}_{\text{term II}}\nonumber\\
 & \quad\quad+3\underbrace{\mathbb{E}_{Y_{1:n},\eventInd_{1:n}|U_{1:n}}[(W_{3}(t|u))^{2}]}_{\text{term III}}.
 \label{eq:W3-main-inequality}
\end{align}
We proceed to bound each of the RHS expectations.

\paragraph{Bounding term I} As stated previously, $W_{1}(t|u)$ is like a kernel regression estimate. By a standard bias-variance decomposition,
\begin{align*}
 & \mathbb{E}_{Y_{1:n},\eventInd_{1:n}|U_{1:n}}[(W_{1}(t|u)-\log S_{\mathsf{T}|\mathsf{U}}(t|u))^{2}]\\
 & \quad =\underbrace{\mathbb{E}_{Y_{1:n},\eventInd_{1:n}|U_{1:n}}[(W_{1}(t|u)-\mathbb{E}_{Y_{1:n},\eventInd_{1:n}|U_{1:n}}[W_1(t|u)])^2]}_{\text{variance}}\\
 & \quad\quad +\underbrace{(\mathbb{E}_{Y_{1:n},\eventInd_{1:n}|U_{1:n}}[W_1(t|u)]-\log S_{\mathsf{T}|\mathsf{U}}(t|u))^{2}}_{\text{bias}}.
\end{align*}
We bound the right-hand side.
\begin{lemma}
\label{lem:W1-bound}
\emph{(Bound on term I)}
Under the same setting as Lemma~\ref{lem:survival-estimator-convenient-form}, except where we only condition on good event $\mathcal{E}_{\beta}(u)$ (i.e., good event $\mathcal{E}_{\text{horizon}}(u)$ can optionally hold),
\[
 \mathbb{E}_{Y_{1:n},\eventInd_{1:n}|U_{1:n}}[(W_{1}(t|u)-\log S_{\mathsf{T}|\mathsf{U}}(t|u))^{2}]
 \le\!\!\!\!\underbrace{\frac{K(0)}{4\theta^{2}K(\tau)\Xi_{u}}}_{\text{variance term bound}}\!\!\!\!+\underbrace{\frac{(1+\beta)^{2\alpha}\tau^{2\alpha}}{\theta^{2}}\Big(\lambda_{\mathsf{T}}t+\frac{f_{\mathsf{T}}^{*}\lambda_{\mathsf{C}}t^{2}}{2}\Big)^{2}}_{\text{bias term bound}}.
\]
\end{lemma}
\begin{proof}
(Variance term bound) The variance term equals
\begin{align*}
 & \mathbb{E}_{Y_{1:n},\eventInd_{1:n}|U_{1:n}}[(W_{1}(t|u)-\mathbb{E}_{Y_{1:n},\eventInd_{1:n}|U_{1:n}}[W_1(t|u)])^2]\\
 & =\!\!\!\sum_{i\in\mathcal{N}_Q(u)}\!\!\!\Big(\frac{K_{i}}{\sum_{j\in\mathcal{N}_Q(u)}K_{j}}\Big)^{2}
 \underbrace{\mathbb{E}_{Y_{i},\eventInd_{i}|U_{i}}\Big[\Big(-\frac{\eventInd_{i}\ind\{Y_{i}\le t\}}{S_{\mathsf{Y}|\mathsf{U}}(Y_{i}|u)}
 -\mathbb{E}_{Y_{1:n},\eventInd_{1:n}|U_{1:n}}\Big[-\frac{\eventInd_{i}\ind\{Y_{i}\le t\}}{S_{\mathsf{Y}|\mathsf{U}}(Y_{i}|u)}\Big]\Big)^{2}\Big]}_{\text{variance of }-\frac{\eventInd_{i}\ind\{Y_{i}\le t\}}{S_{\mathsf{Y}|\mathsf{U}}(Y_{i}|u)}\text{ conditioned on }U_{i}}.
\end{align*}
Note that this sum is not vacuous since under event $\mathcal{E}_{\beta}(u)$ and using Lemma~\ref{lem:lower-bound-on-good-neighbors-deep-kernel-netting-regression}, $|\mathcal{N}_Q(u)|>0$. Next, note that $-\frac{\eventInd_{i}\ind\{Y_{i}\le t\}}{S_{\mathsf{Y}|\mathsf{U}}(Y_{i}|u)}$ is a bounded random variable. It is nonzero only when $Y_{i}\le t$, for which the denominator could be as large as 1 and as small as $S_{\mathsf{Y}|\mathsf{U}}(Y_{i}|u)\ge S_{\mathsf{Y}|\mathsf{U}}(t|u)\ge S_{\mathsf{Y}|\mathsf{U}}(t_{\text{horizon}}|u)\ge\theta$ using the fact that $S_{\mathsf{Y}|\mathsf{U}}(\cdot|u)$ monotonically decreases and using Assumption $\mathbf{A}^{\text{survival}}$(a). Consequently, $-\frac{\eventInd_{i}\ind\{Y_{i}\le t\}}{S_{\mathsf{Y}|\mathsf{U}}(Y_{i}|u)}\in[-\frac{1}{\theta},0]$, which by a standard result on bounded random variables implies that the variance of this random variable is at most $\frac{1}{4\theta^{2}}$. Hence,
\[
 \mathbb{E}_{Y_{1:n},\eventInd_{1:n}|U_{1:n}}[(W_{1}(t|u)-\mathbb{E}_{Y_{1:n},\eventInd_{1:n}|U_{1:n}}[W_1(t|u)])^{2}]\\
 \le\sum_{i\in\mathcal{N}_Q(u)}\Big(\frac{K_{i}}{\sum_{j\in\mathcal{N}_Q(u)}K_{j}}\Big)^{2}\frac{1}{4\theta^{2}}.
\]
Then by H\"{o}lder's inequality and Lemma~\ref{lem:lower-bound-on-good-neighbors-deep-kernel-netting-regression},
\begin{align*}
 \sum_{i\in\mathcal{N}_Q(u)}\Big(\frac{K_{i}}{\sum_{j\in\mathcal{N}_Q(u)}K_{j}}\Big)^{2}\frac{1}{4\theta^{2}}
 & \le\frac{1}{4\theta^{2}}\Big[\max_{i\in\mathcal{N}_Q(u)}\frac{K_{i}}{\sum_{j\in\mathcal{N}_Q(u)}K_{j}}\Big]\underbrace{\sum_{i\in\mathcal{N}_Q(u)}\frac{K_{i}}{\sum_{j\in\mathcal{N}_Q(u)}K_{j}}}_{=1}\\
 & \le\frac{1}{4\theta^{2}}\max_{i\in\mathcal{N}_Q(u)}\frac{K(0)}{K(\tau)|\mathcal{N}_Q(u)|}\\
 & =\frac{K(0)}{4\theta^{2}K(\tau)|\mathcal{N}_Q(u)|}\le\frac{K(0)}{4\theta^{2}K(\tau)\Xi_{u}},
\end{align*}
which establishes the variance term bound.

(Bias term bound) First, recall from equation~\eqref{eq:hazard} that the hazard function is
$
-\frac{\partial}{\partial t}\log S_{\mathsf{T}|\mathsf{U}}(t|u)=\frac{f_{\mathsf{T}|\mathsf{U}}(s|u)}{S_{\mathsf{T}|\mathsf{U}}(s|u)}$.
Then by the fundamental theorem of calculus,
\begin{align}
 \log S_{\mathsf{T}|\mathsf{U}}(t|u)
 & =-\int_{0}^{t}\frac{f_{\mathsf{T}|\mathsf{U}}(s|u)}{S_{\mathsf{T}|\mathsf{U}}(s|u)}ds\nonumber\\
 & =-\int_{0}^{t}\frac{1}{S_{\mathsf{T}|\mathsf{U}}(s|u)S_{\mathsf{C}|\mathsf{U}}(s|u)}S_{\mathsf{C}|\mathsf{U}}(s|u)f_{\mathsf{T}|\mathsf{U}}(s|u)ds\nonumber\\
 & =-\int_{0}^{t}\frac{1}{S_{\mathsf{Y}|\mathsf{U}}(s|u)}S_{\mathsf{C}|\mathsf{U}}(s|u)f_{\mathsf{T}|\mathsf{U}}(s|u)ds,
 \label{eq:W1-bias-target}
\end{align}
where the last step uses the result that for two independent random variables $A_{1}$ and $A_{2}$, the random variable $\min\{A_{1},A_{2}\}$ has a tail function (1 minus CDF) that is the product of the tail functions of $A_{1}$ and $A_{2}$; in our setting, $Y_j=\min\{T_j,C_j\}$ where $T_j$ and $C_j$ are conditionally independent given $U_j$, so $S_{\mathsf{Y}|\mathsf{U}}(s|u)=S_{\mathsf{T}|\mathsf{U}}(s|u)S_{\mathsf{C}|\mathsf{U}}(s|u)$.

Next, we have
\begin{equation}
 \mathbb{E}_{Y_{1:n},\eventInd_{1:n}|U_{1:n}}[W_{1}(t|u)]
 =\sum_{i\in\mathcal{N}_Q(u)}\Big(\frac{K_{i}}{\sum_{j\in\mathcal{N}_Q(u)}K_{j}}\Big)\mathbb{E}_{Y_{i},\eventInd_{i}|U_{i}}\Big[-\frac{\eventInd_{i}\ind\{Y_{i}\le t\}}{S_{\mathsf{Y}|\mathsf{U}}(Y_{i}|u)}\Big],
 \label{eq:W1-bias-expected-estimate}
\end{equation}
where
\begin{align*}
 \mathbb{E}_{Y_{i},\eventInd_{i}|U_{i}}\Big[-\frac{\eventInd_{i}\ind\{Y_{i}\le t\}}{S_{\mathsf{Y}|\mathsf{U}}(Y_{i}|u)}\Big]
 & =-\int_{0}^{t}\Big[\int_{s}^{\infty}\frac{1}{S_{\mathsf{Y}|\mathsf{U}}(s|u)}d\mathbb{P}_{\mathsf{C}|\mathsf{U}}(c|U_{i})\Big]d\mathbb{P}_{\mathsf{T}|\mathsf{U}}(s|U_{i})\\
 & =-\int_{0}^{t}\frac{1}{S_{\mathsf{Y}|\mathsf{U}}(s|u)}S_{\mathsf{C}|\mathsf{U}}(s|U_{i})f_{\mathsf{T}|\mathsf{U}}(s|U_{i})ds.
\end{align*}
We take the difference of equations~\eqref{eq:W1-bias-expected-estimate} and~\eqref{eq:W1-bias-target} to get
\begin{align*}
 & \mathbb{E}_{Y_{1:n},\eventInd_{1:n}|U_{1:n}}[W_{1}(t|u)]-\log S_{\mathsf{T}|\mathsf{U}}(t|u)\\
 & \quad =\sum_{i\in\mathcal{N}_Q(u)}\Big(\frac{K_{i}}{\sum_{j\in\mathcal{N}_Q(u)}K_{j}}\Big)
 \int_{0}^{t}\frac{S_{\mathsf{C}|\mathsf{U}}(s|u)f_{\mathsf{T}|\mathsf{U}}(s|u)-S_{\mathsf{C}|\mathsf{U}}(s|U_{i})f_{\mathsf{T}|\mathsf{U}}(s|U_{i})}{S_{\mathsf{Y}|\mathsf{U}}(s|u)}ds.
\end{align*}
Note that Assumption $\mathbf{A}^{\text{survival}}$(b) implies that $S_{\mathsf{C}|\mathsf{U}}(s|\cdot)f_{\mathsf{T}|\mathsf{U}}(s|\cdot)$ is H\"{o}lder continuous with parameters $(\lambda_{\mathsf{T}}+f_{\mathsf{T}}^*\lambda_{\mathsf{C}}s)$ and $\alpha$, where $f_{\mathsf{T}}^*=\max_{s\in[0,t_{\text{horizon}}],u\in\mathcal{U}}f_{\mathsf{T}|\mathsf{U}}(s|u)$; this maximum exists by compactness of $[0,t_{\text{horizon}}]$ and $\mathcal{U}$. Then by H\"{o}lder's inequality, Assumption $\mathbf{A}^{\text{survival}}$, and Lemma~\ref{lem:neighbor-embedding},
\begin{align*}
 &|\mathbb{E}_{Y_{1:n},\eventInd_{1:n}|U_{1:n}}[W_{1}(t|u)]-\log S_{\mathsf{T}|\mathsf{U}}(t|u)|\\
 & \quad \le\max_{i\in\mathcal{N}_Q(u)}\Big|\int_{0}^{t}\frac{1}{\underbrace{S_{\mathsf{Y}|\mathsf{U}}(s|u)}_{\ge S_{\mathsf{Y}|\mathsf{U}}(t_{\text{horizon}}|u)\ge\theta}}[S_{\mathsf{C}|\mathsf{U}}(s|u)f_{\mathsf{T}|\mathsf{U}}(s|u)-S_{\mathsf{C}|\mathsf{U}}(s|U_{i})f_{\mathsf{T}|\mathsf{U}}(s|U_{i})]ds\Big|\\
 & \quad \le\frac{1}{\theta}\max_{i\in\mathcal{N}_Q(u)}\int_{0}^{t}|S_{\mathsf{C}|\mathsf{U}}(s|u)f_{\mathsf{T}|\mathsf{U}}(s|u)-S_{\mathsf{C}|\mathsf{U}}(s|U_{i})f_{\mathsf{T}|\mathsf{U}}(s|U_{i})|ds\\
 & \quad \le\frac{1}{\theta}\max_{i\in\mathcal{N}_Q(u)}\int_{0}^{t}(\lambda_{\mathsf{T}}+f_{\mathsf{T}}^{*}\lambda_{\mathsf{C}}s)\|u-U_{i}\|_{2}^{\alpha}ds\\
 & \quad \le\frac{1}{\theta}\max_{i\in\mathcal{N}_Q(u)}\int_{0}^{t}(\lambda_{\mathsf{T}}+f_{\mathsf{T}}^{*}\lambda_{\mathsf{C}}s)((1+\beta)\tau)^{\alpha}ds\\
 & \quad =\frac{(1+\beta)^{\alpha}\tau^{\alpha}}{\theta}\int_{0}^{t}(\lambda_{\mathsf{T}}+f_{\mathsf{T}}^{*}\lambda_{\mathsf{C}}s)ds\\
 & \quad =\frac{(1+\beta)^{\alpha}\tau^{\alpha}}{\theta}\Big(\lambda_{\mathsf{T}}t+\frac{f_{\mathsf{T}}^{*}\lambda_{\mathsf{C}}t^{2}}{2}\Big).
\end{align*}
Squaring both sides yields the bias term bound.
\end{proof}

\paragraph{Bounding term II}
The term $W_{2}(t|u)$ is related to a CDF estimation problem. Specifically, define the following estimate of $S_{\mathsf{Y}|\mathsf{U}}(s|u)$:
\[
\widehat{S}_{\mathsf{Y}|\mathsf{U}}(s|u)
:=\frac{d_{K}^{+}(u,s)}{\sum_{\ell\in\mathcal{N}_Q(u)}K_{\ell}} \\
=\sum_{j\in\mathcal{N}_Q(u)}\frac{K_{j}}{\sum_{\ell\in\mathcal{N}_Q(u)}K_{\ell}}\ind\{Y_{j}>s\}.
\]
Then $1-\widehat{S}_{\mathsf{Y}|\mathsf{U}}(s|u)$ is a weighted empirical distribution that estimates $1-S_{\mathsf{Y}|\mathsf{U}}(s|u)$. We introduce the third and final good event
\begin{equation}
 \mathcal{E}_{\text{CDF}}(u)
 :=\Big\{\underbrace{\sup_{s\ge0}|\widehat{S}_{\mathsf{Y}|\mathsf{U}}(s|u)-\mathbb{E}_{Y_{1:n}|U_{1:n}}[\widehat{S}_{\mathsf{Y}|\mathsf{U}}(s|u)]|}_{\spadesuit:=}\le\varepsilon_{\text{CDF}}\Big\},
 \label{eq:event-good-CDF}
\end{equation}
where
\[
\varepsilon_{\text{CDF}}=\sqrt{\frac{9K^{2}(0)}{4K^{2}(\tau)|\mathcal{N}_Q(u)|}\log\Big(|\mathcal{N}_Q(u)|^{3}\frac{144K^{2}(\tau)}{9K^{2}(0)}\Big)}.
\]
We show later (Lemma~\ref{lem:prob-survival-good-events}) that this event holds with high probability for large $n$. Note that event $\mathcal{E}_{\text{CDF}}(u)$ is like a variance term, saying that $\widehat{S}_{\mathsf{Y}|\mathsf{U}}(s|u)$ is close to its expectation. We also define a bias term
\begin{equation}
\heartsuit:=\sup_{s\in[0,t]}|\widehat{S}_{\mathsf{Y}|\mathsf{U}}(s|u)-\mathbb{E}_{Y_{1:n}|U_{1:n}}[\widehat{S}_{\mathsf{Y}|\mathsf{U}}(s|u)]|.
\label{eq:CDF-bias}
\end{equation}
Then the following lemma relates $W_{2}(t|u)$ to the variance and bias terms.

\begin{lemma}
\label{lem:W2-bound}
\emph{(Bound on term II)}
Under the same setting as Lemma~\ref{lem:survival-estimator-convenient-form}, except where we now condition on all three good events $\mathcal{E}_{\beta}(u)$, $\mathcal{E}_{\text{horizon}}(u)$, and $\mathcal{E}_{\text{CDF}}(u)$ holding,
\begin{align*}
 (W_{2}(t|u))^{2}
 & \le\frac{12K^{2}(0)}{K^{2}(\tau)\theta^{4}\Xi_{u}^{2}}+\frac{12K^{2}(0)}{K^{2}(\tau)\theta^{4}}\Big(\spadesuit^{2}+\heartsuit^{2}\Big).
\end{align*}
where
\[
\spadesuit^{2} \le\varepsilon_{\text{CDF}}^{2},\quad
\heartsuit^{2} \le(\lambda_{\mathsf{T}}+\lambda_{\mathsf{C}})^{2}t^{2}(1+\beta)^{2\alpha}\tau^{2\alpha}.
\]
Thus,
\[
 \mathbb{E}_{Y_{1:n},\eventInd_{1:n}|U_{1:n}}[(W_{2}(t|u))^{2}] \le\frac{12K^{2}(0)}{K^{2}(\tau)\theta^{4}\Xi_{u}^{2}}
 + \frac{12K^{2}(0)}{K^{2}(\tau)\theta^{4}}[\varepsilon_{\text{CDF}}^{2}+(\lambda_{\mathsf{T}}+\lambda_{\mathsf{C}})^{2}t^{2}(1+\beta)^{2\alpha}\tau^{2\alpha}].
\]
\end{lemma}
\begin{proof}
We use H\"{o}lder's inequality and a bit of algebra to get
\begin{align}
 |W_{2}(t|u)|
 & =\bigg|\sum_{i\in\mathcal{N}_Q(u)}\frac{K_{i}\eventInd_{i}\ind\{Y_{i}\le t\}\big[\frac{1}{S_{\mathsf{Y}|\mathsf{U}}(Y_{i}|u)}-\frac{\sum_{\ell\in\mathcal{N}_Q(u)}K_{\ell}}{K_{i}+d_{K}^{+}(u,Y_{i})}\big]}{\sum_{j\in\mathcal{N}_Q(u)}K_{j}}\bigg|\nonumber\\
 & \le\max_{i\in\mathcal{N}_Q(u)}\bigg|\eventInd_{i}\ind\{Y_{i}\le t\}\bigg[\frac{1}{S_{\mathsf{Y}|\mathsf{U}}(Y_{i}|u)}-\frac{\sum_{\ell\in\mathcal{N}_Q(u)}K_{\ell}}{K_{i}+d_{K}^{+}(u,Y_{i})}\bigg]\bigg|\nonumber\\
 & =\max_{i\in\mathcal{N}_Q(u)}\bigg|\frac{\eventInd_{i}\ind\{Y_{i}\le t\}\sum_{\ell\in\mathcal{N}_Q(u)}K_{\ell}}{(K_{i}+d_{K}^{+}(u,Y_{i}))S_{\mathsf{Y}|\mathsf{U}}(Y_{i}|u)}
 \bigg[\frac{K_{i}+d_{K}^{+}(u,Y_{i})}{\sum_{\ell\in\mathcal{N}_Q(u)}K_{\ell}}-S_{\mathsf{Y}|\mathsf{U}}(Y_{i}|u)\bigg]\bigg|\nonumber\\
 & =\max_{i\in\mathcal{N}_Q(u)}\bigg|\underbrace{\frac{\eventInd_{i}\ind\{Y_{i}\le t\}\sum_{\ell\in\mathcal{N}_Q(u)}K_{\ell}}{(K_{i}+d_{K}^{+}(u,Y_{i}))S_{\mathsf{Y}|\mathsf{U}}(Y_{i}|u)}}_{\clubsuit_{i}:=}
 \bigg[\frac{K_{i}}{\sum_{\ell\in\mathcal{N}_Q(u)}K_{\ell}}+\underbrace{\widehat{S}_{\mathsf{Y}|\mathsf{U}}(Y_{i}|u)-S_{\mathsf{Y}|\mathsf{U}}(Y_{i}|u)}_{\text{CDF/tail estimation}}\bigg]\bigg|\nonumber\\
 & \le\sup_{\substack{i\in\mathcal{N}_Q(u),\\s\in[0,t]}}\Big|\clubsuit_{i}\Big[\frac{K_{i}}{\sum_{\ell\in\mathcal{N}_Q(u)}K_{\ell}}+\widehat{S}_{\mathsf{Y}|\mathsf{U}}(s|u)-S_{\mathsf{Y}|\mathsf{U}}(s|u)\Big]\Big|,
 \label{eq:W2-bound-helper1}
\end{align}
where the last step uses the fact that for $\clubsuit_{i}$ to be nonzero, we must have $Y_{i}\le t$, for which we then replace $Y_{i}$ with a worst case $s\in[0,t]$ in the CDF/tail estimation problem.

By squaring both sides of inequality~\eqref{eq:W2-bound-helper1} (and noting that the square can go into the sup on the RHS) followed by using the fact that $(\sum_{i=1}^{\ell}a_{i})^{2}\le\ell\sum_{i=1}^{\ell}a_{i}^{2}$,
\begin{align}
 (W_{2}(t|u))^{2}
 & \le\sup_{\substack{i\in\mathcal{N}_Q(u),\\s\in[0,t]}}\clubsuit_{i}^2\Big[\frac{K_{i}}{\sum_{\ell\in\mathcal{N}_Q(u)}K_{\ell}}+\widehat{S}_{\mathsf{Y}|\mathsf{U}}(s|u)-S_{\mathsf{Y}|\mathsf{U}}(s|u)\Big]^2\nonumber\\
 & =\sup_{\substack{i\in\mathcal{N}_Q(u),\\s\in[0,t]}}\clubsuit_{i}^2\Big[\frac{K_{i}}{\sum_{\ell\in\mathcal{N}_Q(u)}K_{\ell}} +\widehat{S}_{\mathsf{Y}|\mathsf{U}}(s|u)
 -\mathbb{E}_{Y_{1:n},\eventInd_{1:n}|U_{1:n}}[\widehat{S}_{\mathsf{Y}|\mathsf{U}}(s|u)]\nonumber\\
 & \phantom{=\sup_{\substack{i\in\mathcal{N}_Q(u),\\s\in[0,t]}}\clubsuit_{i}^2\Big[}
 +\mathbb{E}_{Y_{1:n},\eventInd_{1:n}|U_{1:n}}[\widehat{S}_{\mathsf{Y}|\mathsf{U}}(s|u)]-S_{\mathsf{Y}|\mathsf{U}}(s|u)\Big]^2\nonumber\\
 & \le\sup_{\substack{i\in\mathcal{N}_Q(u),\\s\in[0,t]}}\Big[3\clubsuit_{i}^{2}\Big(\frac{K_{i}}{\sum_{\ell\in\mathcal{N}_Q(u)}K_{\ell}}\Big)^{2} +3\clubsuit_{i}^{2}(\widehat{S}_{\mathsf{Y}|\mathsf{U}}(s|u)
 -\mathbb{E}_{Y_{1:n},\eventInd_{1:n}|U_{1:n}}[\widehat{S}_{\mathsf{Y}|\mathsf{U}}(s|u)])^{2}\nonumber\\
 & \phantom{\le\sup_{\substack{i\in\mathcal{N}_Q(u),\\s\in[0,t]}}\Big[}
   +3\clubsuit_{i}^{2}(\mathbb{E}_{Y_{1:n},\eventInd_{1:n}|U_{1:n}}[\widehat{S}_{\mathsf{Y}|\mathsf{U}}(s|u)]-S_{\mathsf{Y}|\mathsf{U}}(s|u))^{2}\Big] \nonumber \\
 & \le3\max_{i\in\mathcal{N}_Q(u)}\clubsuit_{i}^{2}\Big(\frac{K_{i}}{\sum_{\ell\in\mathcal{N}_Q(u)}K_{\ell}}\Big)^{2}
 +\Big(3\max_{i\in\mathcal{N}_Q(u)}\clubsuit_{i}^{2}\Big)\spadesuit^{2}
 +\Big(3\max_{i\in\mathcal{N}_Q(u)}\clubsuit_{i}^{2}\Big)\heartsuit^{2},
 \label{eq:W2-bound-helper2}
\end{align}
where the variance term $\spadesuit$ and bias term $\heartsuit$ are defined in equations~\eqref{eq:event-good-CDF} and~\eqref{eq:CDF-bias}. We now bound $\clubsuit_{i}^{2}(\frac{K_{i}}{\sum_{\ell\in\mathcal{N}_Q(u)}K_{\ell}})^{2}$, $\clubsuit_{i}^{2}$, and $\heartsuit^{2}$.

Since $d_{K}^{+}(u,\cdot)$ and $S_{\mathsf{Y}|\mathsf{U}}(\cdot|u)$ are monotonically decreasing and with the constraint in the numerator that $Y_{i}\le t$, we have
\begin{align}
 \clubsuit_{i}^{2}\Big(\frac{K_{i}}{\sum_{\ell\in\mathcal{N}_Q(u)}K_{\ell}}\Big)^{2}
 & =\Big(\frac{\eventInd_{i}\ind\{Y_{i}\le t\}K_{i}}{(K_{i}+d_{K}^{+}(u,Y_{i}))S_{\mathsf{Y}|\mathsf{U}}(Y_{i}|u)}\Big)^{2}\nonumber\\
 & \le\Big(\frac{\eventInd_{i}\ind\{Y_{i}\le t\}K_{i}}{d_{K}^{+}(u,t)S_{\mathsf{Y}|\mathsf{U}}(t|u)}\Big)^{2}\nonumber \\
 & \le\Big(\frac{\eventInd_{i}\ind\{Y_{i}\le t\}K_{i}}{K(\tau)\cdot\frac{|\mathcal{N}_Q(u)|\theta}{2}\cdot\theta}\Big)^{2}&&\text{by }\mathcal{E}_{\text{horizon}}(u),\mathbf{A}^{\text{survival}}\text{(a)}\nonumber\\
 & \le\Big(\frac{K(0)}{K(\tau)\cdot\frac{|\mathcal{N}_Q(u)|\theta}{2}\cdot\theta}\Big)^{2}\nonumber \\
 & =\frac{4K^{2}(0)}{K^{2}(\tau)\theta^{4}|\mathcal{N}_Q(u)|^{2}}\nonumber \\
 & \le\frac{4K^{2}(0)}{K^{2}(\tau)\theta^{4}\Xi_{u}^{2}}&&\text{by Lemma~\ref{lem:lower-bound-on-good-neighbors-deep-kernel-netting-regression}}.
 \label{eq:W2-bound-helper3}
\end{align}
By similar ideas,
\begin{equation}
\clubsuit_{i}^{2}
=\Big(\frac{\eventInd_{i}\ind\{Y_{i}\le t\}\sum_{\ell\in\mathcal{N}_Q(u)}K_{\ell}}{(K_{i}+d_{K}^{+}(u,Y_{i}))S_{\mathsf{Y}|\mathsf{U}}(Y_{i}|u)}\Big)^{2}
 \le\Big(\frac{|\mathcal{N}_Q(u)|K(0)}{(K(\tau)|\mathcal{N}_Q(u)|\theta/2)\theta}\Big)^{2} =\frac{4K^{2}(0)}{K^{2}(\tau)\theta^{4}}.
 \label{eq:W2-bound-helper4}
\end{equation}
Lastly, we bound $\heartsuit=\sup_{s\in[0,t]}\diamondsuit(s)$, where
$\diamondsuit(s):=|\mathbb{E}_{Y_{1:n}|U_{1:n}}[\widehat{S}_{\mathsf{Y}|\mathsf{U}}(s|u)]-S_{\mathsf{Y}|\mathsf{U}}(s|u)|$. For $s\in[0,t]$, we have
\begin{align}
 \mathbb{E}_{Y_{1:n}|U_{1:n}}[\widehat{S}_{\mathsf{Y}|\mathsf{U}}(s|u)]
 & =\mathbb{E}_{Y_{1:n}|U_{1:n}}\bigg[\sum_{j\in\mathcal{N}_Q(u)}\frac{K_{j}}{\sum_{\ell\in\mathcal{N}_Q(u)}K_{\ell}}\ind\{Y_{j}>s\}\bigg]\nonumber\\
 & =\sum_{j\in\mathcal{N}_Q(u)}\frac{K_{j}}{\sum_{\ell\in\mathcal{N}_Q(u)}K_{\ell}}\mathbb{E}_{Y_{j}|U_{j}}[\ind\{Y_{j}>s\}].\nonumber\\
 & =\sum_{j\in\mathcal{N}_Q(u)}\frac{K_{j}}{\sum_{\ell\in\mathcal{N}_Q(u)}K_{\ell}}S_{\mathsf{Y}|\mathsf{U}}(s|U_{j}).
 \label{eq:W2-bound-helper5}
\end{align}
Note that Assumption $\mathbf{A}^{\text{survival}}$(b) implies that $S_{\mathsf{Y}|\mathsf{U}}(s|\cdot)$ is H\"{o}lder continuous with parameters $(\lambda_{\mathsf{T}}+\lambda_{\mathsf{C}})s$ and $\alpha$. Then using equation~\eqref{eq:W2-bound-helper5}, H\"{o}lder's inequality, H\"{o}lder continuity, and Lemma~\ref{lem:neighbor-embedding},
\begin{align*}
 \diamondsuit(s)
 & = \bigg|\sum_{j\in\mathcal{N}_Q(u)}\frac{K_{j}}{\sum_{\ell\in\mathcal{N}_Q(u)}K_{\ell}}(S_{\mathsf{Y}|\mathsf{U}}(s|U_{j})-S_{\mathsf{Y}|\mathsf{U}}(s|u))\bigg|\\
 & \le\max_{j\in\mathcal{N}_Q(u)}|S_{\mathsf{Y}|\mathsf{U}}(s|U_{j})-S_{\mathsf{Y}|\mathsf{U}}(s|u)|\\
 & \le\max_{j\in\mathcal{N}_Q(u)}(\lambda_{\mathsf{T}}+\lambda_{\mathsf{C}})s\|U_{j}-u\|_{2}^{\alpha}\\
 & \le\max_{j\in\mathcal{N}_Q(u)}(\lambda_{\mathsf{T}}+\lambda_{\mathsf{C}})s(1+\beta)^{\alpha}\tau^{\alpha}\\
 & =(\lambda_{\mathsf{T}}+\lambda_{\mathsf{C}})t(1+\beta)^{\alpha}\tau^{\alpha}.
\end{align*}
Therefore,
\begin{equation}
\heartsuit=\sup_{s\in[0,t]}\diamondsuit(s)\le(\lambda_{\mathsf{T}}+\lambda_{\mathsf{C}})t(1+\beta)^{\alpha}\tau^{\alpha}.
\label{eq:W2-bound-helper6}
\end{equation}
Combining inequalities \eqref{eq:W2-bound-helper2}, \eqref{eq:W2-bound-helper3}, \eqref{eq:W2-bound-helper4}, and \eqref{eq:W2-bound-helper6} yields the claim.
\end{proof}

\paragraph{Bounding term III}
Lastly, we have the following bound.
\begin{lemma}
\label{lem:W3-bound}
\emph{(Bound on term III)}
Under the same setting as Lemma~\ref{lem:survival-estimator-convenient-form} (conditioning on $\mathcal{E}_{\beta}(u)$ and $\mathcal{E}_{\text{horizon}}(u)$) and with the addition of Assumption $\mathbf{A}^{\text{intrinsic}}$ and the requirement that $n\ge\frac{2K(0)}{K(\tau)\theta C_{d'}((1-\beta)\tau)^{d'}}$, we have
\[
\mathbb{E}_{Y_{1:n},\eventInd_{1:n}|U_{1:n}}[(W_{3}(t|u))^{2}]\le\frac{16K^{4}(0)}{K^{4}(\tau)\theta^{4}\Xi_{u}^{2}}.
\]
\end{lemma}
\begin{proof}
We write
$
|W_{3}(t|u)|=\sum_{i\in\mathcal{N}_Q(u)}\Upsilon_{i}
$,
where
\begin{align*}
\Upsilon_{i} & :=\eventInd_{i}\ind\{Y_{i}\le t\}\sum_{\ell=2}^{\infty}\frac{1}{\ell[1+(d_{K}^{+}(u,Y_{i})/K_{i})]^{\ell}}.
\end{align*}
Using the constraint that $Y_{i}\le t$, that $d_{K}^{+}(u,\cdot)$ monotonically decreases, and that $t\le t_{\text{horizon}}$, we have
\begin{align*}
\Upsilon_{i} & \le\eventInd_{i}\ind\{Y_{i}\le t\}\sum_{\ell=2}^{\infty}\frac{1}{\ell[1+(d_{K}^{+}(u,t)/K_{i})]^{\ell}}\\
 & \le\sum_{\ell=2}^{\infty}\frac{1}{\ell[1+(d_{K}^{+}(u,t)/K_{i})]^{\ell}}\\
 & \le\sum_{\ell=2}^{\infty}\frac{1}{\ell[1+(d_{K}^{+}(u,t_{\text{horizon}})/K_{i})]^{\ell}}.
\end{align*}
Next, recall that good event $\mathcal{E}_{\text{horizon}}(u)$ ensures that $d^{+}(u,t_{\text{horizon}})>\frac{|\mathcal{N}_Q(u)|\theta}{2}$. Furthermore, $d_{K}^{+}(u,t_{\text{horizon}})\ge K(\tau)d^{+}(u,t_{\text{horizon}})$, and $K_{i}\le K(0)$, so
\[
\frac{1}{\ell[1+(d_{K}^{+}(u,t_{\text{horizon}})/K_{i})]^{\ell}}\le\frac{1}{\ell\big[1+\frac{K(\tau)\frac{|\mathcal{N}_Q(u)|\theta}{2}}{K(0)}\big]^{\ell}}.
\]
Hence,
\[
\Upsilon_{i}\le\sum_{\ell=2}^{\infty}\frac{1}{\ell\big[1+\frac{K(\tau)|\mathcal{N}_Q(u)|\theta}{2K(0)}\big]^{\ell}}.
\]
Noting that $\sum_{\ell=2}^{\infty}\frac{1}{\ell(1+z)^{\ell}}=\log(1+\frac{1}{z})-\frac{1}{1+z}\le\frac{1}{(1+z)^{2}}$ for all $z\ge0.46241$, then provided that
\begin{equation}
\frac{K(\tau)|\mathcal{N}_Q(u)|\theta}{2K(0)}\ge0.46241,
\label{eq:W3-condition-to-check}
\end{equation}
we have 
\[
\Upsilon_{i} \le\frac{1}{(1+\frac{K(\tau)|\mathcal{N}_Q(u)|\theta}{2K(0)})^{2}} \le\frac{1}{(\frac{K(\tau)|\mathcal{N}_Q(u)|\theta}{2K(0)})^{2}} =\frac{4K^{2}(0)}{K^{2}(\tau)|\mathcal{N}_Q(u)|^{2}\theta^{2}}.
\]
Thus,
\begin{align*}
|W_{3}(t|u)| & =\sum_{i\in\mathcal{N}_Q(u)}\Upsilon_{i} \le\sum_{i\in\mathcal{N}_Q(u)}\frac{4K^{2}(0)}{K^{2}(\tau)|\mathcal{N}_Q(u)|^{2}\theta^{2}} \\
&=\frac{4K^{2}(0)}{K^{2}(\tau)|\mathcal{N}_Q(u)|\theta^{2}} \le\frac{4K^{2}(0)}{K^{2}(\tau)\theta^{2}}\cdot\frac{1}{\Xi_{u}}.
\end{align*}
Squaring and taking expectation $\mathbb{E}_{Y_{1:n},\eventInd_{1:n}|U_{1:n}}$ of both sides, we have
\[
\mathbb{E}_{Y_{1:n},\eventInd_{1:n}|U_{1:n}}[(W_{3}(t|u))^{2}]\le\frac{16K^{4}(0)}{K^{4}(\tau)\theta^{4}}\cdot\frac{1}{\Xi_{u}^{2}}.
\]
The only missing piece is to ensure that condition~\eqref{eq:W3-condition-to-check} holds. Under event $\mathcal{E}_{\beta}(u)$ and Assumption $\mathbf{A}^{\text{intrinsic}}$, with the help of Lemma~\ref{lem:lower-bound-on-good-neighbors-deep-kernel-netting-regression},
\[
|\mathcal{N}_Q(u)|
\ge\Xi_{u}>\frac{1}{2}n\mathbb{P}_{\mathsf{U}}(B(u,(1-\beta)\tau))\ge
\frac{1}{2}n C_{d'}((1-\beta)\tau)^{d'}.
\]
Thus, since we assume that
$n\ge\frac{2K(0)}{K(\tau)\theta C_{d'}((1-\beta)\tau)^{d'}}$,
we have
\begin{align*}
 \frac{K(\tau)|\mathcal{N}_Q(u)|\theta}{2}
 & \ge\frac{K(\tau)\frac{1}{2}nC_{d'}((1-\beta)\tau)^{d'}\theta}{2}\\
 & \ge\frac{K(\tau)\frac{1}{2}[\frac{2K(0)}{K(\tau)\theta C_{d'}((1-\beta)\tau)^{d'}}]C_{d'}((1-\beta)\tau)^{d'}\theta}{2}\\
 & =\frac{1}{2}\ge0.46241,
\end{align*}
which verifies that condition~\eqref{eq:W3-condition-to-check} holds.
\end{proof}

\paragraph{Deriving a final pointwise bound}

Putting together bound~\eqref{eq:W3-main-inequality} and Lemmas~\ref{lem:W1-bound}, \ref{lem:W2-bound}, and \ref{lem:W3-bound}, we have that when all three good events $\mathcal{E}_\beta(u)$, $\mathcal{E}_{\text{horizon}}(u)$, and $\mathcal{E}_{\text{CDF}}(u)$ hold, and $n\ge\frac{2K(0)}{K(\tau)\theta C_{d'}((1-\beta)\tau)^{d'}}$,
\begin{align*}
 & \mathbb{E}_{Y_{1:n},\eventInd_{1:n}|U_{1:n}}\big[(\widehat{S}_{\mathsf{T}|\mathsf{U}}(t|u)-S_{\mathsf{T}|\mathsf{U}}(t|u))^{2}\big]\\
 & \quad \le3\mathbb{E}_{Y_{1:n},\eventInd_{1:n}|U_{1:n}}[(W_{1}(t|u)-\log S_{\mathsf{T}|\mathsf{U}}(t|u))^{2}]+3\mathbb{E}_{Y_{1:n},\eventInd_{1:n}|U_{1:n}}[(W_{2}(t|u))^{2}]\\
 & \quad+3\mathbb{E}_{Y_{1:n},\eventInd_{1:n}|U_{1:n}}[(W_{3}(t|u))^{2}]\\
 & \quad \le3\bigg[\frac{K(0)}{4\theta^{2}K(\tau)\Xi_{u}}+\frac{(1+\beta)^{2\alpha}\tau^{2\alpha}}{\theta^{2}}\Big(\lambda_{\mathsf{T}}t+\frac{f_{\mathsf{T}}^{*}\lambda_{\mathsf{C}}t^{2}}{2}\Big)^{2}\bigg]\\
 & \quad\quad +3\bigg[\frac{12K^{2}(0)}{K^{2}(\tau)\Xi_{u}^{2}\theta^{4}}+\frac{12K^{2}(0)}{K^{2}(\tau)\theta^{4}}\varepsilon_{\text{CDF}}^{2}
 +\frac{12K^{2}(0)}{K^{2}(\tau)\theta^{4}}(\lambda_{\mathsf{T}}+\lambda_{\mathsf{C}})^{2}t^{2}(1+\beta)^{2\alpha}\tau^{2\alpha}\bigg]\\
 & \quad\quad +3\bigg[\frac{16K^{4}(0)}{K^{4}(\tau)\theta^{4}\Xi_{u}^{2}}\bigg]\\
 & \quad \le\frac{1}{n}\cdot\frac{332K^{4}(0)}{\theta^{4}K^{4}(\tau)\mathbb{P}_{\mathsf{U}}(B(u,(1-\beta)\tau))}\log\Big(n^{3}\frac{144K^{2}(\tau)}{9K^{2}(0)}\Big)\\
 & \quad\quad +(1+\beta)^{2\alpha}\tau^{2\alpha}\bigg[\frac{3}{\theta^{2}}\Big(\lambda_{\mathsf{T}}t+\frac{f_{\mathsf{T}}^{*}\lambda_{\mathsf{C}}t^{2}}{2}\Big)^{2}
 +\frac{36K^{2}(0)}{K^{2}(\tau)\theta^{4}}(\lambda_{\mathsf{T}}+\lambda_{\mathsf{C}})^{2}t^{2}\bigg].
\end{align*}
When not all good events occur, we resort to the trivial bound $${\mathbb{E}_{Y_{1:n},\eventInd_{1:n}|U_{1:n}}\big[(\widehat{S}_{\mathsf{T}|\mathsf{U}}(t|u)-S_{\mathsf{T}|\mathsf{U}}(t|u))^{2}\big]\le1}.$$ Abbreviating the three good events as $\mathcal{E}_{\text{good}}(u):=\mathcal{E}_\beta(u)\cup\mathcal{E}_{\text{horizon}}(u)\cup\mathcal{E}_{\text{CDF}}(u)$,
\begin{align}
 & \mathbb{E}_{Y_{1:n},\eventInd_{1:n},U_{1:n}}\big[(\widehat{S}_{\mathsf{T}|\mathsf{U}}(t|u)-S_{\mathsf{T}|\mathsf{U}}(t|u))^{2}\big]\nonumber\\
 & \quad =\mathbb{E}_{U_{1:n}}\big[\mathbb{E}_{Y_{1:n},\eventInd_{1:n}|U_{1:n}}\big[(\widehat{S}_{\mathsf{T}|\mathsf{U}}(t|u)-S_{\mathsf{T}|\mathsf{U}}(t|u))^{2}\ind\{\mathcal{E}_{\text{good}}(u)\}\nonumber\\
 & \quad\quad\! \phantom{=\mathbb{E}_{U_{1:n}}\big[\mathbb{E}_{Y_{1:n},\eventInd_{1:n}|U_{1:n}}}+(\widehat{S}_{\mathsf{T}|\mathsf{U}}(t|u)-S_{\mathsf{T}|\mathsf{U}}(t|u))^{2}\ind\{\mathcal{E}_{\text{good}}^c(u)\}\big]\big]\nonumber\\
 & \quad \le\mathbb{E}_{U_{1:n}}\big[\mathbb{E}_{Y_{1:n},\eventInd_{1:n}|U_{1:n}}\big[(\widehat{S}_{\mathsf{T}|\mathsf{U}}(t|u)-S_{\mathsf{T}|\mathsf{U}}(t|u))^{2}\ind\{\mathcal{E}_{\text{good}}(u)\}+\ind\{\mathcal{E}_{\text{good}}^c(u)\}\big]\big]\nonumber\\
 & \quad \le\frac{1}{n}\cdot\frac{332K^{4}(0)}{\theta^{4}K^{4}(\tau)\mathbb{P}_{\mathsf{U}}(B(u,(1-\beta)\tau))}\log\Big(n^{3}\frac{144K^{2}(\tau)}{9K^{2}(0)}\Big)\nonumber\\
 & \quad\quad +(1+\beta)^{2\alpha}\tau^{2\alpha}\bigg[\frac{3}{\theta^{2}}\Big(\lambda_{\mathsf{T}}t+\frac{f_{\mathsf{T}}^{*}\lambda_{\mathsf{C}}t^{2}}{2}\Big)^{2}
 +\frac{36K^{2}(0)}{K^{2}(\tau)\theta^{4}}(\lambda_{\mathsf{T}}+\lambda_{\mathsf{C}})^{2}t^{2}\bigg]\nonumber\\
 & \quad\quad +\mathbb{E}_{U_{1:n},Y_{1:n},\eventInd_{1:n}}[\ind\{\mathcal{E}_{\text{good}}^c(u)\}].
 \label{eq:survival-penultimate-ptwise-bound}
\end{align}
We next bound $\mathbb{E}_{U_{1:n},Y_{1:n},\eventInd_{1:n}}[\ind\{\mathcal{E}_{\text{good}}^c(u)\}]$.

\begin{lemma}
\label{lem:prob-survival-good-events}
\emph{(The survival analysis good events hold with high probability)}
Consider the same setting as Lemma~\ref{lem:survival-estimator-convenient-form} except without conditioning on any good events. Moreover, we add Assumption $\mathbf{A}^{\text{intrinsic}}$ and require that $n\ge{\frac{2}{C_{d'}((1-\beta)\tau)^{d'}}\big(\frac{9K^{2}(0)}{144K^{2}(\tau)}(2+\sqrt{2})\big)^{1/3}}$. Then
\[
\mathbb{E}_{U_{1:n},Y_{1:n},\eventInd_{1:n}}[\ind\{\mathcal{E}_{\text{good}}^c(u)\}]
\le\frac{14}{\theta^{2}n\mathbb{P}_{\mathsf{U}}(B(u,(1-\beta)\tau))}.
\]
\end{lemma}
We defer the proof of this lemma to Appendix~\ref{sec:pf-lem-prob-survival-good-events} as it is somewhat technical.

Putting together bound~\eqref{eq:survival-penultimate-ptwise-bound} and Lemma~\ref{lem:prob-survival-good-events}, and in particular absorbing Lemma~\ref{lem:prob-survival-good-events}'s bound into the leading error term in~\eqref{eq:survival-penultimate-ptwise-bound}, we get the final pointwise bound: for $n\ge{\frac{2}{C_{d'}((1-\beta)\tau)^{d'}}\max\big\{\frac{K(0)}{K(\tau)\theta},\big(\frac{9K^{2}(0)}{144K^{2}(\tau)}(2+\sqrt{2})\big)^{1/3}}\big\}$,
\begin{align}
 & \mathbb{E}_{Y_{1:n},\eventInd_{1:n},U_{1:n}}\big[(\widehat{S}_{\mathsf{T}|\mathsf{U}}(t|u)-S_{\mathsf{T}|\mathsf{U}}(t|u))^{2}\big]\nonumber\\
 & \quad\le\frac{346K^{4}(0)}{n\theta^{4}K^{4}(\tau)\mathbb{P}_{\mathsf{U}}(B(u,(1-\beta)\tau))}\log\Big(n^{3}\frac{144K^{2}(\tau)}{9K^{2}(0)}\Big)\nonumber\\
 & \quad\quad+(1+\beta)^{2\alpha}\tau^{2\alpha}\bigg[\frac{3}{\theta^{2}}\Big(\lambda_{\mathsf{T}}t+\frac{f_{\mathsf{T}}^{*}\lambda_{\mathsf{C}}t^{2}}{2}\Big)^{2}
 +\frac{36K^{2}(0)}{K^{2}(\tau)\theta^{4}}(\lambda_{\mathsf{T}}+\lambda_{\mathsf{C}})^{2}t^{2}\bigg].
 \label{eq:survival-pointwise-bound}
\end{align}

\paragraph{Completing the proof}
Finally, we account for randomness in the test embedding vector $U\sim\mathbb{P}_{\mathsf{U}}$ and integrate over time. Suppose that $n\ge{\frac{2}{C_{d'}((1-\beta)\tau)^{d'}}\max\big\{\frac{K(0)}{K(\tau)\theta},\big(\frac{9K^{2}(0)}{144K^{2}(\tau)}(2+\sqrt{2})\big)^{1/3}}\big\}$. By \mbox{Fubini's} theorem, iterated expectation, and the final pointwise bound~\eqref{eq:survival-pointwise-bound},
\begin{align*}
 & \mathbb{E}_{Y_{1:n},\eventInd_{1:n},U_{1:n},U}\Big[\frac{\int_{0}^{t_{\text{horizon}}}(\widehat{S}_{\mathsf{T}|\mathsf{U}}(t|U)-S_{\mathsf{T}|\mathsf{U}}(t|U))^{2}dt}{t_{\text{horizon}}}\Big]\\
 & \quad =\underset{~~~~[0,t_{\text{horizon}}]}{\int}\frac{\mathbb{E}_{U}[\mathbb{E}_{Y_{1:n},\eventInd_{1:n},U_{1:n}|U}[(\widehat{S}_{\mathsf{T}|\mathsf{U}}(t|U)-S_{\mathsf{T}|\mathsf{U}}(t|U))^{2}]]}{t_{\text{horizon}}}dt\\
 & \quad \le\underset{~~~~[0,t_{\text{horizon}}]}{\int}\frac{1}{t_{\text{horizon}}}\mathbb{E}_{U}\Bigg[\frac{1}{n}\cdot\frac{346K^{4}(0)\log(n^{3}\frac{144K^{2}(\tau)}{9K^{2}(0)})}{\theta^{4}K^{4}(\tau)\mathbb{P}_{\mathsf{U}}(B(U,(1-\beta)\tau))}\\
 & \qquad\qquad\qquad\qquad\quad\quad\;+(1+\beta)^{2\alpha}\tau^{2\alpha}\bigg[\frac{3}{\theta^{2}}\Big(\lambda_{\mathsf{T}}t+\frac{f_{\mathsf{T}}^{*}\lambda_{\mathsf{C}}t^{2}}{2}\Big)^{2}
 +\frac{36K^{2}(0)}{K^{2}(\tau)\theta^{4}}(\lambda_{\mathsf{T}}+\lambda_{\mathsf{C}})^{2}t^{2}\bigg]\Bigg]dt \\
 & \quad =\frac{346K^{4}(0)\log(n^{3}\frac{144K^{2}(\tau)}{9K^{2}(0)})}{n\theta^{4}K^{4}(\tau)}\mathbb{E}_{U}\Big[\frac{1}{\mathbb{P}_{\mathsf{U}}(B(U,(1-\beta)\tau))}\Big]\\
 & \quad\quad +(1+\beta)^{2\alpha}\tau^{2\alpha}\bigg[\frac{t_{\text{horizon}}^{2}}{\theta^{2}}\Big(\lambda_{\mathsf{T}}^2+\frac{3\lambda_{\mathsf{T}}f_{\mathsf{T}}^{*}\lambda_{\mathsf{C}}t_{\text{horizon}}}{4}
 +\frac{3(f_{\mathsf{T}}^{*})^{2}\lambda_{\mathsf{C}}^{2}t_{\text{horizon}}^{2}}{20}\Big)\\
 & \quad \phantom{\quad+(1+\beta)^{2\alpha}\tau^{2\alpha}\bigg[}~
 +\frac{12K^{2}(0)}{K^{2}(\tau)\theta^{4}}(\lambda_{\mathsf{T}}+\lambda_{\mathsf{C}})^{2}t_{\text{horizon}}^{2}\bigg].
\end{align*}
At this point, we can bound $\mathbb{E}_{U}[\frac{1}{\mathbb{P}_{\mathsf{U}}(B(U,(1-\beta)\tau))}]$. This is where covering numbers and intrinsic dimension come into play.

\paragraph{Case 1 (using a covering argument)} Let $\zeta_{1},\zeta_{2},\dots,\zeta_{N}$ be a $\frac{(1-\beta)\tau}{2}$-cover of $\mathcal{U}$ that has the smallest size possible, i.e., $N=N({\textstyle \frac{(1-\beta)\tau}{2}};\mathcal{U})$. Then
\begin{align}
 \mathbb{E}_{U}\Big[\frac{1}{\mathbb{P}_{\mathsf{U}}(B(U,(1-\beta)\tau))}\Big]
 & \le\sum_{j=1}^{N}\mathbb{E}_{U}\Big[\frac{\ind\{U\in B(\zeta_{j},\frac{(1-\beta)\tau}{2})\}}{\mathbb{P}_{\mathsf{U}}(B(U,(1-\beta)\tau))}\Big].
 \label{eq:deep-kernel-netting-regression-ptwise-bias-helper11}
\end{align}
Note that $U\in B(\zeta_{j},\frac{(1-\beta)\tau}{2})$ implies that $B(U,(1-\beta)\tau)$ contains $B(\zeta_{j},\frac{(1-\beta)\tau}{2})$. To see this, let $u'$ be any point in $B(\zeta_{j},\frac{(1-\beta)\tau}{2})$. Then by the triangle inequality,
\begin{align*}
\|u'-U\|_{2}
&\le\underbrace{\|u'-\zeta_{j}\|_{2}}_{\le\frac{(1-\beta)\tau}{2}}+\underbrace{\|\zeta_{j}-U\|_{2}}_{\frac{(1-\beta)\tau}{2}}
\le(1-\beta)\tau,
\end{align*}
i.e., $u'$ is also in $B(U,(1-\beta)\tau)$. Hence, $U\in B(\zeta_{j},\frac{(1-\beta)\tau}{2})$ implies that $B(U,(1-\beta)\tau)\supseteq B(\zeta_{j},\frac{(1-\beta)\tau}{2})$, which in turn implies that $\mathbb{P}_{\mathsf{U}}(B(U,(1-\beta)\tau))\ge \mathbb{P}_{\mathsf{U}}(B(\zeta_{j},\frac{(1-\beta)\tau}{2}))$.
Thus, the RHS of inequality~\eqref{eq:deep-kernel-netting-regression-ptwise-bias-helper11} can be bounded as
\begin{align*}
 \sum_{j=1}^{N}\mathbb{E}_{U}\Big[\frac{\ind\{U\in B(\zeta_{j},\frac{(1-\beta)\tau}{2})\}}{\mathbb{P}_{\mathsf{U}}(B(U,(1-\beta)\tau))}\Big]
 \le\sum_{j=1}^{N}\mathbb{E}_{U}\Big[\frac{\ind\{U\in B(\zeta_{j},\frac{(1-\beta)\tau}{2})\}}{\mathbb{P}_{\mathsf{U}}(B(\zeta_j,\frac{(1-\beta)\tau}{2}))}\Big]
 =N =N({\textstyle \frac{(1-\beta)\tau}{2}};\mathcal{U}),
\end{align*}
at which point we conclude that $\mathbb{E}_{U}[\frac{1}{\mathbb{P}_{\mathsf{U}}(B(U,(1-\beta)\tau))}]\le N({\textstyle \frac{(1-\beta)\tau}{2}};\mathcal{U})$.

\paragraph{Case 2 (using Assumption~$\mathbf{A}^{\normalfont \textit{intrinsic}}$)} For all $U\in\mathcal{U}$,
we have
$
\frac{1}{\mathbb{P}_{\mathsf{U}}(B(U,(1-\beta)\tau))}\le\frac{1}{C_{d'}((1-\beta)\tau)^{d'}}$.
Hence,
$\mathbb{E}_{U}[\frac{1}{\mathbb{P}_{\mathsf{U}}(B(U,(1-\beta)\tau))}]
\le\frac{1}{C_{d'}((1-\beta)\tau)^{d'}}$.
$\hfill\blacksquare$

\subsubsection{Proof of Lemma~\ref{lem:prob-survival-good-events}}
\label{sec:pf-lem-prob-survival-good-events}

Observe that
\begin{align*}
 \mathcal{E}_{\text{good}}(u)
 & = (\mathcal{E}_{\beta}(u)\cap\mathcal{E}_{\text{horizon}}(u)\cap\mathcal{E}_{\text{CDF}}(u))^{c}\\
 & =\mathcal{E}_{\beta}^c(u)\cup\mathcal{E}_{\text{horizon}}^c(u)\cup\mathcal{E}_{\text{CDF}}^c(u)\\
 & =\mathcal{E}_{\beta}^c(u)\cup[\mathcal{E}_{\text{horizon}}^c(u)\cap\mathcal{E}_{\beta}(u)]
 \cup[\mathcal{E}_{\text{CDF}}^c(u)\cap\mathcal{E}_{\beta}(u)].
\end{align*}
Hence, by a union bound
\begin{align}
&\mathbb{E}_{U_{1:n},Y_{1:n},\eventInd_{1:n}}[\ind\{\mathcal{E}_{\text{good}}(u)\}]\nonumber\\
&\quad \le
\underbrace{\mathbb{E}_{U_{1:n},Y_{1:n},\eventInd_{1:n}}[\ind\{\mathcal{E}_{\beta}^c(u)\}]}_{(a)}
+
\underbrace{\mathbb{E}_{U_{1:n},Y_{1:n},\eventInd_{1:n}}[\ind\{\mathcal{E}_{\text{horizon}}^c(u)\cap\mathcal{E}_{\beta}(u)\}]}_{(b)} \nonumber\\
&\quad\quad+
\underbrace{\mathbb{E}_{U_{1:n},Y_{1:n},\eventInd_{1:n}}[\ind\{\mathcal{E}_{\text{CDF}}^c(u)\cap\mathcal{E}_{\beta}(u)\}]}_{(c)}.
\label{eq:pf-lem-prob-survival-good-events-CDF-union-bound}
\end{align}
We upper-bound each of these three terms next.

\paragraph{Term $(a)$} This term is bounded precisely by Lemma~\ref{lem:prob-event-good-neighbors}. In particular,
\begin{equation}
\text{term }(a)\le\frac{8}{n\mathbb{P}_{\mathsf{U}}(B(u,(1-\beta)\tau))}.
\label{eq:pf-lem-prob-survival-good-events-CDF-term-a-bound}
\end{equation}

\paragraph{Term $(b)$} For term $(b)$, note that within the expectation there is no dependence on $\eventInd_{1:n}$. In particular,
\begin{align}
 & \mathbb{E}_{U_{1:n},Y_{1:n},\eventInd_{1:n}}[\ind\{\mathcal{E}_{\text{horizon}}^c(u)\cap\mathcal{E}_{\beta}(u)\}] \nonumber\\
 & \quad=\mathbb{E}_{U_{1:n},Y_{1:n},\eventInd_{1:n}}\Big[\ind\Big\{ d^{+}(u,t_{\text{horizon}})\le\frac{|\mathcal{N}_Q(u)|\theta}{2}\Big\}\ind\Big\{\Xi_{u}>\frac{1}{2}n\mathbb{P}_{\mathsf{U}}(B(u,(1-\beta)\tau))\Big\}\Big]\nonumber\\
 & \quad=\mathbb{E}_{U_{1:n},Y_{1:n}}\Big[\ind\Big\{ d^{+}(u,t_{\text{horizon}})\le\frac{|\mathcal{N}_Q(u)|\theta}{2}\Big\}
 \ind\Big\{\Xi_{u}>\frac{1}{2}n\mathbb{P}_{\mathsf{U}}(B(u,(1-\beta)\tau))\Big\}\Big]\nonumber\\
 & \quad=\mathbb{E}_{U_{1:n}}\Big[\ind\{\Xi_{u}>\frac{1}{2}n\mathbb{P}_{\mathsf{U}}(B(u,(1-\beta)\tau))\}
 \mathbb{E}_{Y_{1:n}|U_{1:n}}\Big[\ind\Big\{ d^{+}(u,t_{\text{horizon}})\le\frac{|\mathcal{N}_Q(u)|\theta}{2}\Big\}\Big]\Big].
 \label{eq:pf-lem-prob-survival-good-events-horizon-helper1}
\end{align}
By Assumption $\mathbf{A}^{\text{survival}}$(a),
\begin{equation}
|\mathcal{N}_Q(u)|\theta=\sum_{j\in\mathcal{N}_Q(u)}\theta\le\sum_{j\in\mathcal{N}_Q(u)}S_{\mathsf{Y}|\mathsf{U}}(t_{\text{horizon}}|U_{j})=:\mu^{+}.
\label{eq:pf-lem-prob-survival-good-events-horizon-key-inequality}
\end{equation}
Hence, $d^{+}(u,t_{\text{horizon}})\le\frac{|\mathcal{N}_Q(u)|\theta}{2}$ implies that $d^{+}(u,t_{\text{horizon}})\le\frac{\mu^{+}}{2}$, i.e.,
\[
\ind\Big\{ d^{+}(u,t_{\text{horizon}})\le\frac{|\mathcal{N}_Q(u)|\theta}{2}\Big\} \le\ind\Big\{ d^{+}(u,t_{\text{horizon}})\le\frac{\mu^{+}}{2}\Big\}.
\]
Thus,
\begin{align}
 & \mathbb{E}_{U_{1:n}}\Big[\ind\{\Xi_{u}>\frac{1}{2}n\mathbb{P}_{\mathsf{U}}(B(u,(1-\beta)\tau))\}
 \mathbb{E}_{Y_{1:n}|U_{1:n}}\Big[\ind\Big\{ d^{+}(u,t_{\text{horizon}})\le\frac{|\mathcal{N}_Q(u)|\theta}{2}\Big\}\Big]\Big]\nonumber\\
 & \quad\le\mathbb{E}_{U_{1:n}}\Big[\ind\{\Xi_{u}>\frac{1}{2}n\mathbb{P}_{\mathsf{U}}(B(u,(1-\beta)\tau))\}
 \mathbb{E}_{Y_{1:n}|U_{1:n}}\Big[\ind\Big\{ d^{+}(u,t_{\text{horizon}})\le\frac{\mu^{+}}{2}\Big\}\Big]\Big].
 \label{eq:pf-lem-prob-survival-good-events-horizon-helper2}
\end{align}
We now upper-bound $\mathbb{E}_{Y_{1:n}|U_{1:n}}[\ind\{d^{+}(u,t_{\text{horizon}})\le\frac{\mu^{+}}{2}\}]$, under the constraint that $\Xi_{u}>\frac{1}{2}n\mathbb{P}_{\mathsf{U}}(B(u,(1-\beta)\tau))$. Importantly, since the $Y_{1:n}$ variables are conditionally independent given $U_{1:n}$ and since $\mathcal{N}_Q(u)$ is deterministic given $U_{1:n}$, then after conditioning on $U_{1:n}$, the sum
\[
d^{+}(u,t_{\text{horizon}})=\sum_{j\in\mathcal{N}_Q(u)}\ind\{Y_{j}>t_{\text{horizon}}\}
\]
is over independent random variables and, furthermore, using Lemma~\ref{lem:lower-bound-on-good-neighbors-deep-kernel-netting-regression}, the above summation is not vacuous since $|\mathcal{N}_Q(u)|\ge\Xi_{u}>0$. Thus, with $\mathcal{N}_Q(u)$ nonempty, the expectation of the above sum is
\[
 \mathbb{E}_{Y_{1:n}|U_{1:n}}[d^{+}(u,t_{\text{horizon}})]
 =\sum_{j\in\mathcal{N}_Q(u)}\mathbb{E}_{Y_{j}|U_{j}}[\ind\{Y_{j}>t_{\text{horizon}}\}]
 =\sum_{j\in\mathcal{N}_Q(u)}S_{\mathsf{Y}|\mathsf{U}}(t_{\text{horizon}}|U_{j}) =\mu^{+}.
\]
Applying Hoeffding's inequality
\begin{align}
 \mathbb{E}_{Y_{1:n}|U_{1:n}}\Big[\ind\Big\{ d^{+}(u,t_{\text{horizon}})\le\frac{\mu^{+}}{2}\Big\}\Big]
 & \le\exp\Big(-\frac{2(\frac{1}{2}\mu^{+})^{2}}{|\mathcal{N}_Q(u)|}\Big)\nonumber\\
 & \le\exp\Big(-\frac{2(\frac{|\mathcal{N}_Q(u)|\theta}{2})^{2}}{|\mathcal{N}_Q(u)|}\Big)&&\text{by inequality~\eqref{eq:pf-lem-prob-survival-good-events-horizon-key-inequality}}\nonumber\\
 & =\exp\Big(-\frac{\theta^{2}}{2}|\mathcal{N}_Q(u)|\Big)\nonumber\\
 & \le\exp\Big(-\frac{\theta^{2}}{2}\Xi_{u}\Big)&&\text{by Lemma~\ref{lem:lower-bound-on-good-neighbors-deep-kernel-netting-regression}}\nonumber\\
 & \le\frac{2}{\theta^{2}\Xi_{u}}\nonumber\\
 & <\frac{4}{\theta^{2}n\mathbb{P}_{\mathsf{U}}(B(u,(1-\beta)\tau))}&&\text{by good event }\mathcal{E}_{\beta}(u),
 \label{eq:pf-lem-prob-survival-good-events-horizon-helper3}
\end{align}
where the last step uses the constraint $\Xi_{u}>\frac{1}{2}n\mathbb{P}_{\mathsf{U}}(B(u,(1-\beta)\tau))$.

Stringing together inequalities~\eqref{eq:pf-lem-prob-survival-good-events-horizon-helper1}, \eqref{eq:pf-lem-prob-survival-good-events-horizon-helper2}, and \eqref{eq:pf-lem-prob-survival-good-events-horizon-helper3}, we see that
\begin{equation}
\text{term }(b)\le\frac{4}{\theta^{2}n\mathbb{P}_{\mathsf{U}}(B(u,(1-\beta)\tau))}.
\label{eq:pf-lem-prob-survival-good-events-CDF-term-b-bound}
\end{equation}

\paragraph{Term $(c)$} Recall that we use the shorthand notation
$\spadesuit=\sup_{s\ge0}|\widehat{S}_{\mathsf{Y}|\mathsf{U}}(s|u)-\mathbb{E}_{Y_{1:n}|U_{1:n}}[\widehat{S}_{\mathsf{Y}|\mathsf{U}}(s|u)]|$. Then term $(c)$ can be written as
\begin{align}
 & \mathbb{E}_{U_{1:n},Y_{1:n},\eventInd_{1:n}}\Big[\ind\{\spadesuit>\varepsilon_{\text{CDF}}\}
 \ind\Big\{\Xi_{u}>\frac{1}{2}n\mathbb{P}_{\mathsf{U}}(B(u,(1-\beta)\tau))\Big\}\Big]\nonumber\\
 & \quad=\mathbb{E}_{U_{1:n},Y_{1:n}}\Big[\ind\{\spadesuit>\varepsilon_{\text{CDF}}\}
 \ind\Big\{\Xi_{u}>\frac{1}{2}n\mathbb{P}_{\mathsf{U}}(B(u,(1-\beta)\tau))\Big\}\Big]\nonumber\\
 & \quad=\mathbb{E}_{U_{1:n}}\Big[\ind\Big\{\Xi_{u}>\frac{1}{2}n\mathbb{P}_{\mathsf{U}}(B(u,(1-\beta)\tau))\Big\}
 \mathbb{E}_{Y_{1:n}|U_{1:n}}\Big[\ind\{\spadesuit>\varepsilon_{\text{CDF}}\}\Big]\Big].
 \label{eq:pf-lem-prob-survival-good-events-CDF-helper1}
\end{align}
Under the constraint that $\Xi_{u}>\frac{1}{2}n\mathbb{P}_{\mathsf{U}}(B(u,(1-\beta)\tau))$ and conditioned on $U_{1:n}$,
\[
1-\widehat{S}_{\mathsf{Y}|\mathsf{U}}(s|u)=\sum_{j\in\mathcal{N}_Q(u)}\frac{K_{j}}{\sum_{\ell\in\mathcal{N}_Q(u)}K_{\ell}}\ind\{Y_{j}\le s\}
\]
is a weighted empirical distribution constructed from more than $\frac{1}{2}n\mathbb{P}_{\mathsf{U}}(B(u,(1-\beta)\tau))$ samples. Applying Proposition~3.1 of \citet{chen2019nearest},
\begin{align}
 \mathbb{E}_{Y_{1:n}|U_{1:n}}[\ind\{\spadesuit>\varepsilon_{\text{CDF}}\}]
 & \le\frac{6}{\varepsilon_{\text{CDF}}}\exp\Big(-\frac{2\varepsilon_{\text{CDF}}^{2}(\sum_{j\in\mathcal{N}_Q(u)}K_{j})^{2}}{9\sum_{\ell\in\mathcal{N}_Q(u)}K_{\ell}^{2}}\Big)\nonumber\\
 & \le\frac{6}{\varepsilon_{\text{CDF}}}\exp\Big(-\frac{2\varepsilon_{\text{CDF}}^{2}(|\mathcal{N}_Q(u)|K(\tau))^{2}}{9|\mathcal{N}_Q(u)|K^{2}(0)}\Big)\nonumber\\
 & =\frac{6}{\varepsilon_{\text{CDF}}}\exp\Big(-\frac{2\varepsilon_{\text{CDF}}^{2}K^{2}(\tau)}{9K^{2}(0)}|\mathcal{N}_Q(u)|\Big).
 \label{eq:pf-lem-prob-survival-good-events-CDF-helper2}
\end{align}
We show that our choice of $\varepsilon_{\text{CDF}}$ ensures that the RHS is at most $\frac{1}{|\mathcal{N}_Q(u)|}$, i.e., we want to show that
\begin{equation}
\frac{6}{\varepsilon_{\text{CDF}}}\exp\Big(-\frac{2\varepsilon_{\text{CDF}}^{2}K^{2}(\tau)}{9K^{2}(0)}|\mathcal{N}_Q(u)|\Big)\le\frac{1}{|\mathcal{N}_Q(u)|}.
\label{eq:pf-lem-prob-survival-good-events-CDF-helper3}
\end{equation}
Lemma D.1(b) of \citet{chen2019nearest} says that inequality~\eqref{eq:pf-lem-prob-survival-good-events-CDF-helper3} holds under the sufficient conditions
\begin{equation}
\varepsilon_{\text{CDF}}<6|\mathcal{N}_Q(u)|\quad\text{and}\quad|\mathcal{N}_Q(u)|\ge\Big(\frac{9K^{2}(0)e}{144K^{2}(\tau)}\Big)^{1/3}.
\label{eq:chen2019-lemma-d1-sufficient}
\end{equation}
The latter holds since we assume that $n\ge{\frac{2}{C_{d'}((1-\beta)\tau)^{d'}}\big(\frac{9K^{2}(0)}{144K^{2}(\tau)}(2+\sqrt{2})\big)^{1/3}}$, so using good event $\mathcal{E}_{\beta}(u)$, Lemma~\ref{lem:lower-bound-on-good-neighbors-deep-kernel-netting-regression} and Assumption $\mathbf{A}^{\text{intrinsic}}$, we have
\begin{align}
 |\mathcal{N}_Q(u)|
 & \ge\Xi_{u} \nonumber\\
 & >\frac{1}{2}n\mathbb{P}_{\mathsf{U}}(B(u,(1-\beta)\tau))\nonumber\\
 & \ge\frac{1}{2}n C_{d'}((1-\beta)\tau)^{d'}\nonumber\\
 & \ge\frac{1}{2}\Big[\frac{2}{C_{d'}((1-\beta)\tau)^{d'}}\Big(\frac{9K^{2}(0)}{144K^{2}(\tau)}(2+\sqrt{2})\Big)^{1/3}\Big]
 C_{d'}((1-\beta)\tau)^{d'} \nonumber\\
 & =\Big(\frac{9K^{2}(0)}{144K^{2}(\tau)}(2+\sqrt{2})\Big)^{1/3},
 \label{eq:pf-lem-prob-survival-good-events-CDF-large-enough-n-good-outcome}
\end{align}
which is strictly greater than $\big(\frac{9K^{2}(0)e}{144K^{2}(\tau)}\big)^{1/3}$.

At this point, it suffices for us to show that the condition $\varepsilon_{\text{CDF}}<6|\mathcal{N}_Q(u)|$ holds. This inequality can be written as
\begin{equation}
|\mathcal{N}_Q(u)|^{3}  >\frac{9K^{2}(0)}{144K^{2}(\tau)}\log\Big(|\mathcal{N}_Q(u)|^{3}\frac{144K^{2}(\tau)}{9K^{2}(0)}\Big).
\label{eq:eps-CDF-upper-bound-condition}
\end{equation}
Define
\begin{align*}
a&:=\frac{9K^{2}(0)}{144K^{2}(\tau)}, &b&:=\frac{9K^{2}(0)}{144K^{2}(\tau)}\log\Big(\frac{144K^{2}(\tau)}{9K^{2}(0)}\Big)+\frac{18K^{2}(0)}{144K^{2}(\tau)}.
\end{align*}
Then Lemma D.2(c) of \citet{chen2019nearest} says that if $\frac{b}{a}+\log a>1$ and
\begin{equation}
 |\mathcal{N}_Q(u)|^{3}
 \ge a\Big(1+\sqrt{2\log(ae^{b/a-1})}+\log(ae^{b/a-1})\Big)
 =\Big(\frac{9K^{2}(0)}{144K^{2}(\tau)}(2+\sqrt{2})\Big)^{1/3},
 \label{eq:chen2019-lemma-d2-condition}
\end{equation}
then
\[
 |\mathcal{N}_Q(u)|^{3}
 \ge a\log(|\mathcal{N}_Q(u)|^{3})+b
 =\frac{9K^{2}(0)}{144K^{2}(\tau)}\log\Big(|\mathcal{N}_Q(u)|^{3}\frac{144K^{2}(\tau)}{9K^{2}(0)}\Big)+\frac{18K^{2}(0)}{144K^{2}(\tau)},
\]
which implies that condition~\eqref{eq:eps-CDF-upper-bound-condition} holds. In fact, we have already shown that condition~\eqref{eq:chen2019-lemma-d2-condition} holds (see inequality~\eqref{eq:pf-lem-prob-survival-good-events-CDF-large-enough-n-good-outcome}). Thus, it suffices to show that $\frac{b}{a}+\log a>1$. Indeed,
\[
 \frac{b}{a}+\log a
 =\frac{\frac{9K^{2}(0)}{144K^{2}(\tau)}\log\Big(\frac{144K^{2}(\tau)}{9K^{2}(0)}\Big)+\frac{18K^{2}(0)}{144K^{2}(\tau)}}{\frac{9K^{2}(0)}{144K^{2}(\tau)}}
 +\log\Big(\frac{9K^{2}(0)}{144K^{2}(\tau)}\Big)\\
 =2>1.
\]
At this point, we can conclude that condition~\eqref{eq:eps-CDF-upper-bound-condition} holds, which completes our justification that the sufficient conditions~\eqref{eq:chen2019-lemma-d1-sufficient} hold. Thus, we can combine equation~\eqref{eq:pf-lem-prob-survival-good-events-CDF-helper1}, inequality~\eqref{eq:pf-lem-prob-survival-good-events-CDF-helper2}, and inequality~\eqref{eq:pf-lem-prob-survival-good-events-CDF-helper3} to get that
\begin{equation}
\text{term }(c)\le\frac{1}{|\mathcal{N}_Q(u)|}\le\frac{1}{\Xi_{u}}<\frac{2}{n\mathbb{P}_{\mathsf{U}}(B(u,(1-\beta)\tau))},
\label{eq:pf-lem-prob-survival-good-events-CDF-term-c-bound}
\end{equation}
making use of event $\mathcal{E}_{\beta}(u)$ and Lemma~\ref{lem:lower-bound-on-good-neighbors-deep-kernel-netting-regression}.

\paragraph{Completing the proof} Plugging the bounds on terms $(a)$, $(b)$, and $(c)$ (given in inequalities \eqref{eq:pf-lem-prob-survival-good-events-CDF-term-a-bound}, \eqref{eq:pf-lem-prob-survival-good-events-CDF-term-b-bound}, and \eqref{eq:pf-lem-prob-survival-good-events-CDF-term-c-bound}) back into the union bound~\eqref{eq:pf-lem-prob-survival-good-events-CDF-union-bound}, we get
\begin{align*}
 & \mathbb{E}_{U_{1:n},Y_{1:n},\eventInd_{1:n}}[\ind\{(\mathcal{E}_{\beta}(u)\cap\mathcal{E}_{\text{horizon}}(u)\cap\mathcal{E}_{\text{CDF}}(u))^{c}\}]\\
 & \quad
 \le\frac{8}{n\mathbb{P}_{\mathsf{U}}(B(u,(1-\beta)\tau))}
   +\frac{4}{\theta^{2}n\mathbb{P}_{\mathsf{U}}(B(u,(1-\beta)\tau))}
   +\frac{2}{n\mathbb{P}_{\mathsf{U}}(B(u,(1-\beta)\tau))}\\
 & \quad
 \le\frac{14}{\theta^{2}n\mathbb{P}_{\mathsf{U}}(B(u,(1-\beta)\tau))}.\tag*{$\blacksquare$}
\end{align*}

\section{Warm-Start Strategy by \texorpdfstring{\citet{chen2020deep}}{Chen (2020)}\label{subsec:warm-start}}

Previously, \citet{chen2020deep} showed how to warm-start deep kernel survival analysis using any pre-trained kernel function, such as one obtained from decision trees or their ensembled variants (random forests, gradient tree boosting). However, Chen's strategy does not scale to large datasets, as we now explain in more detail.

Let $\mathbb{K}_{0}$ be a pre-trained kernel function (learned from, for instance, a random survival forest or \textsc{xgboost}). Then compute the $n_{\pre}$-by-$n_{\pre}$ matrix~$\mathbf{K}$, where $\mathbf{K}_{i,j}=\mathbb{K}_{0}(X_i^{\pre},X_j^{\pre})$ for every pair of pre-training feature vectors $X_i^{\pre}$ and $X_j^{\pre}$.\footnote{Technically, in the original paper by \citet{chen2020deep}, the actual training feature vectors are used rather than pre-training feature vectors, but we describe the setup here using pre-training feature vectors to have the exposition be closer to that of Section~\ref{sec:scalable-tree-ensemble-warm-start}).} Note that computing the matrix $\mathbf{K}$ is prohibitively expensive for large datasets, where even storing this whole matrix can be impractical. For the moment, suppose that we can compute and store this matrix. Next, if we are using a Gaussian kernel $\mathbb{K}(X_i^{\pre},X_j^{\pre})=\exp(-\|\phi(X_i^{\pre})-\phi(X_j^{\pre})\|_2^2)$, then by setting this expression equal to $\mathbf{K}_{i,j}$ and with a bit of algebra, we obtain
\begin{equation*}
\|\phi(X_i^{\pre})-\phi(X_j^{\pre})\|_2=\sqrt{\log\frac{1}{\mathbf{K}_{i,j}}}\quad\text{for all }i,j.
\end{equation*}
To prevent division by 0 in the log, we could for instance add a small constant to all values of $\mathbf{K}$ or clip values of $\mathbf{K}$ under a user-specified threshold to be equal to the threshold.

Next, we can use metric multidimensional scaling (MDS) \citep{borg2005modern} to learn an embedding vector $\widetilde{X}_i^{\text{MDS}}$ for each pre-training feature $X_i^{\pre}$ such that ${\|\widetilde{X}_i^{\text{MDS}}-\widetilde{X}_j^{\text{MDS}}\|_{2}}\approx\sqrt{\log\frac{1}{\mathbf{K}_{i,j}}}$ for all pairs $i$ and $j$. Then for a base neural net $\phi$ with a user-specified architecture, we warm-start $\phi$ by minimizing the mean squared error loss
\[
\frac{1}{n}\sum_{i=1}^{n}(\phi(X_i^{\pre})-\widetilde{X}_{i}^{\text{MDS}})^{2}.
\]
Effectively we are having the neural net $\phi$ approximate the MDS embedding, which approximates the Euclidean distances induced by the pre-trained kernel function under the assumption of a Gaussian kernel. This MDS-based strategy generalizes to other choices of kernel functions that can be written in terms of the distance function $\rho(x,x')=\|\phi(x)-\phi(x')\|_{2}$, so long as we can solve the equation $\mathbb{K}(X_i^{\pre},X_j^{\pre})=\mathbf{K}_{i,j}$ for $\rho(X_i^{\pre},X_j^{\pre})$; in fact, $\rho$ could even be chosen to be non-Euclidean if generalized MDS is used, which handles non-Euclidean distances \citep{bronstein2006generalized}.

Our proposed alternative warm-start strategy in Section~\ref{sec:scalable-tree-ensemble-warm-start} avoids ever storing a full $n_{\pre}$-by-$n_{\pre}$ Gram matrix by instead only ever looking at the Gram matrix restricted to minibatches. Our warm-start strategy also does not need to use any MDS solver.

\section{Details on Experiments}
\label{sec:experiment-details}

We now provide additional details on data preprocessing (Appendix~\ref{sec:datasets-preprocessing-details}), and on hyperparameter grids and optimization (Appendix~\ref{sec:hyperparameter-grids-opt-details}). An additional cluster visualization is in Appendix~\ref{sec:additional-visualizations}.

\subsection{Preprocessing Notes}
\label{sec:datasets-preprocessing-details}

Continuous features are standardized (subtract mean and divide by standard deviation). Categorical features are one-hot encoded unless they correspond to features with a clear ordering in which case they are converted to be on a scale from 0 to 1 (evenly spaced for where the levels are). The \textsc{unos} data preprocessing is a bit more involved and follows the steps of \citet{yoon2018personalized} and \citet{lee2018deephit}. The \textsc{kkbox} dataset is preprocessed via the \texttt{pycox} package \citep{kvamme2019time}.

\subsection{Hyperparameter Grids (Including Neural Net Architecture Settings) and Optimization Details}
\label{sec:hyperparameter-grids-opt-details}

We provide hyperparameter grids and optimization details in this section. Before doing so, we make a few remarks. First, we implement the elastic-net-regularized Cox model as well as all neural net models in PyTorch \citep{paszke2019pytorch} so that we can train these models via minibatch gradient descent, which is helpful due to the size of some of our datasets.
Note that \textsc{xgboost} supports survival analysis with a few losses. We specifically use the Cox loss. We train all models that are implemented in PyTorch (including \textsc{elastic-net cox}) using Adam \citep{kingma2014adam} with a budget of 100 epochs (except for the \textsc{kkbox} dataset, where we use 10). Early stopping is used based on the validation set (no improvement in the best validation loss within 10 epochs). For simplicity, we always use a batch size of 1024 (from some preliminary experiments, we find that smaller batch sizes such as 128, 256, and 512 yield similar results for baselines but for survival kernets, a larger batch size appears to be helpful).

For the \textsc{kkbox} dataset, because of how large it is, even computing the validation loss is computationally expensive ($C^{\text{td}}$ index is a ranking metric that considers every pair of data points). To make training more practical, we use a subsampling strategy for computing the validation loss where we randomly partition the validation data into groups of size~$2^{14}$, compute the $C^{\text{td}}$ index per group, and then average the indices computed across the groups. We empirically found this approximation to be close to the exact calculation and takes much less time to compute. For the test set as well as bootstrap confidence intervals, we do an exact calculation rather than using this subsampling strategy.

We now list hyperparameter grids of the different methods. Even more details are available in our code.

\medskip
\noindent
\textsc{elastic-net cox}:
\begin{itemize}[leftmargin=*,itemsep=0.2em,parsep=0em,partopsep=0em,topsep=0.1em]
\item Learning rate: 1.0, 0.1, 0.01, 0.001, 0.0001
\item Regularization weight: 0 (corresponds to the standard Cox model \citep{cox1972regression}), 0.01, 1
\item Elastic-net knob for how much to use $\ell_1$ regularization vs squared $\ell_2$ regularization: 0 (ridge regression), 0.5 (evenly balance $\ell_1$ and squared $\ell_2$), 1 (lasso)
\end{itemize}

\medskip
\noindent
\textsc{xgboost}:
\begin{itemize}[leftmargin=*,itemsep=0.2em,parsep=0em,partopsep=0em,topsep=0.1em]
\item Number of random features selected per node: \\ $\frac{1}{2}\times\text{sqrt of number of features}$, $\text{sqrt of number of features}$, $2\times\text{sqrt of number of features}$, use all features
\item Learning rate ``eta'': 0.1, 0.3, 1
\item Number of parallel trees: 1, 10
\item Training data subsampling when growing trees: 0.6, 0.8
\item Max depth: 3, 6, 9, 12
\item Max number of rounds: 100 (for \textsc{kkbox}, only 20)
\end{itemize}

\medskip
\noindent
\textsc{deepsurv}:
\begin{itemize}[leftmargin=*,itemsep=0.2em,parsep=0em,partopsep=0em,topsep=0.1em]
\item Learning rate: 0.01, 0.001
\item Number of hidden layers: 2, 4, 6
\item Number of nodes per hidden layer: 32, 64, 128
\item Nonlinear activation of hidden layers: ReLU
\item Use batch norm after each hidden layer: yes
\end{itemize}
The final layer (a linear layer) is not a hidden layer and corresponds to a Cox model; it has 1 output node and no bias.

\medskip
\noindent
\textsc{deephit} has the same search grid as \textsc{deepsurv} and also additionally has the following:
\begin{itemize}[leftmargin=*,itemsep=0.2em,parsep=0em,partopsep=0em,topsep=0.1em]
\item ``alpha'': 0, 0.001, 0.01
\item ``sigma'': 0.1, 1
\item Number of time steps to discretize to: 0 (i.e., use all unique observed times of death), 64, 128
\end{itemize}
Importantly, for discretizing time, we use evenly spaced quantiles of observed times (for details, see Section 3.1 of \citet{kvamme2019continuous}), which in some circumstances could result in the number of time steps used be \emph{fewer} than what the user specifies. For example, when discretizing to 64 time steps, we take 64 evenly spaced quantiles of observed times. Suppose that 99.9\% of the observed times happen to all be exactly 1 day whereas the other 0.1\% of the observed times are exactly 2 days. Then the 64 evenly spaced quantiles would actually all correspond to 1 day as the observed time. Of course, if all the observed times are unique, then we would not have this sort of issue.

\medskip
\noindent
\textsc{dcm} has the same search grid as \textsc{deepsurv} and also additionally has the following:
\begin{itemize}[leftmargin=*,itemsep=0.2em,parsep=0em,partopsep=0em,topsep=0.1em]
\item Number of Cox distributions: 3, 4, 5, 6
\item Smoothing factor: $10^{-3}$
\end{itemize}
The MLP base neural net's final number of output nodes is precisely the number of dimensions of the embedding space. We set this number to be the minimum of: (1) the number of nodes per hidden layer, and (2) the number of features in feature space (after preprocessing).

\medskip
\noindent
\textsc{dksa} (we used the reference code by the original author specifically with random survival forest initialization; the only changes made were to get it to work with our experimental harness that loads in data, tunes hyperparameters, and gets test set metrics):
\begin{itemize}[leftmargin=*,itemsep=0.2em,parsep=0em,partopsep=0em,topsep=0.1em]
\item Number of random features selected per node (for random survival forest warm-start): \\ $\frac{1}{2}\times\text{sqrt of number of features}$, $\text{sqrt of number of features}$, $2\times\text{sqrt of number of features}$
\item Minimum data points per leaf (for random survival forest warm-start): 8, 32, 128
\item Learning rate: 0.01, 0.001
\item Number of hidden layers: 2, 4, 6
\item Number of nodes per hidden layer: 32, 64, 128
\item Nonlinear activation of hidden layers: ReLU
\item Use batch norm after each hidden layer: yes
\item Number of time steps to discretize to: 0 (i.e., use all unique observed times of death), 64, 128; we use the same discretization strategy stated above for \textsc{deephit}
\end{itemize}
The MLP base neural net's final number of output nodes is set the same way as described above for \textsc{dcm}.

Note that we always project the base neural net's output onto a hypersphere of radius $\sqrt{0.1}$ (this is mathematically equivalent to projecting to the unit hypersphere but changing the kernel function to instead be $K(u)=\exp(-u^2/0.1)$). In fact, this idea is commonly used in contrastive learning and the radius is a hyperparameter that can be tuned and is equal to the square root of two multiplied by what is commonly called the ``temperature'' hyperparameter (see for instance the papers by \citet{wang2020understanding} and \citet{liu2021learning}).

\medskip
\noindent
\textsc{kernet}
has the same search grid as \textsc{deepsurv} and also additionally has the following:
\begin{itemize}[leftmargin=*,itemsep=0.2em,parsep=0em,partopsep=0em,topsep=0.1em]
\item $\eta$: 0, 0.001, 0.01 (defined in $L_{\text{rank}}$ as part of equation~\eqref{eq:full-dksa-loss}; this hyperparameter is analogous to \textsc{deephit}'s ``alpha'' hyperparameter)
\item $\sigma_{\text{rank}}$: 0.1, 1  (also defined in $L_{\text{rank}}$ as part of equation~\eqref{eq:full-dksa-loss}; this hyperparameter is analogous to \textsc{deephit}'s ``sigma'' hyperparameter)
\item Number of time steps to discretize to: 0 (i.e., use all unique observed times of death), 64, 128; we use the same discretization strategy stated above for \textsc{deephit}
\item $\tau$: $\sqrt{2\log(10)}\approx2.146$ (we additionally set the maximum number of approximate nearest neighbors to be 128)
\item $\beta$: 1/4, 1/2 (we construct $\varepsilon$-nets with $\varepsilon=\beta\tau$)
\end{itemize}
We set the MLP base neural net's final number of output nodes the same way as the \textsc{dksa} baseline and again project the base neural net's output to a hypersphere with radius $\sqrt{0.1}$. From preliminary experiments, using a unit hypersphere (i.e., setting the radius to be equal to 1) typically resulted in worse validation accuracy scores.

\medskip
\noindent
\textsc{tuna-kernet} is the same as \textsc{kernet} except that our \textsc{tuna} warm-start strategy is used.

\subsection{More Detailed Computation Time Breakdown Example}
\label{sec:more-detailed-time-breakdown}

As an illustrative example, we give a more detailed breakdown of Table~\ref{tab:tuna-kernet-nosplit-sft-kkbox-breakdown} where specifically for the step that fine-tunes the base neural net with the DKSA loss, we not only separate computation times with respect to $\beta$ but we also separate the computation times with respect to the number of time steps that we discretize to. The resulting breakdown is shown in Table~\ref{tab:tuna-kernet-nosplit-sft-kkbox-breakdown-detailed}. This same idea could of course be used to analyze the amount of time of other hyperparameters too. Specifically for \textsc{tuna-kernet (no split, sft)}, we see that for the \textsc{kkbox} dataset, discretizing using 64 or 128 time steps takes a similar amount of time (although using 128 time steps is of course a little more costly), while discretizing using all unique observed times of death is significantly more costly in time (but for \textsc{kkbox} specifically, it turns out that discretizing using all unique observed times of death significantly improves accuracy). As a minor technical reminder, as stated in the Appendix~\ref{sec:hyperparameter-grids-opt-details}, when 64 time steps are used, we are not guaranteed to have exactly 64 time steps and we might have fewer (similarly for when 128 time steps are used).

\begin{table}
\centering
\caption{Time breakdown in training the \textsc{tuna-kernet (no split, sft)} model on the \textsc{kkbox} dataset for our experimental setup (mean $\pm$ standard deviation across 5 experimental repeats). The time breakdown for training \textsc{tuna-kernet (no split)} is the same without the final summary fine-tuning time. Note that our \textsc{tuna} warm-start strategy corresponds to the combination of the first two tasks listed below. Also, note that every task listed below involves tuning different hyperparameters. For details on hyperparameters, see Appendix~\ref{sec:hyperparameter-grids-opt-details}.}
\label{tab:tuna-kernet-nosplit-sft-kkbox-breakdown-detailed}
\vspace{-.5em}
\small
\setlength{\tabcolsep}{6pt}
\adjustbox{max width=\textwidth}{
\begin{tabular}{ccc}
\toprule
Task & Hyperparameter settings & Time (minutes) \\
\midrule
Train \textsc{xgboost} model & 192$^*$ & 292.358 $\pm$ 1.943 \\
Approximate \textsc{xgboost} kernel with base neural net & 18$^\dagger$ & 107.271 $\pm$ 0.481 \\
Fine-tune base neural net with DKSA loss ($\beta=1/2$, discretize to 64 time steps) & 12$^\ddagger$ & 163.928 $\pm$ 13.769 \\
Fine-tune base neural net with DKSA loss ($\beta=1/2$, discretize to 128 time steps) & 12$^\ddagger$ & 170.394 $\pm$ 16.056 \\
Fine-tune base neural net with DKSA loss ($\beta=1/2$, discretize to all unique observed times of death) & 12$^\ddagger$ & 459.561 $\pm$ 24.087 \\
Fine-tune base neural net with DKSA loss ($\beta=1/4$, discretize to 64 time steps) & 12$^\ddagger$ & 189.322 $\pm$ 10.688 \\
Fine-tune base neural net with DKSA loss ($\beta=1/4$, discretize to 128 time steps) & 12$^\ddagger$ & 196.903 $\pm$ 10.287 \\
Fine-tune base neural net with DKSA loss ($\beta=1/4$, discretize to all unique observed times of death) & 12$^\ddagger$ & 504.150 $\pm$ 15.288 \\
Summary fine-tuning & 2$^\text{\S}$ & 471.814 $\pm$ 37.123 \\
\midrule
Total & & 2555.702 $\pm$ 65.595 \\
\bottomrule
\end{tabular}} \\
\smallskip
$^*$ corresponds to the full \textsc{xgboost} hyperparameter grid that we use \\
$^\dagger$ base neural net hyperparameters (number of hidden layers, nodes per hidden layer), learning rate \\
$^\ddagger$ survival loss hyperparameters ($\eta$ and $\sigma_{\text{rank}}$ from equation~\eqref{eq:full-dksa-loss}), learning rate \\
$^\text{\S}$ learning rate
\end{table}

\subsection{Additional Cluster Visualization}
\label{sec:additional-visualizations}

For the \textsc{kkbox} dataset, we visualize the largest 10 clusters found by the final \textsc{tuna-kernet (no split, sft)} model in Figure~\ref{fig:kkbox}. These largest 10 clusters only contain 28.7\% of the proper training data.

\begin{figure}[!p]
\centering
\begin{subfigure}[b]{\linewidth}
\centering
\includegraphics[scale=.37]{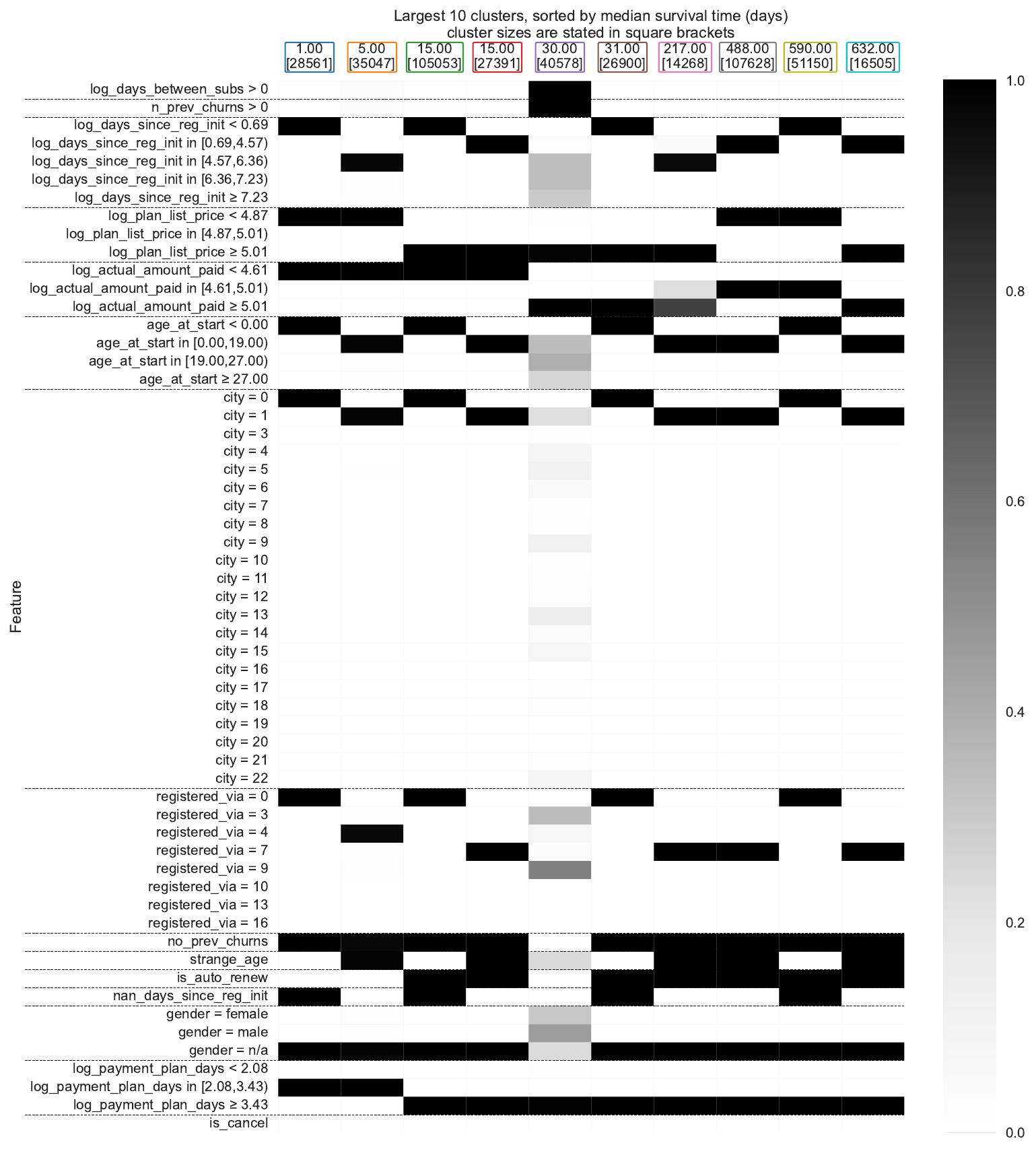}
\caption{}
\label{subfig:kkbox-heatmap}
\end{subfigure} \\ \vspace{.5em}
\begin{subfigure}[b]{\linewidth}
\centering
\hspace{1em}
\includegraphics[scale=.37]{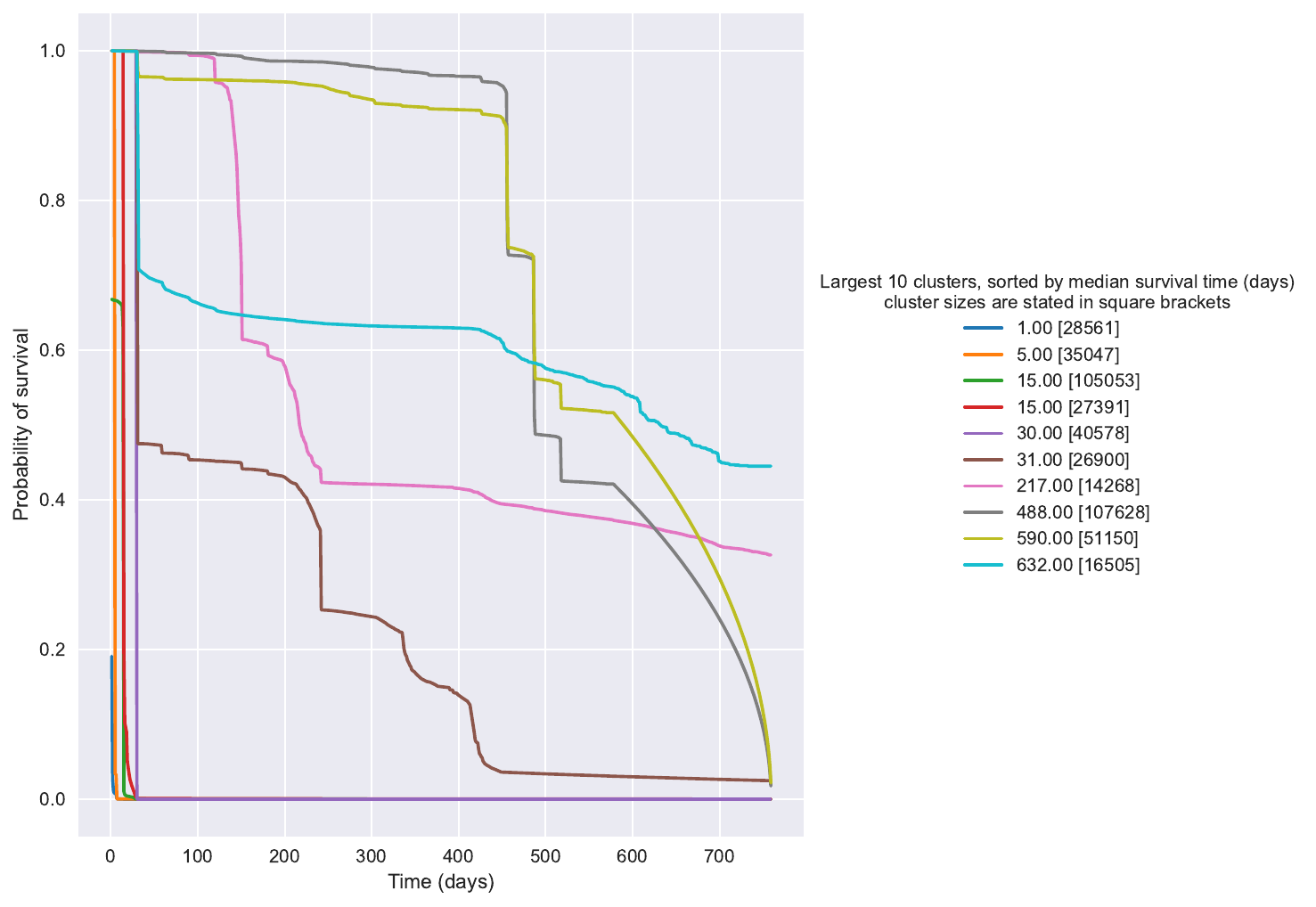}
\vspace{-.5em}
\caption{}
\label{subfig:kkbox-survival-curves}
\end{subfigure}
\caption{Visualization of the largest 10 clusters for the final \textsc{tuna-kernet (no split, sft)} model trained on the \textsc{kkbox} dataset. Panel (a) shows a feature heatmap visualization. Panel (b) shows survival curves for the same clusters as in panel (a); the x-axis measures the number of days since an individual subscribed to the music streaming service.}
\label{fig:kkbox}
\end{figure}

\subsection{Effect of Changing Threshold Distance\texorpdfstring{ $\tau$}{}}
\label{sec:different-threshold-distances}

While the threshold distance $\tau$ for survival kernets could be tuned as part of hyperparameter tuning, we leave it fixed as we found in preliminary experiments that at least for the different values of $\tau$ that we tried, the qualitative flavor of the results did not drastically change. Specifically, we tried $\tau=\sqrt{\log 10}$ (i.e., training points only contribute to the prediction of a test point if they have kernel weight/similarity score at least $0.1$ with the test point), $\tau=\sqrt{2\log 10}$ (minimum kernel weight $0.01$), and $\tau=\sqrt{4\log 10}$ (minimum kernel weight $0.0001$). The validation set (and not test set) $C^{\text{td}}$ indices are shown in Table~\ref{tab:distance-threshold-val}.

For example, looking specifically at the \textsc{tuna-kernet (no split, sft)} model, findings (a) and (b) in Section~\ref{sec:results-prediction-accuracy} still hold when we instead use $\tau=\sqrt{\log 10}$ (i.e., training points only contribute to the prediction of a test point if they have kernel weight/similarity score at least $0.1$ with the test point) or $\tau=\sqrt{4\log 10}$ (minimum kernel weight $0.0001$); look at Table~\ref{tab:distance-threshold} and compare the achieved test set $C^{\text{td}}$ indices with those of the baseline models in Table~\ref{tab:main-benchmark}.

\begin{table}[t]
\centering
\caption{\emph{Validation} set $C^{\text{td}}$ indices (mean $\pm$ standard deviation across 5 experimental repeats) for \textsc{tuna-kernet (no split, sft)} using different threshold distances.\label{tab:distance-threshold-val}}
\vspace{-.5em}
\setlength{\tabcolsep}{5pt}
\adjustbox{max width=\textwidth}{%
\begin{tabular}{ccccc}
\toprule
\multirow{2}{*}[-2.5pt]{Threshold distance} & \multicolumn{4}{c}{Dataset} \\
\cmidrule{2-5}
& \textsc{rotterdam/gbsg} & \textsc{support} & \textsc{unos} & \textsc{kkbox}\\
\midrule
$\tau = \sqrt{\log 10}$ & 0.6852 $\pm$ 0.0212 & 0.6463 $\pm$ 0.0074 & 0.6134 $\pm$ 0.0035 & 0.9047 $\pm$ 0.0008 \\
$\tau = \sqrt{2 \log 10}$ & 0.6874 $\pm$ 0.0198 & 0.6483 $\pm$ 0.0067 & {\bftab 0.6156} $\pm$ 0.0022 & {\bftab 0.9052} $\pm$ 0.0005 \\
$\tau = \sqrt{4 \log 10}$ & {\bftab 0.6887} $\pm$ 0.0239 & {\bftab 0.6484} $\pm$ 0.0055 & 0.6135 $\pm$ 0.0026 & 0.9030 $\pm$ 0.0043 \\
\bottomrule
\end{tabular}}
\end{table}

\begin{table}[t]
\centering
\caption{Test set $C^{\text{td}}$ indices (mean $\pm$ standard deviation across 5 experimental repeats) for \textsc{tuna-kernet (no split, sft)}  using different threshold distances.\label{tab:distance-threshold}}
\vspace{-.5em}
\setlength{\tabcolsep}{5pt}
\adjustbox{max width=\textwidth}{%
\begin{tabular}{ccccc}
\toprule
\multirow{2}{*}[-2.5pt]{Threshold distance} & \multicolumn{4}{c}{Dataset} \\
\cmidrule{2-5}
& \textsc{rotterdam/gbsg} & \textsc{support} & \textsc{unos} & \textsc{kkbox}\\
\midrule
$\tau = \sqrt{\log 10}$ & 0.6708 $\pm$ 0.0098 & 0.6373 $\pm$ 0.0047 & 0.6175 $\pm$ 0.0059 & 0.9051 $\pm$ 0.0006 \\
$\tau = \sqrt{2 \log 10}$ & 0.6719 $\pm$ 0.0135 & 0.6426 $\pm$ 0.0045 & 0.6211 $\pm$ 0.0025 & 0.9057 $\pm$ 0.0003 \\
$\tau = \sqrt{4 \log 10}$ & 0.6706 $\pm$ 0.0124 & 0.6403 $\pm$ 0.0031 & 0.6181 $\pm$ 0.0026 & 0.9035 $\pm$ 0.0042 \\
\bottomrule
\end{tabular}}
\end{table}

\newpage

\bibliography{survival_kernets}

\begin{thebibliography}{78}
\providecommand{\natexlab}[1]{#1}
\providecommand{\url}[1]{\texttt{#1}}
\expandafter\ifx\csname urlstyle\endcsname\relax
  \providecommand{\doi}[1]{doi: #1}\else
  \providecommand{\doi}{doi: \begingroup \urlstyle{rm}\Url}\fi

\bibitem[Agarwal et~al.(2021)Agarwal, Shah, Shen, and
  Song]{agarwal2021robustness}
Anish Agarwal, Devavrat Shah, Dennis Shen, and Dogyoon Song.
\newblock On robustness of principal component regression.
\newblock \emph{Journal of the American Statistical Association}, 116\penalty0
  (536):\penalty0 1731--1745, 2021.

\bibitem[Akiba et~al.(2019)Akiba, Sano, Yanase, Ohta, and
  Koyama]{akiba2019optuna}
Takuya Akiba, Shotaro Sano, Toshihiko Yanase, Takeru Ohta, and Masanori Koyama.
\newblock Optuna: A next-generation hyperparameter optimization framework.
\newblock In \emph{ACM SigKDD International Conference on Knowledge Discovery
  and Data Mining}, 2019.

\bibitem[Andoni and Razenshteyn(2015)]{andoni2015optimal}
Alexandr Andoni and Ilya Razenshteyn.
\newblock Optimal data-dependent hashing for approximate near neighbors.
\newblock In \emph{Symposium on Theory of Computing}, 2015.

\bibitem[Andoni et~al.(2015)Andoni, Indyk, Laarhoven, Razenshteyn, and
  Schmidt]{andoni2015practical}
Alexandr Andoni, Piotr Indyk, Thijs Laarhoven, Ilya Razenshteyn, and Ludwig
  Schmidt.
\newblock Practical and optimal {LSH} for angular distance.
\newblock In \emph{Advances in Neural Information Processing Systems}, 2015.

\bibitem[Antolini et~al.(2005)Antolini, Boracchi, and
  Biganzoli]{antolini2005time}
Laura Antolini, Patrizia Boracchi, and Elia Biganzoli.
\newblock A time-dependent discrimination index for survival data.
\newblock \emph{Statistics in Medicine}, 24:\penalty0 3927--3944, 2005.

\bibitem[Asmis et~al.(2008)Asmis, Ding, Seymour, Shepherd, Leighl, Winton,
  Whitehead, Spaans, Graham, and Goss]{asmis2008age}
Timothy~R. Asmis, Keyue Ding, Lesley Seymour, Frances~A. Shepherd, Natasha~B.
  Leighl, Tim~L. Winton, Marlo Whitehead, Johanna~N. Spaans, Barbara~C. Graham,
  and Glenwood~D. Goss.
\newblock Age and comorbidity as independent prognostic factors in the
  treatment of non--small-cell lung cancer: A review of national cancer
  institute of canada clinical trials group trials.
\newblock \emph{Journal of Clinical Oncology}, 26\penalty0 (1):\penalty0
  54--59, 2008.

\bibitem[Athey et~al.(2019)Athey, Tibshirani, and Wager]{athey2019generalized}
Susan Athey, Julie Tibshirani, and Stefan Wager.
\newblock Generalized random forests.
\newblock \emph{The Annals of Statistics}, 47\penalty0 (2):\penalty0
  1148--1178, 2019.

\bibitem[Banerjee et~al.(2005)Banerjee, Dhillon, Ghosh, and
  Sra]{banerjee2005clustering}
Arindam Banerjee, Inderjit~S. Dhillon, Joydeep Ghosh, and Suvrit Sra.
\newblock Clustering on the unit hypersphere using von mises-fisher
  distributions.
\newblock \emph{Journal of Machine Learning Research}, 6\penalty0
  (46):\penalty0 1345--1382, 2005.

\bibitem[Ben-David et~al.(2007)Ben-David, P{\'a}l, and Simon]{ben2007stability}
Shai Ben-David, D{\'a}vid P{\'a}l, and Hans~Ulrich Simon.
\newblock Stability of k-means clustering.
\newblock In \emph{International Conference on Computational Learning Theory},
  2007.

\bibitem[Beran(1981)]{beran1981nonparametric}
Rudolf Beran.
\newblock Nonparametric regression with randomly censored survival data.
\newblock \emph{Technical report, University of California, Berkeley}, 1981.

\bibitem[Biau(2012)]{biau2012analysis}
G{\'e}rard Biau.
\newblock Analysis of a random forests model.
\newblock \emph{Journal of Machine Learning Research}, 13\penalty0
  (38):\penalty0 1063--1095, 2012.

\bibitem[Blei et~al.(2003)Blei, Ng, and Jordan]{blei2003latent}
David~M. Blei, Andrew~Y. Ng, and Michael~I. Jordan.
\newblock Latent {D}irichlet allocation.
\newblock \emph{Journal of Machine Learning Research}, 3\penalty0
  (Jan):\penalty0 993--1022, 2003.

\bibitem[Borg and Groenen(2005)]{borg2005modern}
Ingwer Borg and Patrick J.~F. Groenen.
\newblock \emph{Modern Multidimensional Scaling: Theory and Applications}.
\newblock Springer Science \& Business Media, 2005.

\bibitem[Breiman(2000)]{breiman2000some}
Leo Breiman.
\newblock Some infinity theory for predictor ensembles.
\newblock \emph{Technical report 577, Statistics Department, University of
  California, Berkeley}, 2000.

\bibitem[Bronstein et~al.(2006)Bronstein, Bronstein, and
  Kimmel]{bronstein2006generalized}
Alexander~M. Bronstein, Michael~M. Bronstein, and Ron Kimmel.
\newblock Generalized multidimensional scaling: a framework for
  isometry-invariant partial surface matching.
\newblock \emph{Proceedings of the National Academy of Sciences}, 103\penalty0
  (5):\penalty0 1168--1172, 2006.

\bibitem[Brown(1975)]{brown1975use}
Charles~C. Brown.
\newblock On the use of indicator variables for studying the time-dependence of
  parameters in a response-time model.
\newblock \emph{Biometrics}, 31\penalty0 (4):\penalty0 863--872, 1975.

\bibitem[Chagny and Roche(2014)]{chagny2014adaptive}
Ga{\"e}lle Chagny and Angelina Roche.
\newblock Adaptive and minimax estimation of the cumulative distribution
  function given a functional covariate.
\newblock \emph{Electronic Journal of Statistics}, 8\penalty0 (2):\penalty0
  2352--2404, 2014.

\bibitem[Chapfuwa et~al.(2018)Chapfuwa, Tao, Li, Page, Goldstein, Duke, and
  Henao]{chapfuwa2018adversarial}
Paidamoyo Chapfuwa, Chenyang Tao, Chunyuan Li, Courtney Page, Benjamin
  Goldstein, Lawrence~Carin Duke, and Ricardo Henao.
\newblock Adversarial time-to-event modeling.
\newblock In \emph{International Conference on Machine Learning}, 2018.

\bibitem[Chapfuwa et~al.(2020)Chapfuwa, Li, Mehta, Carin, and
  Henao]{chapfuwa2020survival}
Paidamoyo Chapfuwa, Chunyuan Li, Nikhil Mehta, Lawrence Carin, and Ricardo
  Henao.
\newblock Survival cluster analysis.
\newblock In \emph{Conference on Health, Inference, and Learning}, 2020.

\bibitem[Chen et~al.(2019)Chen, Li, Tao, Barnett, Su, and Rudin]{chen2019looks}
Chaofan Chen, Oscar Li, Chaofan Tao, Alina~Jade Barnett, Jonathan Su, and
  Cynthia Rudin.
\newblock This looks like that: deep learning for interpretable image
  recognition.
\newblock In \emph{Advances in Neural Information Processing Systems}, 2019.

\bibitem[Chen(2019)]{chen2019nearest}
George~H. Chen.
\newblock Nearest neighbor and kernel survival analysis: Nonasymptotic error
  bounds and strong consistency rates.
\newblock In \emph{International Conference on Machine Learning}, 2019.

\bibitem[Chen(2020)]{chen2020deep}
George~H. Chen.
\newblock Deep kernel survival analysis and subject-specific survival time
  prediction intervals.
\newblock In \emph{Machine Learning for Healthcare Conference}, 2020.

\bibitem[Chen and Shah(2018)]{chen2018explaining}
George~H. Chen and Devavrat Shah.
\newblock Explaining the success of nearest neighbor methods in prediction.
\newblock \emph{Foundations and Trends{\textregistered} in Machine Learning},
  10\penalty0 (5-6):\penalty0 337--588, 2018.

\bibitem[Chen and Guestrin(2016)]{chen2016xgboost}
Tianqi Chen and Carlos Guestrin.
\newblock {XGBoost}: A scalable tree boosting system.
\newblock In \emph{ACM SigKDD International Conference on Knowledge Discovery
  and Data Mining}, 2016.

\bibitem[Cox(1972)]{cox1972regression}
David~R. Cox.
\newblock Regression models and life-tables.
\newblock \emph{Journal of the Royal Statistical Society: Series B},
  34\penalty0 (2):\penalty0 187--220, 1972.

\bibitem[Dabrowska(1989)]{dabrowska1989uniform}
Dorota~M. Dabrowska.
\newblock Uniform consistency of the kernel conditional {K}aplan-{M}eier
  estimate.
\newblock \emph{The Annals of Statistics}, 17\penalty0 (3):\penalty0
  1157--1167, 1989.

\bibitem[Danks and Yau(2022)]{danks2022derivative}
Dominic Danks and Christopher Yau.
\newblock Derivative-based neural modelling of cumulative distribution
  functions for survival analysis.
\newblock In \emph{International Conference on Artificial Intelligence and
  Statistics}, 2022.

\bibitem[de~Silva and Tenenbaum(2002)]{de2003global}
Vin de~Silva and Joshua~B. Tenenbaum.
\newblock Global versus local methods in nonlinear dimensionality reduction.
\newblock In \emph{Advances in Neural Information Processing Systems}, 2002.

\bibitem[Denil et~al.(2014)Denil, Matheson, and de~Freitas]{denil2014narrowing}
Misha Denil, David Matheson, and Nando de~Freitas.
\newblock Narrowing the gap: Random forests in theory and in practice.
\newblock In \emph{International Conference on Machine Learning}, 2014.

\bibitem[Dipchand(2018)]{dipchand2018current}
Anne~I. Dipchand.
\newblock Current state of pediatric cardiac transplantation.
\newblock \emph{Annals of Cardiothoracic Surgery}, 7\penalty0 (1):\penalty0
  31--55, 2018.

\bibitem[Engelhard et~al.(2020)Engelhard, Berchuck, D'Arcy, and
  Henao]{engelhard2020neural}
Matthew Engelhard, Samuel Berchuck, Joshua D'Arcy, and Ricardo Henao.
\newblock Neural conditional event time models.
\newblock In \emph{Machine Learning for Healthcare Conference}, 2020.

\bibitem[Ester et~al.(1996)Ester, Kriegel, Sander, and Xu]{ester1996density}
Martin Ester, Hans-Peter Kriegel, J{\"o}rg Sander, and Xiaowei Xu.
\newblock A density-based algorithm for discovering clusters in large spatial
  databases with noise.
\newblock In \emph{International Conference on Knowledge Discovery and Data
  Mining}, 1996.

\bibitem[Foekens et~al.(2000)Foekens, Peters, Look, Portengen, Schmitt, Kramer,
  Br{\"u}nner, J{\"a}nicke, Meijer-van Gelder, Henzen-Logmans, van Putten, and
  Klijn]{foekens2000urokinase}
John~A. Foekens, Harry~A. Peters, Maxime~P. Look, Henk Portengen, Manfred
  Schmitt, Michael~D. Kramer, Nils Br{\"u}nner, Fritz J{\"a}nicke, Marion~E.
  Meijer-van Gelder, Sonja~C. Henzen-Logmans, Wim L.~J. van Putten, and Jan
  G.~M. Klijn.
\newblock The urokinase system of plasminogen activation and prognosis in 2780
  breast cancer patients.
\newblock \emph{Cancer Research}, 60\penalty0 (3):\penalty0 636--643, 2000.

\bibitem[Fotso(2018)]{fotso2018deep}
Stephane Fotso.
\newblock Deep neural networks for survival analysis based on a multi-task
  framework.
\newblock \emph{arXiv preprint arXiv:1801.05512}, 2018.

\bibitem[Fr{\"a}nti and Sieranoja(2019)]{franti2019much}
Pasi Fr{\"a}nti and Sami Sieranoja.
\newblock How much can k-means be improved by using better initialization and
  repeats?
\newblock \emph{Pattern Recognition}, 93:\penalty0 95--112, 2019.

\bibitem[Frey et~al.(2007)Frey, Zhou, Harvey, and White]{frey2007co}
Charles Frey, Hong Zhou, Danielle Harvey, and Richard~H. White.
\newblock Co-morbidity is a strong predictor of early death and multi-organ
  system failure among patients with acute pancreatitis.
\newblock \emph{Journal of Gastrointestinal Surgery}, 11\penalty0 (6):\penalty0
  733--742, 2007.

\bibitem[Gensheimer and Narasimhan(2019)]{gensheimer2019scalable}
Michael~F. Gensheimer and Balasubramanian Narasimhan.
\newblock A scalable discrete-time survival model for neural networks.
\newblock \emph{PeerJ}, 7:\penalty0 e6257, 2019.

\bibitem[Giunchiglia et~al.(2018)Giunchiglia, Nemchenko, and van~der
  Schaar]{giunchiglia2018rnn}
Eleonora Giunchiglia, Anton Nemchenko, and Mihaela van~der Schaar.
\newblock {RNN-SURV}: A deep recurrent model for survival analysis.
\newblock In \emph{International Conference on Artificial Neural Networks},
  2018.

\bibitem[Goldstein et~al.(2020)Goldstein, Han, Puli, Perotte, and
  Ranganath]{goldstein2020x}
Mark Goldstein, Xintian Han, Aahlad Puli, Adler Perotte, and Rajesh Ranganath.
\newblock {X-CAL}: Explicit calibration for survival analysis.
\newblock In \emph{Advances in Neural Information Processing Systems}, 2020.

\bibitem[Haider et~al.(2020)Haider, Hoehn, Davis, and
  Greiner]{haider2020effective}
Humza Haider, Bret Hoehn, Sarah Davis, and Russell Greiner.
\newblock Effective ways to build and evaluate individual survival
  distributions.
\newblock \emph{Journal of Machine Learning Research}, 21\penalty0
  (85):\penalty0 1--63, 2020.

\bibitem[Hanneke et~al.(2021)Hanneke, Kontorovich, Sabato, and
  Weiss]{hanneke2021universal}
Steve Hanneke, Aryeh Kontorovich, Sivan Sabato, and Roi Weiss.
\newblock Universal {Bayes} consistency in metric spaces.
\newblock \emph{The Annals of Statistics}, 49\penalty0 (4):\penalty0
  2129--2150, 2021.

\bibitem[Hastie et~al.(2009)Hastie, Tibshirani, and
  Friedman]{hastie2009elements}
Trevor Hastie, Robert Tibshirani, and Jerome~H. Friedman.
\newblock \emph{The Elements of Statistical Learning: Data Mining, Inference,
  and Prediction (2nd ed.)}.
\newblock Springer, 2009.

\bibitem[He et~al.(2015)He, Zhang, Ren, and Sun]{he2015delving}
Kaiming He, Xiangyu Zhang, Shaoqing Ren, and Jian Sun.
\newblock Delving deep into rectifiers: Surpassing human-level performance on
  {ImageNet} classification.
\newblock In \emph{IEEE International Conference on Computer Vision}, 2015.

\bibitem[Kalbfleisch and Prentice(1980)]{kalbfleisch1980statistical}
John~D. Kalbfleisch and Ross~L. Prentice.
\newblock \emph{The Statistical Analysis of Failure Time Data}.
\newblock Wiley, 1980.

\bibitem[Kaplan and Meier(1958)]{kaplan1958nonparametric}
Edward~L. Kaplan and Paul Meier.
\newblock Nonparametric estimation from incomplete observations.
\newblock \emph{Journal of the American Statistical Association}, 53\penalty0
  (282):\penalty0 457--481, 1958.

\bibitem[Katzman et~al.(2018)Katzman, Shaham, Cloninger, Bates, Jiang, and
  Kluger]{katzman2018deepsurv}
Jared~L. Katzman, Uri Shaham, Alexander Cloninger, Jonathan Bates, Tingting
  Jiang, and Yuval Kluger.
\newblock {DeepSurv}: personalized treatment recommender system using a {Cox}
  proportional hazards deep neural network.
\newblock \emph{BMC Medical Research Methodology}, 18\penalty0 (24), 2018.

\bibitem[Kingma and Ba(2015)]{kingma2014adam}
Diederik~P. Kingma and Jimmy Ba.
\newblock Adam: A method for stochastic optimization.
\newblock In \emph{International Conference on Learning Representations}, 2015.

\bibitem[Knaus et~al.(1995)Knaus, Harrell, Lynn, Goldman, Phillips, Connors,
  Dawson, Fulkerson, Califf, Desbiens, Layde, Oye, Bellamy, Hakim, and
  Wagner]{knaus1995support}
William~A. Knaus, Frank~E. Harrell, Joanne Lynn, Lee Goldman, Russell~S.
  Phillips, Alfred~F. Connors, Neal~V. Dawson, William~J. Fulkerson, Robert~M.
  Califf, Norman Desbiens, Peter Layde, Robert~K. Oye, Paul~E. Bellamy,
  Rosemarie~B. Hakim, and Douglas~P. Wagner.
\newblock The {SUPPORT} prognostic model: Objective estimates of survival for
  seriously ill hospitalized adults.
\newblock \emph{Annals of Internal Medicine}, 122\penalty0 (3):\penalty0
  191--203, 1995.

\bibitem[Kontorovich et~al.(2017)Kontorovich, Sabato, and
  Weiss]{kontorovich2017nearest}
Aryeh Kontorovich, Sivan Sabato, and Roi Weiss.
\newblock Nearest-neighbor sample compression: Efficiency, consistency,
  infinite dimensions.
\newblock In \emph{Advances in Neural Information Processing Systems}, 2017.

\bibitem[Kpotufe and Verma(2017)]{kpotufe2017time}
Samory Kpotufe and Nakul Verma.
\newblock Time-accuracy tradeoffs in kernel prediction: controlling prediction
  quality.
\newblock \emph{Journal of Machine Learning Research}, 18\penalty0
  (44):\penalty0 1--29, 2017.

\bibitem[Kumar and Kannan(2010)]{kumar2010clustering}
Amit Kumar and Ravindran Kannan.
\newblock Clustering with spectral norm and the k-means algorithm.
\newblock In \emph{IEEE Symposium on Foundations of Computer Science}, 2010.

\bibitem[Kvamme and Borgan(2021)]{kvamme2019continuous}
H{\aa}vard Kvamme and {\O}rnulf Borgan.
\newblock Continuous and discrete-time survival prediction with neural
  networks.
\newblock \emph{Lifetime Data Analysis}, 27:\penalty0 710--736, 2021.

\bibitem[Kvamme et~al.(2019)Kvamme, Borgan, and Scheel]{kvamme2019time}
H{\aa}vard Kvamme, {\O}rnulf Borgan, and Ida Scheel.
\newblock Time-to-event prediction with neural networks and {C}ox regression.
\newblock \emph{Journal of Machine Learning Research}, 20\penalty0
  (129):\penalty0 1--30, 2019.

\bibitem[Le-Khac et~al.(2020)Le-Khac, Healy, and Smeaton]{le2020contrastive}
Phuc~H. Le-Khac, Graham Healy, and Alan~F. Smeaton.
\newblock Contrastive representation learning: A framework and review.
\newblock \emph{IEEE Access}, 8:\penalty0 193907--193934, 2020.

\bibitem[Lee et~al.(2018)Lee, Zame, Yoon, and van~der Schaar]{lee2018deephit}
Changhee Lee, William~R. Zame, Jinsung Yoon, and Mihaela van~der Schaar.
\newblock {DeepHit}: A deep learning approach to survival analysis with
  competing risks.
\newblock In \emph{AAAI Conference on Artificial Intelligence}, 2018.

\bibitem[Li et~al.(2020)Li, Zuo, Coston, Weiss, and Chen]{li2020neural}
Linhong Li, Ren Zuo, Amanda Coston, Jeremy~C. Weiss, and George~H. Chen.
\newblock Neural topic models with survival supervision: Jointly predicting
  time-to-event outcomes and learning how clinical features relate.
\newblock In \emph{International Conference on Artificial Intelligence in
  Medicine}, 2020.

\bibitem[Liu et~al.(2021)Liu, Lin, Liu, Xiong, Sch{\"o}lkopf, and
  Weller]{liu2021learning}
Weiyang Liu, Rongmei Lin, Zhen Liu, Li~Xiong, Bernhard Sch{\"o}lkopf, and
  Adrian Weller.
\newblock Learning with hyperspherical uniformity.
\newblock In \emph{International Conference on Artificial Intelligence and
  Statistics}, 2021.

\bibitem[Malkov and Yashunin(2020)]{malkov2020efficient}
Yu~A. Malkov and Dmitry~A. Yashunin.
\newblock Efficient and robust approximate nearest neighbor search using
  hierarchical navigable small world graphs.
\newblock \emph{IEEE Transactions on Pattern Analysis and Machine
  Intelligence}, 42\penalty0 (4):\penalty0 824--836, 2020.

\bibitem[Manduchi et~al.(2022)Manduchi, Marcinkevi{\v{c}}s, Massi, Weikert,
  Sauter, Gotta, M{\"u}ller, Vasella, Neidert, Pfister, Stieltjes, and
  Vogt]{manduchi2022deep}
Laura Manduchi, Ri{\v{c}}ards Marcinkevi{\v{c}}s, Michela~C. Massi, Thomas
  Weikert, Alexander Sauter, Verena Gotta, Timothy M{\"u}ller, Flavio Vasella,
  Marian~C. Neidert, Marc Pfister, Bram Stieltjes, and Julia~E. Vogt.
\newblock A deep variational approach to clustering survival data.
\newblock In \emph{International Conference on Learning Representations}, 2022.

\bibitem[Nagpal et~al.(2021)Nagpal, Yadlowsky, Rostamzadeh, and
  Heller]{nagpal2021deep}
Chirag Nagpal, Steve Yadlowsky, Negar Rostamzadeh, and Katherine Heller.
\newblock Deep {C}ox mixtures for survival regression.
\newblock In \emph{Machine Learning for Healthcare Conference}, 2021.

\bibitem[Paszke et~al.(2019)Paszke, Gross, Massa, Lerer, Bradbury, Chanan,
  Killeen, Lin, Gimelshein, Antiga, Desmaison, Kopf, Yang, DeVito, Raison,
  Tejani, Chilamkurthy, Steiner, Fang, Bai, and Chintala]{paszke2019pytorch}
Adam Paszke, Sam Gross, Francisco Massa, Adam Lerer, James Bradbury, Gregory
  Chanan, Trevor Killeen, Zeming Lin, Natalia Gimelshein, Luca Antiga, Alban
  Desmaison, Andreas Kopf, Edward Yang, Zachary DeVito, Martin Raison, Alykhan
  Tejani, Sasank Chilamkurthy, Benoit Steiner, Lu~Fang, Junjie Bai, and Soumith
  Chintala.
\newblock {PyTorch}: An imperative style, high-performance deep learning
  library.
\newblock In \emph{Advances in Neural Information Processing Systems}, 2019.

\bibitem[Prokhorenkova and Shekhovtsov(2020)]{prokhorenkova2020graph}
Liudmila Prokhorenkova and Aleksandr Shekhovtsov.
\newblock Graph-based nearest neighbor search: From practice to theory.
\newblock In \emph{International Conference on Machine Learning}, 2020.

\bibitem[Rakhlin and Caponnetto(2006)]{rakhlin2006stability}
Alexander Rakhlin and Andrea Caponnetto.
\newblock Stability of k-means clustering.
\newblock In \emph{Advances in Neural Information Processing Systems}, 2006.

\bibitem[Ranganath et~al.(2016)Ranganath, Perotte, Elhadad, and
  Blei]{ranganath2016deep}
Rajesh Ranganath, Adler Perotte, No{\'e}mie Elhadad, and David Blei.
\newblock Deep survival analysis.
\newblock In \emph{Machine Learning for Healthcare Conference}, 2016.

\bibitem[Russo et~al.(2010)Russo, Hong, Davies, Chen, Mancini, Oz, Rose,
  Gelijns, and Naka]{russo2010effect}
Mark~J. Russo, Kimberly~N. Hong, Ryan~R. Davies, Jonathan~M. Chen, Donna~M.
  Mancini, Mehmet~C. Oz, Eric~A. Rose, Annetine Gelijns, and Yoshifumi Naka.
\newblock The effect of body mass index on survival following heart
  transplantation: do outcomes support consensus guidelines?
\newblock \emph{Annals of Surgery}, 251\penalty0 (1):\penalty0 144--152, 2010.

\bibitem[Schumacher et~al.(1994)Schumacher, Bastert, Bojar, Huebner,
  Olschewski, Sauerbrei, Schmoor, Beyerle, Neumann, and
  Rauschecker]{schumacher1994randomized}
M.~Schumacher, G.~Bastert, H.~Bojar, K.~Huebner, M.~Olschewski, W.~Sauerbrei,
  C.~Schmoor, C.~Beyerle, R.~L. Neumann, and H.~F. Rauschecker.
\newblock Randomized 2 x 2 trial evaluating hormonal treatment and the duration
  of chemotherapy in node-positive breast cancer patients. german breast cancer
  study group.
\newblock \emph{Journal of Clinical Oncology}, 12\penalty0 (10):\penalty0
  2086--2093, 1994.

\bibitem[Sopik and Narod(2018)]{sopik2018relationship}
Victoria Sopik and Steven~A. Narod.
\newblock The relationship between tumour size, nodal status and distant
  metastases: on the origins of breast cancer.
\newblock \emph{Breast Cancer Research and Treatment}, 170\penalty0
  (3):\penalty0 647--656, 2018.

\bibitem[Tang et~al.(2022)Tang, Ma, Mei, and Zhu]{tang2022soden}
Weijing Tang, Jiaqi Ma, Qiaozhu Mei, and Ji~Zhu.
\newblock {SODEN}: A scalable continuous-time survival model through ordinary
  differential equation networks.
\newblock \emph{Journal of Machine Learning Research}, 23\penalty0
  (34):\penalty0 1--29, 2022.

\bibitem[van Eeghen et~al.(2015)van Eeghen, Bakker, van Bochove, and
  Loffeld]{van2015impact}
Elmer~E. van Eeghen, Sandra~D. Bakker, Aart van Bochove, and Ruud J. L.~F.
  Loffeld.
\newblock Impact of age and comorbidity on survival in colorectal cancer.
\newblock \emph{Journal of Gastrointestinal Oncology}, 6\penalty0 (6):\penalty0
  605--612, 2015.

\bibitem[Vershynin(2018)]{vershynin2018high}
Roman Vershynin.
\newblock \emph{High-Dimensional Probability: An Introduction with Applications
  in Data Science}.
\newblock Cambridge University Press, 2018.

\bibitem[von Luxburg(2010)]{von2010clustering}
Ulrike von Luxburg.
\newblock Clustering stability: An overview.
\newblock \emph{Foundations and Trends{\textregistered} in Machine Learning},
  2\penalty0 (3):\penalty0 235--274, 2010.

\bibitem[Wager and Athey(2018)]{wager2018estimation}
Stefan Wager and Susan Athey.
\newblock Estimation and inference of heterogeneous treatment effects using
  random forests.
\newblock \emph{Journal of the American Statistical Association}, 113\penalty0
  (523):\penalty0 1228--1242, 2018.

\bibitem[Wang and Isola(2020)]{wang2020understanding}
Tongzhou Wang and Phillip Isola.
\newblock Understanding contrastive representation learning through alignment
  and uniformity on the hypersphere.
\newblock In \emph{International Conference on Machine Learning}, 2020.

\bibitem[Wu et~al.(2021)Wu, Yang, Fasching, and Tresp]{wu2021uncertainty}
Zhiliang Wu, Yinchong Yang, Peter~A. Fasching, and Volker Tresp.
\newblock Uncertainty-aware time-to-event prediction using deep kernel
  accelerated failure time models.
\newblock In \emph{Machine Learning for Healthcare Conference}, 2021.

\bibitem[Yoon et~al.(2018)Yoon, Zame, Banerjee, Cadeiras, Alaa, and van~der
  Schaar]{yoon2018personalized}
Jinsung Yoon, William~R. Zame, Amitava Banerjee, Martin Cadeiras, Ahmed~M.
  Alaa, and Mihaela van~der Schaar.
\newblock Personalized survival predictions via trees of predictors: An
  application to cardiac transplantation.
\newblock \emph{PloS one}, 13\penalty0 (3):\penalty0 e0194985, 2018.

\bibitem[Zhao et~al.(2021)Zhao, Phung, Huynh, Jin, Du, and
  Buntine]{zhao2021topic}
He~Zhao, Dinh Phung, Viet Huynh, Yuan Jin, Lan Du, and Wray Buntine.
\newblock Topic modelling meets deep neural networks: A survey.
\newblock In \emph{International Joint Conference on Artificial Intelligence},
  2021.

\bibitem[Zhong et~al.(2021)Zhong, Mueller, and Wang]{zhong2021deep}
Qixian Zhong, Jonas Mueller, and Jane-Ling Wang.
\newblock Deep extended hazard models for survival analysis.
\newblock \emph{Advances in Neural Information Processing Systems}, 2021.

\bibitem[Zhong et~al.(2022)Zhong, Mueller, and Wang]{zhong2022deep}
Qixian Zhong, Jonas Mueller, and Jane-Ling Wang.
\newblock Deep learning for the partially linear {Cox} model.
\newblock \emph{The Annals of Statistics}, 50\penalty0 (3):\penalty0
  1348--1375, 2022.

\end{thebibliography}

\end{document}